\documentclass[twoside,11pt]{article}

\usepackage{blindtext}

\usepackage[preprint,nohyperref]{jmlr2e}

\usepackage[dvipsnames]{xcolor}
\usepackage[colorlinks, citecolor=blue, pagebackref=true]{hyperref}
\renewcommand*\backref[1]{\ifx#1\relax \else (Cited on p. #1) \fi}

\usepackage{lastpage}
\jmlrheading{25}{2024}{1-\pageref{LastPage}}{3/24; Revised }{}{21-0000}{Clément Bonet, Lucas Drumetz, and Nicolas Courty}

\ShortHeadings{Sliced-Wasserstein Distances and Flows on Cartan-Hadamard Manifold}{Bonet, Drumetz, Courty}
\firstpageno{1}

\newcommand{\sw}{\mathrm{SW}}
\newcommand{\chsw}{\mathrm{CHSW}}
\newcommand{\gchsw}{\mathrm{GCHSW}}
\newcommand{\hchsw}{\mathrm{HCHSW}}
\newcommand{\hsw}{\mathrm{HSW}}
\newcommand{\ghsw}{\mathrm{GHSW}}
\newcommand{\hhsw}{\mathrm{HHSW}}
\newcommand{\spdsw}{\mathrm{SPDSW}}
\newcommand{\isw}{\mathrm{ISW}}
\newcommand{\sign}{\mathrm{sign}}
\newcommand{\chr}{\mathrm{CHR}}
\newcommand{\id}{\mathrm{Id}}
\newcommand{\tr}{\mathrm{Tr}}

\def\vol{\mathrm{Vol}}

\usepackage{xcolor}
\usepackage{multirow}

\usepackage{amsmath,amsfonts,mathtools} % amssymb
\usepackage{bbold}
\usepackage{bm}

\usepackage{graphicx}
\usepackage{float}
\usepackage{subfig}
\usepackage{booktabs} % for professional tables

\usepackage{algorithm}
\usepackage{algorithmic}

\DeclareMathOperator*{\argmin}{argmin}

\DeclareMathOperator*{\arccosh}{arccosh}
\DeclareMathOperator*{\arctanh}{arctanh}

\usepackage{cleveref}

\begin{document}

\title{Sliced-Wasserstein Distances and Flows on Cartan-Hadamard Manifolds}

\author{\name Clément Bonet \email clement.bonet@ensae.fr \\
       \addr ENSAE, CREST, Institut Polytechnique de Paris %Department of Statistics\\
       % University of Washington\\
       % Seattle, WA 98195-4322, USA
       \AND
       \name Lucas Drumetz \email lucas.drumetz@imt-atlantique.fr \\
       \addr IMT Atlantique, Lab-STICC
       % University of California\\
       % Berkeley, CA 94720-1776, USA
       \AND
       \name Nicolas Courty \email nicolas.courty@irisa.fr \\
       \addr Université Bretagne Sud, Irisa
       }
       
\editor{My editor}

\maketitle

\begin{abstract}%   <- trailing '%' for backward compatibility of .sty file
    While many Machine Learning methods were developed or transposed on Riemannian manifolds to tackle data with known non Euclidean geometry, Optimal Transport (OT) methods on such spaces have not received much attention. The main OT tool on these spaces is the Wasserstein distance which suffers from a heavy computational burden. On Euclidean spaces, a popular alternative is the Sliced-Wasserstein distance, which leverages a closed-form solution of the Wasserstein distance in one dimension, but which is not readily available on manifolds. In this work, we derive general constructions of Sliced-Wasserstein distances on Cartan-Hadamard manifolds, Riemannian manifolds with non-positive curvature, which include among others Hyperbolic spaces or the space of Symmetric Positive Definite matrices. Then, we propose different applications. Additionally, we derive non-parametric schemes to minimize these new distances by approximating their Wasserstein gradient flows.
\end{abstract}


% \looseness=-1 This chapter aims at providing a general recipe to construct intrinsic extensions of the Sliced-Wasserstein distance on Riemannian manifolds. While many Machine Learning methods were developed or transposed on Riemannian manifolds to tackle data with known non Euclidean geometry, Optimal Transport methods on such spaces have not received much  attention. The main OT tools on these spaces are the Wasserstein distance and its entropic regularization with geodesic ground cost, but with the same bottleneck as in the Euclidean space. Hence, it is of much interest to develop new OT distances on such spaces, which allow to alleviate the computational burden. This chapter introduces a general construction and will be followed by three chapters covering specific cases of Riemannian manifolds with Machine Learning applications. Namely, we will study the particular case of Hyperbolic spaces, of the space of Symmetric Positive Definite matrices and of the Sphere.

\section{Introduction}

% Working directly on Riemannian manifolds has received a lot of attention in recent years. On the one hand, i
It is widely accepted that data have an underlying structure on a low dimensional manifold \citep{bengio2013representation}. However, it can be intricate to work directly on such data manifolds as we lack from an analytical model. Therefore, most works only focus on Euclidean space and do not take advantage of this representation. In some cases though, the data naturally and explicitly lies on a manifold, or can be embedded on some known manifolds allowing %one to take into account its 
to leverage their intrinsic structure. In such cases, it has been shown to be beneficial to exploit such structure by leveraging the metric of the manifold rather than relying on an Euclidean embedding.
% \red{using tools \nc{metrics ?} intrinsically defined on the manifold rather than relying on an Euclidean embedding.} % working directly on the manifold. 
To name a few examples, directional or geophysical data - data for which only the direction provides information - naturally lie on the sphere \citep{mardia2000directional} and hence their structure can be exploited by using methods suited to the sphere. Another popular example is given by data having a known hierarchical structure. Then, such data benefit from being embedded into hyperbolic spaces \citep{nickel2017poincare}.

Motivated by these examples, many works proposed new tools to handle data lying on Riemannian manifolds. To cite a few, \citet{fletcher2004principal, huckemann2006principal} developed PCA to perform dimension reduction on manifolds while \citet{le2019approximation} studied density approximation, \citet{feragen2015geodesic, jayasumana2015kernel,fang2021kernel} studied kernel methods and \citet{azangulov2022stationary, azangulov2023stationary} developed Gaussian processes on (homogeneous) manifolds. More recently, there has been many interests into developing new neural networks with architectures taking into account the geometry of the ambient manifold \citep{bronstein2017geometric} such as Residual Neural Networks \citep{katsmann2022riemannian}, discrete Normalizing Flows \citep{bose2020latent, rezende2020normalizing, rezende2021implicit} or Continuous Normalizing Flows \citep{mathieu2020riemannian, lou2020neural, rozen2021moser, yataka2022grassmann}. In the generative model literature, we can also cite the recent \citep{chen2023riemannian} which extended the flow matching training of Continuous Normalizing Flows to Riemannian manifolds, or \citet{bortoli2022riemannian, huang2022riemannian} who performed score based generative modeling and \citet{thornton2022riemannian} who studied Schrödinger bridges on manifolds.

\looseness=-1 To compare probability distributions or perform generative modeling tasks, one usually needs suitable discrepancies or distances. In Machine Learning, classical divergences used are for example the Kullback-Leibler divergence or the Maximum Mean Discrepancy (MMD). While these distances are well defined for distributions lying on Riemannian manifolds, generally by crafting dedicated kernels for the MMD \citep{feragen2015geodesic}, %\notsure{taking an extra care for the choice of the kernel in MMD, see \emph{e.g.} \citep{feragen2015geodesic},} \nc{ {\bf alternative:} generally by crafting dedicated kernels (\emph{e.g.} \citep{feragen2015geodesic}),} \red{ou enlever ce bout?} 
other choices which take more into account the geometry of the underlying space are Optimal Transport based distances whose most prominent example is the Wasserstein distance. 

While the Wasserstein distance is well defined on Riemannian manifolds, and has been studied in many works theoretically, see \emph{e.g.} \citep{mccann2001polar, villani2009optimal}, it suffers from a significant computational burden. %the same computational burdens as in the Euclidean case. 
In the Euclidean case, different solutions were proposed to alleviate this computational cost, such as adding an entropic regularization and leveraging the Sinkhorn algorithm \citep{cuturi2013sinkhorn}, approximating the Wasserstein distance using minibatchs \citep{fatras2020learning}, using low-rank couplings \citep{scetbon2022low} or tree metrics \citep{le2019tree}. These approximations can be readily extended to Riemannian manifolds using the right ground costs. For example, \citet{alvarez2020unsupervised, hoyos2020aligning} used the entropic regularized formulation on Hyperbolic spaces. Another popular alternative to the Wasserstein distance is the so-called Sliced-Wasserstein distance (SW). While on Euclidean spaces, the Sliced-Wasserstein distance is a tractable alternative allowing to work in large scale settings, it cannot be directly extended to Riemannian manifolds since it relies on orthogonal projections of the measures on straight lines. Hence, % \notsure{as underlined in the conclusion of the thesis of \citet{nadjahi2021sliced}} \nc{je pense qu'on peut enlever}, 
deriving new SW based distance on manifolds could be of much interest. This question has led to several works in this direction, first on compact manifolds in \citep{rustamov2020intrinsic} and then in \citep{bonet2023spherical, quellmalz2023sliced} on the sphere.  Here, we focus on the particular case of Cartan-Hadamard manifolds which encompass in particular Euclidean spaces, Hyperbolic spaces \citep{bonet2022hyperbolic} or Symmetric Positive Definite matrices endowed with appropriate metrics \citep{bonet2023sliced}. %This work is an extension of \citep{bonet2022hyperbolic} and \citep{bonet2023sliced} as it provides a more general formalism\red{, additional examples and applications, and a study of Wasserstein gradient flows}.
% \nc{ce sera surtout à préciser dans la cover letter. pas de italique}

% \textbf{Contributions.} \looseness=-1 In this article, we start by providing some background on Optimal Transport and on Riemannian manifolds. Then, we introduce different ways to construct intrinsically Sliced-Wasserstein discrepancies on geodesically complete Riemannian manifolds with non-positive curvature (Cartan-Hadamard manifolds) by either using geodesic projections or horospherical projections. We specify the framework to different Cartan-Hadamard manifolds, including manifolds endowed with a pullback Euclidean metric, Hyperbolic spaces, Symmetric positive Definite matrices with specific metrics and product of Cartan-Hadamard manifolds. Then, we derive some theoretical properties common to any sliced discrepancy on these Riemannian manifolds, as well as properties specific to the pullback Euclidean case. We also propose illustrations for the Sliced-Wasserstein distance on the Euclidean space endowed with the Mahalanobis distance on a document classification task, and of the Sliced-Wasserstein distance on product manifolds for comparing datasets represented on the product space of the samples and of the labels. Finally, we propose non-parametric schemes to minimize these different distances using Wasserstein gradient flows, and hence allowing to derive new sampling algorithms on manifolds.

\textbf{Contributions.} \looseness=-1 In this article, we start in \Cref{section:background} by providing some background on Optimal Transport and on Riemannian manifolds. Then, in \Cref{section:irsw}, we introduce different ways to construct intrinsically Sliced-Wasserstein discrepancies on geodesically complete Riemannian manifolds with non-positive curvature (Cartan-Hadamard manifolds) by either using geodesic projections or horospherical projections. In \Cref{section:examples}, we specify the framework to different Cartan-Hadamard manifolds, including manifolds endowed with a pullback Euclidean metric, Hyperbolic spaces, Symmetric positive Definite matrices with specific metrics and product of Cartan-Hadamard manifolds. Then, in \Cref{section:chsw_properties}, we derive some theoretical properties common to any sliced discrepancy on these Riemannian manifolds, as well as properties specific to the pullback Euclidean case. We also propose in \Cref{section:applications} illustrations for the Sliced-Wasserstein distance on the Euclidean space endowed with the Mahalanobis distance on a document classification task, and of the Sliced-Wasserstein distance on product manifolds for comparing datasets represented on the product space of the samples and of the labels. Finally, we propose in \Cref{section:chswf} non-parametric schemes to minimize these different distances using Wasserstein gradient flows, and hence allowing to derive new sampling algorithms on manifolds\footnote{Code available at \url{https://github.com/clbonet/Sliced-Wasserstein_Distances_and_Flows_on_Cartan-Hadamard_Manifolds}}.

\section{Background} \label{section:background}

In this section, we first introduce background on Optimal Transport through the Wasserstein distance and the Sliced-Wasserstein distance on Euclidean spaces. Then, we introduce general Riemannian manifolds. For more details about Optimal Transport, we refer to \citep{villani2009optimal, santambrogio2015optimal, peyre2019computational}. And for more details about Riemannian manifolds, we refer to \citep{gallot1990riemannian,lee2006riemannian,lee2012smooth}.

\subsection{Optimal Transport on Euclidean Spaces} \label{section:ot}

\textbf{Wasserstein Distance.} Optimal transport provides a principled way to compare probability distributions through the Wasserstein distance. Let $p\ge 1$ and $\mu,\nu\in \mathcal{P}_p(\mathbb{R}^d) = \{\mu\in \mathcal{P}(\mathbb{R}),\ \int \|x\|_2^p\ \mathrm{d}\mu(x)<\infty\}$ two probability distributions with $p$ finite moments. Then, the Wasserstein distance is defined as 
\begin{equation}
    W_p^p(\mu,\nu) = \inf_{\gamma\in\Pi(\mu,\nu)}\ \int \|x-y\|_2^p\ \mathrm{d}\gamma(x,y),
\end{equation}
where $\Pi(\mu,\nu) = \{\gamma\in\mathcal{P}(\mathbb{R}^d\times\mathbb{R}^d),\ \pi^1_\#\gamma=\mu,\ \pi^2_\#\gamma=\nu\}$ denotes the set of couplings between $\mu$ and $\nu$, $\pi^1(x,y)=x$, $\pi^2(x,y)=y$ and $\#$ is the push-forward operator, defined as $T_\#\mu(A) = \mu(T^{-1}(A))$ for any Borel set $A\subset \mathbb{R}^d$ and map $T:\mathbb{R}^d\to\mathbb{R}^d$.

For discrete probability distributions with $n$ samples, \emph{e.g.} for $\mu=\frac{1}{n}\sum_{i=1}^n \delta_{x_i}$ and $\nu=\frac{1}{n}\sum_{j=1}^n \delta_{y_j}$ with $x_1,\dots,x_n,y_1,\dots,y_n\in\mathbb{R}^d$, computing $W_p^p$ requires to solve a linear program, which has a $O(n^3 \log n)$ worst case complexity \citep{pele2009fast}. Thus, it becomes intractable in large scale settings.

For unknown probability distributions $\mu$ and $\nu$ from which we have access to samples $x_1,\dots,x_n\sim\mu$ and $y_1,\dots,y_n\sim\nu$, a common practice to estimate $W_p^p(\mu,\nu)$ is to compute the plug-in estimator
\begin{equation}
    \widehat{W}_p^p(\mu,\nu) = W_p^p\left(\frac{1}{n}\sum_{i=1}^n \delta_{x_i}, \frac{1}{n}\sum_{i=1}^n \delta_{y_i}\right).
\end{equation}
However, the approximation error, known as the sample complexity, is quantified in $O(n^{-\frac{1}{d}})$ \citep{boissard2014mean}. Thus, the estimation of the Wasserstein distance from samples degrades in higher dimension if we use the same number of samples, or becomes very heavy to compute if we use enough samples to have the same approximation error.

To alleviate the computational burden and the curse of dimensionality, different variants were proposed. We focus in this work on the Sliced-Wasserstein distance.

\begin{figure}[t]
    \centering
    \hspace*{\fill}
    \subfloat{\label{fig:proj_sw}\includegraphics[width={0.3\linewidth}]{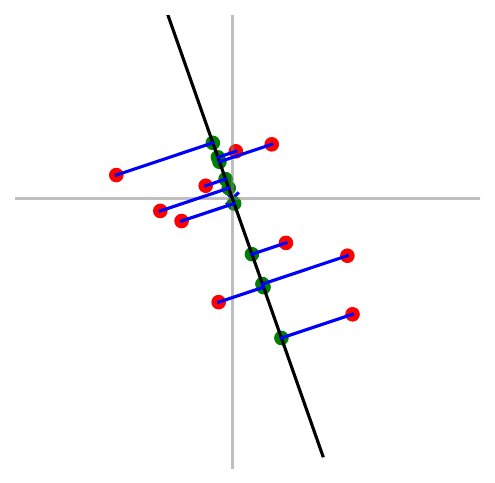}} \hfill
    \subfloat{\label{fig:lines_sw}\includegraphics[width={0.3\linewidth}]{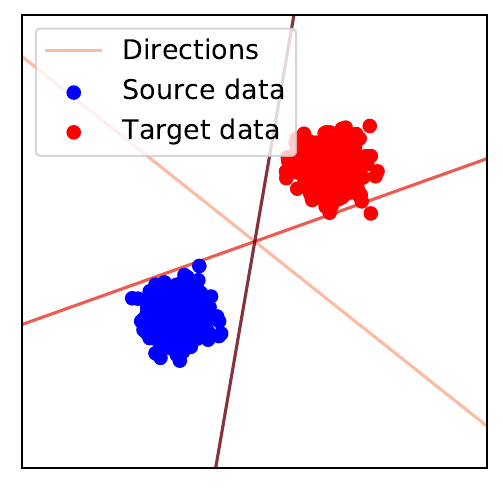}} \hfill
    \subfloat{\label{fig:densities_sw}\includegraphics[width={0.3\linewidth}]{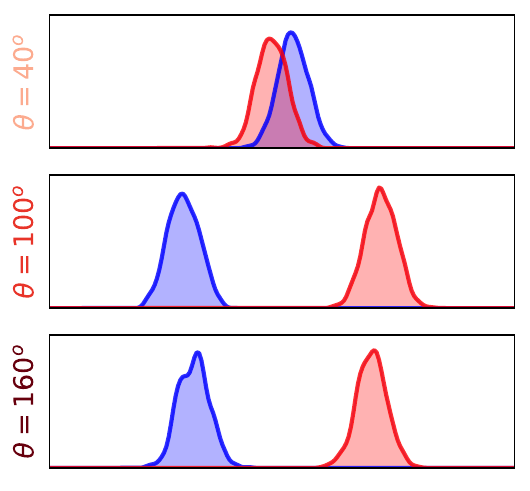}} \hfill
    \hspace*{\fill}
    \caption{(\textbf{Left}) Orthogonal projection of points on a line passing through the origin 0. (\textbf{Middle} and \textbf{Right}) Illustration of the projection of 2d distributions on 3 different lines.}
    \label{fig:illustration_sw}
    % \vspace{-10pt}
\end{figure}

\textbf{Sliced-Wasserstein Distance.} For $\mu,\nu\in\mathcal{P}_p(\mathbb{R})$, it is well known that the Wasserstein distance can be computed in closed-form \citep[Remark 2.30]{peyre2019computational}. More precisely, let $\mu,\nu\in\mathcal{P}_p(\mathbb{R})$, then 
\begin{equation}
    W_p^p(\mu,\nu) = \int_0^1 |F_{\mu}^{-1}(u) - F_{\nu}^{-1}(u)|^p\ \mathrm{d}u,
\end{equation}
where $F_{\mu}^{-1}$ and $F_{\nu}^{-1}$ denote the quantile functions of $\mu$ and $\nu$. For discrete distributions with $n$ samples, quantiles can be computed in $O(n\log n)$ since they only require sorting the samples. Thus, for $x_1<\dots<x_n$, $y_1<\dots<y_n$, $\mu_n = \frac{1}{n}\sum_{i=1}^n \delta_{x_i}$ and $\nu_n=\frac{1}{n}\sum_{i=1}^n \delta_{y_i}$,
\begin{equation}
    W_p^p(\mu_n, \nu_n) = \frac{1}{n} \sum_{i=1}^n |x_i - y_i|^p.
\end{equation}

Motivated by this closed-form, \citet{rabin2012wasserstein} introduced the Sliced-Wasserstein distance, which is defined by first projecting linearly the distributions on every possible direction, and then by taking the average of the one dimensional Wasserstein distances on each line. More precisely, for a direction $\theta\in S^{d-1}$, the coordinate of the orthogonal projection of $x\in\mathbb{R}^d$ on the line $\mathrm{span}(\theta)$ is defined by $P^\theta(x) = \langle x, \theta\rangle$. Then, by denoting by $\lambda$ the uniform measure on the sphere $S^{d-1}=\{\theta\in \mathbb{R}^d,\ \|\theta\|_2 = 1\}$, the $p$-Sliced-Wasserstein distance between $\mu,\nu\in\mathcal{P}_p(\mathbb{R}^d)$ is defined as 
\begin{equation}
    \sw_p^p(\mu,\nu) = \int_{S^{d-1}} W_p^p(P^\theta_\#\mu, P^\theta_\#\nu) \ \mathrm{d}\lambda(\theta).
\end{equation}
The projection process is illustrated in \Cref{fig:illustration_sw}.

Since the outer integral is intractable, a common practice to estimate this integral is to rely on a Monte-Carlo approximation by first sampling $L$ directions $\theta_1,\dots,\theta_L$ and then taking the average of the $L$ Wasserstein distances:
\begin{equation}
    \widehat{\sw}_p^p(\mu,\nu) = \frac{1}{L} \sum_{\ell=1}^L W_p^p(P^{\theta_\ell}_\#\mu, P^{\theta_\ell}_\#\nu).
\end{equation}
Thus, approximating SW requires to compute $L$ Wasserstein distances, and $Ln$ projections, and thus the computational complexity is in $O\big(Ln(\log n + d)\big)$. Note that other integral approximations have been recently proposed. For example, \citet{nguyen2023quasi} proposed to use quasi Monte-Carlo samples, and \citet{nguyen2023control, leluc2023speeding, leluc2024slicedwasserstein} used control variates to reduce the variance of the approximation.

We are now interested in transposing this method to Riemannian manifolds, for which we give a short introduction in the following section.

% \section{Background on Riemannian Manifolds}

% In this Section, we introduce some backgrounds on Riemannian manifolds. We refer to \citep{gallot1990riemannian,lee2006riemannian,lee2012smooth} for more details.

\subsection{Riemannian Manifolds} \label{sec:bg_riemnnian_manifolds}

\textbf{Definition.} A Riemannian manifold $(\mathcal{M}, g)$ of dimension $d$ is a space that behaves locally as a linear space diffeomorphic to $\mathbb{R}^d$, called a tangent space. To any $x\in \mathcal{M}$, one can associate a tangent space $T_x\mathcal{M}$ endowed with an inner product $\langle\cdot,\cdot\rangle_x : T_x\mathcal{M}\times T_x\mathcal{M}\to \mathbb{R}$ which varies smoothly with $x$. This inner product is defined by the metric $g_x$ associated to the Riemannian manifold as $g_x(u,v) = \langle u,v\rangle_x$ for any $x\in \mathcal{M}$, $u,v\in T_x\mathcal{M}$. We note $G(x)$ the matrix representation of $g_x$ defined such that
\begin{equation}
    \forall u,v \in T_x\mathcal{M},\ \langle u,v\rangle_x = g_x(u,v) = u^T G(x) v.
\end{equation}
For some spaces, different metrics can give very different geometries. We call tangent bundle the disjoint union of all tangent spaces $T\mathcal{M} = \{(x,v),\ x\in\mathcal{M} \text{ and } v\in T_x\mathcal{M}\}$, and we call a vector field a map $V:\mathcal{M}\to T\mathcal{M}$ such that $V(x)\in T_x\mathcal{M}$ for all $x\in\mathcal{M}$.

\textbf{Geodesics.} A generalization of straight lines in Euclidean spaces to Riemannian manifolds can be geodesics, which are smooth curves connecting two points $x,y\in\mathcal{M}$ with minimal length, \emph{i.e.} curves $\gamma:[0,1]\to \mathcal{M}$ such that $\gamma(0)=x$, $\gamma(1)=y$, and which minimize the length $\mathcal{L}$ defined as
\begin{equation}
    \mathcal{L}(\gamma) = \int_0^1 \|\gamma'(t)\|_{\gamma(t)}\ \mathrm{d}t,
\end{equation}
where $\|\gamma'(t)\|_{\gamma(t)} = \sqrt{\langle \gamma'(t), \gamma'(t)\rangle_{\gamma(t)}}$. In this work, we will focus on geodesically complete Riemannian manifolds, in which case there is always a geodesic between two points $x,y\in\mathcal{M}$. Furthermore, in this specific case, all geodesics are actually geodesic lines, \emph{i.e.} they can be extended to $\mathbb{R}$. Let $x,y\in \mathcal{M}$, $\gamma:[0,1]\to \mathcal{M}$ a geodesic between $x$ and $y$ such that $\gamma(0)=x$ and $\gamma(1)=y$, then the value of the length defines actually a distance $(x,y)\mapsto d(x,y)$ between $x$ and $y$, which we call the geodesic distance:
\begin{equation}
    d(x,y) = \inf_{\gamma,\ \gamma(0)=x,\ \gamma(1)=y}\ \mathcal{L}(\gamma). %\inf_\gamma \ \int_0^1 \|\gamma'(t)\|_{\gamma(t)}\ \mathrm{d}t.
\end{equation}
% Note that for a geodesic $\gamma$ between $x$ and $y$, we have for any $s,t\in [0,1]$, $d\big(\gamma(t), \gamma(s)\big) = |t-s| d(x,y)$. \nc{enlever cette phrase ?} And this is true for any $s,t\in\mathbb{R}$ for geodesic lines. 

\textbf{Exponential Map.} Let $x\in\mathcal{M}$, then for any $v\in T_x\mathcal{M}$, there exists a unique geodesic $\gamma_{(x,v)}$ starting from $x$ with velocity $v$, \emph{i.e.} such that $\gamma_{(x,v)}(0) = x$ and $\gamma_{(x,v)}'(0) = v$ \citep{sommer2020introduction}. Now, we can define the exponential map as $\exp:T\mathcal{M}\to \mathcal{M}$ which for any $x\in \mathcal{M}$, maps tangent vectors $v\in T_x\mathcal{M}$ back to the manifold at the point reached by the geodesic $\gamma_{(x,v)}$ at time $t=1$:
\begin{equation}
    \forall (x,v)\in T\mathcal{M},\ \exp_x(v) = \gamma_{(x,v)}(1).
\end{equation}
On geodesically complete manifolds, the exponential map is defined on the entire tangent space, but is not necessarily a bijection. When it is the case, we note $\log_x$ the inverse of $\exp_x$, which allows mapping elements from the manifold to the tangent space. We illustrate these different notions on \Cref{fig:illustration_manifold}.

Let $f:\mathcal{M}\to \mathbb{R}$ be a differentiable map. We now define its Riemannian gradient, which is notably very important in order to generalize first-order optimization algorithms to Riemannian manifolds \citep{bonnabel2013stochastic, boumal2023introduction}.
\begin{definition}[Gradient]
    We define the Riemannian gradient of $f$ as the unique vector field $\mathrm{grad}_\mathcal{M}f:\mathcal{M}\to T\mathcal{M}$ satisfying
    \begin{equation}
        \forall (x,v)\in T\mathcal{M},\ \frac{\mathrm{d}}{\mathrm{d}t}f\big(\exp_x(tv)\big)\Big|_{t=0} = \langle v, \mathrm{grad}_\mathcal{M}f(x)\rangle_x.
    \end{equation}
\end{definition}

\begin{figure}
    \centering
    \includegraphics[width=0.5\linewidth]{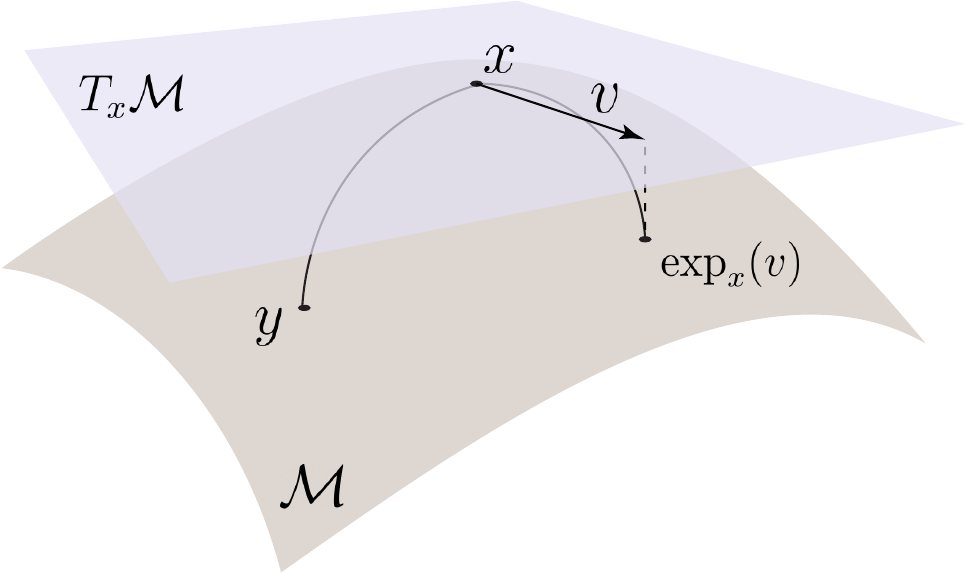}
    \caption{Illustration of geodesics, of the tangent space and the exponential map on a Riemannian manifold.}
    \label{fig:illustration_manifold}
\end{figure}

%% injectivity radius ? (pas besoin car on se focus sur geodesically complete)

\textbf{Sectional Curvature.} A notion which allows studying the geometry as well as the topology of a given Riemannian manifold is the sectional curvature. Let $x\in\mathcal{M}$, and $u,v\in T_x\mathcal{M}$ two linearly independent vectors. Then, the sectional curvature $\kappa_x(u,v)$ is defined geometrically as the Gaussian curvature of the plane $E=\mathrm{span}(u,v)$ \citep{zhang2016riemannian}, \emph{i.e.}
\begin{equation}
    \kappa_x(u,v) = \frac{\langle R(u,v)u, v\rangle_x}{\langle u,u\rangle_x\langle v,v\rangle_x - \langle u,v\rangle_x^2},
\end{equation}
where $R$ is the Riemannian curvature tensor. We refer to \citep{lee2006riemannian} for more details. The behavior of geodesics changes given the curvature of the manifold. For instance, they usually diverge on manifolds of negative sectional curvature and converge on manifolds of positive sectional curvature \citep{hu2023riemannian}. Important examples of Riemannian manifolds are Euclidean spaces which are of constant null curvature, the sphere which is of positive constant curvature and Hyperbolic spaces which are of negative constant curvature (\emph{i.e.} have the same value at any point $x\in\mathcal{M}$ and for any 2-planes $E$) with their standard metrics. We can also cite the torus endowed with the ambient metric which has some points of positive curvature, some points of negative curvature and some points of null curvature \citep{borde2023latent}. In this paper, we will mostly focus on Cartan-Hadamard manifolds which are complete connected Riemannian manifolds of non-positive sectional curvature.

\subsection{Probability Distributions on Riemannian Manifolds}

\textbf{Probability Distributions.} Let $(\mathcal{M}, g)$ be a Riemannian manifold. For $x\in \mathcal{M}$, $G(x)$ induces an infinitesimal change of volume on the tangent space $T_x\mathcal{M}$, and thus a measure on the manifold,
\begin{equation}
    \mathrm{d}\vol(x) = \sqrt{|G(x)|}\ \mathrm{d}x.
\end{equation}
Here, we denote by $\mathrm{d}x$ the Lebesgue measure. We refer to \citep{pennec2006intrinsic} for more details on distributions on manifolds. 
% \red{Now, we discuss some possible distributions on Riemannian manifolds, which can be seen as generalizations of Gaussian distributions.} % or be used to define new distributions.

%% WND, Riemannian normal
% \red{The first way of naturally generalizing Gaussian distributions to Riemannian manifolds is to use the geodesic distance in the density, which becomes
% \begin{equation}
%     f(x) \propto \exp\left(-\frac{1}{2\sigma^2}d(x,\mu)^2\right),
% \end{equation}
% for $\mu\in \mathcal{M}$, $\sigma\in \mathbb{R}$. This was first introduced in \citep{pennec2006intrinsic} and then further considered and theoretically studied on particular Riemannian manifolds in \citep{said2017riemannian, said2017gaussian}. Notably, an important property required to use such a density is that the normalization factor must not depend on the mean parameter $\mu$, which might not always be the case. In particular, it holds on Riemannian symmetric spaces \citep{said2017riemannian}. However, it is not straightforward to sample from such a distribution.}

\looseness=-1 Particularly interesting examples of probability distributions are wrapped distributions \citep{chevallier2020wrapped,chevallier2022exponential,galaz2022wrapped}, which are defined as the push-forward of a distribution $\mu\in \mathcal{P}(T_x\mathcal{M})$ onto $\mathcal{P}(\mathcal{M})$ using \emph{e.g.} the exponential map when it is invertible over the whole tangent space. As it gives a very convenient way to sample on manifolds, this has received much attention notably on hyperbolic spaces with the wrapped normal distribution \citep{nagano2019wrapped, cho2022rotated}, for which the distribution in the tangent space is a Gaussian, and for which all transformations are differentiable, and can be used \emph{e.g.} for variational autoencoders since they are amenable to the reparametrization trick.

% which samples first from a Gaussian in the tangent space, as it gives a very convenient way to sample on the manifold, while all transformations are differentiable, and can hence be used in variational autoencoders for instance.

% \red{More convenient distributions, on which we can use the reparameterization trick, are wrapped distributions \citep{chevallier2020wrapped,chevallier2022exponential,galaz2022wrapped}. The idea is to push-forward a distribution $\mu\in\mathcal{P}(T_x\mathcal{M})$ onto $\mathcal{P}(\mathcal{M})$. A natural function to use is the exponential map when it is invertible over the whole tangent space. This has received much attention \emph{e.g.} on hyperbolic spaces with the wrapped normal distribution \citep{nagano2019wrapped, cho2022rotated}, which samples from a Gaussian in the tangent space, as it gives a very convenient way to sample on the manifold, while all transformations are differentiable, and can hence be used in variational autoencoders for instance.}

% \red{Another solution to sample on a manifold is to condition the samples to belong to the manifold. When restricting an isotropic distribution to lie on the unit sphere, this gives for example the well-known von Mises-Fisher distribution \citep{hauberg2018directional}.}

\textbf{Optimal Transport.} Optimal Transport is also well defined on Riemannian manifolds using appropriate ground costs into the Kantorovich problem. Using the geodesic distance at the power $p\ge 1$, we recover the $p$-Wasserstein distance \citep{mccann2001polar, villani2009optimal}
\begin{equation}
    W_p^p(\mu,\nu) = \inf_{\gamma\in \Pi(\mu,\nu)}\ \int_{\mathcal{M}\times \mathcal{M}} d(x,y)^p\ \mathrm{d}\gamma(x,y),
\end{equation}
where $\mu,\nu\in\mathcal{P}_p(\mathcal{M}) = \{\mu\in \mathcal{P}(\mathcal{M}),\ \int_\mathcal{M} d(x,o)^p\ \mathrm{d}\mu(x)<\infty\}$, with $o\in \mathcal{M}$ some origin which can be arbitrarily chosen (because of the triangular inequality).

% \nc{rajouter une ref ou deux sur transport + riemannien}
% This problem has received much attention, see \emph{e.g.} \citep{villani2009optimal, bianchini2011optimal}. In particular, Brenier's theorem was extended by \citet{mccann2001polar} on Riemannian manifolds. For $\mu,\nu \in\mathcal{P}_2(\mathcal{M})$ when the source measure $\mu$ is absolutely continuous \emph{w.r.t} the volume measure on $\mathcal{M}$, then there exists a unique OT map $T$ such that $T_\#\mu=\nu$ and $T$ is given by, for $\mu$-almost every $x\in \mathcal{M}$, $T(x) = \exp_x\big(-\mathrm{grad}_\mathcal{M} \psi(x)\big)$ with $\psi$ a c-concave map.

\section{Riemannian Sliced-Wasserstein} \label{section:irsw}

\looseness=-1 In this section, we propose natural generalizations of the Sliced-Wasserstein distance on probability distributions supported on Riemannian manifolds by using tools intrinsically defined on them. To do that, we will first consider the Euclidean space as a Riemannian manifold. Doing so, we will be able to generalize it naturally to Riemannian manifolds of non-positive curvature. %Then, we will discuss some related works. % challenges inherent to Riemannian manifolds with positive curvatures and corresponding related works. 
The proofs of this section are postponed to \Cref{proofs:section_irsw}.

\subsection{Euclidean Sliced-Wasserstein as a Riemannian Sliced-Wasserstein Distance} \label{sec:sw_esw_rsw}

It is well known that the Euclidean space can be viewed as a Riemannian manifold of null constant curvature \citep{lee2006riemannian}. From that point of view, we can translate the elements used to build the Sliced-Wasserstein distance as Riemannian elements, and identify how to generalize it to more general Riemannian manifolds.

First, let us recall that the $p$-Sliced-Wasserstein distance for $p\ge 1$ between $\mu,\nu\in\mathcal{P}_p(\mathbb{R}^d)$ is defined as
\begin{equation}
    \sw_p^p(\mu,\nu) = \int_{S^{d-1}} W_p^p(P^\theta_\#\mu, P^\theta_\#\nu)\ \mathrm{d}\lambda(\theta),
\end{equation}
where $P^\theta(x) = \langle x, \theta\rangle$ and $\lambda$ is the uniform distribution $S^{d-1}$. Geometrically, it amounts to projecting the distributions on every possible line going through the origin $0$. Hence, we see that we need first to generalize lines passing through the origin. Then, we need to find suitable projections on these subsets. Finally, we need to make sure that we are still able to compute the Wasserstein distance efficiently between distributions supported on these subsets, in order to still have a computational advantage over solving the linear program.

\textbf{Lines.} From a Riemannian manifold point of view, straight lines can be seen as geodesics, which are, as we saw in \Cref{sec:bg_riemnnian_manifolds}, curves minimizing the distance between any two points on it. For any direction $\theta\in S^{d-1}$, the geodesic passing through $0$ in direction $\theta$ is described by the curve $\gamma_\theta:\mathbb{R} \to \mathbb{R}^d$ defined as $\gamma_\theta(t) = t \theta = \exp_0(t\theta)$ for any $t\in \mathbb{R}$, and the corresponding geodesic is $\mathcal{G}^\theta = \mathrm{span}(\theta)$. Hence, when it makes sense, a natural generalization to straight lines would be to project on geodesics passing through an origin.

\textbf{Projections.} The projection $P^\theta(x)$ of $x\in\mathbb{R}^d$ can be seen as the coordinate of the orthogonal projection on the geodesic $\mathcal{G}^\theta$. Indeed, the orthogonal projection $\Tilde{P}$ is formally defined as 
\begin{equation}
    \Tilde{P}^\theta(x) = \argmin_{y \in \mathcal{G}^\theta}\ \|x-y\|_2 = \langle x, \theta\rangle\theta.
\end{equation}
From this formulation, we see that $\Tilde{P}^\theta$ is a metric projection, which can also be called a geodesic projection on Riemannian manifolds as the metric is a geodesic distance. 
Then, we see that its coordinate on $\mathcal{G}^\theta$ is $t=\langle x,\theta\rangle = P^\theta(x)$, which can be also obtained by first giving a direction to the geodesic, and then computing the distance between $\Tilde{P}^\theta(x)$ and the origin $0$, as
\begin{equation}
    P^\theta(x) = \sign(\langle x,\theta\rangle)\|\langle x,\theta\rangle \theta - 0 \|_2 = \langle x,\theta\rangle.
\end{equation}
Note that this can also be recovered by solving
\begin{equation}
    P^\theta(x) = \argmin_{t \in \mathbb{R}} \ \|\exp_0(t\theta) - x\|_2.
\end{equation}
This formulation will be useful to generalize it to more general manifolds by replacing the Euclidean distance by the right geodesic distance.

Note also that the geodesic projection can be seen as a projection along hyperplanes, \emph{i.e.} the level sets of the projection function $g(x,\theta) = \langle x, \theta\rangle$ are (affine) hyperplanes. This observation will come useful in generalizing SW to manifolds of non-positive curvature.

\textbf{Wasserstein Distance.} The Wasserstein distance between measures lying on the real line has a closed-form which can be computed very easily (see \Cref{section:ot}). On more general Riemannian manifolds, as the geodesics will not necessarily be lines, we will need to check how to compute the Wasserstein distance between the projected measures.

\subsection{On Manifolds of Non-Positive Curvature}

In this part, we focus on complete connected Riemannian manifolds of non-positive curvature, which can also be called Hadamard manifolds or Cartan-Hadamard manifolds \citep{lee2006riemannian, robbin2011introduction, lang2012fundamentals}. These spaces actually include Euclidean spaces, but also spaces with constant negative curvature such as Hyperbolic spaces, or with variable non-positive curvatures such as the space of Symmetric Positive Definite matrices and product of manifolds with constant negative curvature \citep[Lemma 1]{gu2019learning}. We refer to \citep{ballmann2006manifolds} or \citep{bridson2013metric} for more details. These spaces share many properties with Euclidean spaces \citep{bertrand2012geometric} which make it possible to extend the Sliced-Wasserstein distance on them. We will denote $(\mathcal{M}, g)$ a Hadamard manifold in the following.
Particular cases such as Hyperbolic spaces or the space of Symmetric Positive Definite matrices among other will be further studied in \Cref{section:examples}.
% The particular cases of Hyperbolic spaces and the spaces of Symmetric Positive Definite matrices will be further studied respectively in \Cref{chapter:hsw} and \Cref{chapter:spdsw}.

\textbf{Properties of Hadamard Manifolds.} First, as a Hadamard manifold is a complete connected Riemannian manifold, then by the Hopf-Rinow theorem \citep[Theorem 6.13]{lee2006riemannian}, it is also geodesically complete. Therefore, any geodesic curve $\gamma:[0,1]\to\mathcal{M}$ connecting $x\in \mathcal{M}$ to $y\in \mathcal{M}$ can be extended on $\mathbb{R}$ as a geodesic line. Furthermore, by Cartan-Hadamard theorem \citep[Theorem 11.5]{lee2006riemannian}, Hadamard manifolds are diffeomorphic to the Euclidean space $\mathbb{R}^d$, and the exponential map at any $x\in\mathcal{M}$ from $T_x\mathcal{M}$ to $\mathcal{M}$ is bijective with the logarithm map as inverse. Moreover, their injectivity radius is infinite and thus, its geodesics are aperiodic, and can be mapped to the real line, which will allow to find coordinates on the real line, and hence to compute the Wasserstein distance between the projected measures efficiently. The SW discrepancy on such spaces is therefore very analogous to the Euclidean case. Note that Hadamard manifolds belong to the more general class of CAT(0) metric spaces, and hence inherit their properties described in \citep{bridson2013metric}. Now, let us discuss two different possible projections, which both generalize the Euclidean orthogonal projection.

\textbf{Geodesic Projections.} As we saw in \Cref{sec:sw_esw_rsw}, a natural projection on geodesics is the geodesic projection. Let's note $\mathcal{G}$ a geodesic passing through an origin point $o\in \mathcal{M}$. Such origin will often be taken naturally on the space, and corresponds to the analog of $0$ in $\mathbb{R}^d$. Then, the geodesic projection on $\mathcal{G}$ is obtained naturally as
\begin{equation}
    \forall x\in \mathcal{M},\ \Tilde{P}^\mathcal{G}(x) = \argmin_{y\in \mathcal{G}}\ d(x, y).
\end{equation}
From the projection, we can get a coordinate on the geodesic by first giving it a direction and then computing the distance to the origin. By noting $v\in T_o\mathcal{M}$ a vector in the tangent space at the origin, such that $\mathcal{G} = \mathcal{G}^v = \{\exp_o(tv),\ t\in\mathbb{R}\}$, we can give a direction to the geodesic by computing the sign of the inner product in the tangent space of $o$ between $v$ and the log of $\Tilde{P}^\mathcal{G}$. Analogously to the Euclidean case, we can restrict $v$ to be of unit norm, \emph{i.e.} $\|v\|_o=1$. Now, we will denote the projection and coordinate projection on $\mathcal{G}^v$ respectively as $\Tilde{P}^v$ and $P^v$. %\red{to indicate the dependence with respect to the geodesic $\mathcal{G}^v$, we will use $v$ in exponent of $\Tilde{P}$ and $P$.} %we will use $v$ in index \nc{exponent ? rephrase} of $\Tilde{P}$ and $P$ instead of $\mathcal{G}$. 
Hence, we obtain the coordinates using
\begin{equation}
    P^v(x) = \sign\big(\langle \log_o\big(\Tilde{P}^v(x)\big), v\big\rangle_o\big)\ d\big(\Tilde{P}^v(x), o\big).
\end{equation}
We show in \Cref{prop:isometry} that the map $t^v: \mathcal{G}^v \to \mathbb{R}$ defined as
\begin{equation} \label{eq:coordmap}
    \forall x\in \mathcal{G}^v,\ t^v(x) = \sign\big(\langle \log_o(x), v\rangle_o\big)\ d(x, o),
\end{equation}
is an isometry.
\begin{proposition} \label[proposition]{prop:isometry}
    Let $(\mathcal{M}, g)$ be a Hadamard manifold with origin $o$. Let $v\in T_o\mathcal{M}$, then, the map $t^v$ defined in \eqref{eq:coordmap} is an isometry from $\mathcal{G}^v = \{\exp_o(tv),\ t\in\mathbb{R}\}$ to $\mathbb{R}$.
\end{proposition}
% \begin{proof}
%     See \Cref{proof:prop_isometry}.
% \end{proof}
Note that to get directly the coordinate from $x\in \mathcal{M}$, we can also solve directly the following problem:
\begin{equation} \label{eq:geod_proj}
    P^v(x) = \argmin_{t\in \mathbb{R}}\ d\big(\exp_o(tv), x\big).
\end{equation}

% \red{Add prop to show that it is well defined? \emph{i.e.} minimise fonction geodesically convexe sur un ensemble convexe ?}

Using that Hadamard manifolds belong to the more general class of CAT(0) metric spaces, by \citep[II. Proposition 2.2]{bridson2013metric}, the geodesic distance is geodesically convex. Hence, minimizing the function $t\mapsto d\big(\exp_o(tv), x\big)^2$ is a coercive strictly-convex problem, and thus admits a unique solution. Therefore, \eqref{eq:geod_proj} is well defined. Moreover, we have the following characterization for the optimum:
\begin{proposition} \label[proposition]{prop:charac_geod_proj}
    Let $(\mathcal{M},g)$ be a Hadamard manifold with origin $o$. Let $v\in T_o\mathcal{M}$, and note $\gamma(t) = \exp_o(tv)$ for all $t\in\mathbb{R}$. Then, for any $x\in\mathcal{M}$,
    \begin{equation}
        P^v(x) = \argmin_{t\in\mathbb{R}}\ d(\gamma(t), x)^2 \iff \left\langle \gamma'\big(P^v(x)\big), \log_{\gamma\big(P^v(x)\big)}(x)\right\rangle_{\gamma\big(P^v(x)\big)} = 0.
    \end{equation}
\end{proposition}

% \begin{proof}
%     \notsure{See \Cref{proof:prop_charac_geod_proj}.}
% \end{proof}

In the Euclidean case $\mathbb{R}^d$, as geodesics are of the form $\gamma(t) = t\theta$ for any $t\in\mathbb{R}$ and for a direction $\theta\in S^{d-1}$, and as $\log_x(y) = y-x$ for $x,y\in\mathbb{R}^d$, we recover the projection formula:
\begin{equation}
    \left\langle \gamma'\big(P^\theta(x)\big), \log_{\gamma\big(P^\theta(x)\big)}(x)\right\rangle_{\gamma\big(P^\theta(x)\big)} = 0 \iff \langle \theta, x-P^\theta(x)\theta\rangle = 0 \iff P^\theta(x) = \langle \theta, x\rangle.
\end{equation}

\textbf{Busemann Projections.} \looseness=-1 The level sets of the geodesic projections are geodesic subspaces. It has been shown that projecting along geodesics is not always the best solution as it might not preserve distances between the original points \citep{chami2021horopca}. Indeed, on Euclidean spaces, as mentioned earlier, the projections are actually along hyperplanes, which preserve the distances between points belonging to another geodesic with the same direction (see \Cref{fig:parallel_proj}). On Hadamard manifolds, there are analogs of hyperplanes, which can be obtained through the level sets of the Busemann function which we introduce now.

Let $\gamma:\mathbb{R}\to\mathcal{M}$ be a geodesic line, then the Busemann function associated to $\gamma$ is defined as \citep[II. Definition 8.17]{bridson2013metric}
\begin{equation}
    \forall x\in\mathcal{M},\ B^\gamma(x) = \lim_{t\to\infty}\ \left(d\big(x,\gamma(t)\big) - t\right).
\end{equation}
On Hadamard manifolds, and more generally on CAT(0) spaces% with $\gamma$ a geodesic ray
, the limit exists \citep[II. Lemma 8.18]{bridson2013metric}. This function returns a coordinate on the geodesic $\gamma$, which can be understood as a normalized distance to infinity towards the direction given by $\gamma$ \citep{chami2021horopca}. The level sets of this function are called horospheres. On spaces of constant curvature (\emph{i.e.} Euclidean or Hyperbolic spaces), horospheres are of constant null curvature and hence very similar to hyperplanes. We illustrate horospheres in Hyperbolic spaces in the middle of \Cref{fig:parallel_proj} and in \Cref{fig:projections_hsw}.

For example, in the Euclidean case, we can show that the Busemann function associated to $\mathcal{G}^\theta = \mathrm{span}(\theta)$ for $\theta\in S^{d-1}$ is given by
\begin{equation}
    \forall x\in \mathbb{R}^d,\ B^\theta(x) = -\langle x, \theta\rangle.
\end{equation}
It actually coincides, up to a sign, with the inner product, which can be seen as a coordinate on the geodesic $\mathcal{G}^\theta$. Moreover, its level sets in this case are (affine) hyperplanes orthogonal to $\theta$.

Hence, the Busemann function provides a principled way to project measures on a Hadamard manifold to the real line, provided that we can compute its closed-form. To find the projection on the geodesic $\gamma$, we can solve the equation in $s\in\mathbb{R}$, $B^\gamma(x) = B^\gamma\big(\gamma(s)\big) = -s$, and we find that the projection on the geodesic $\gamma$ characterized by $v\in T_o\mathcal{M}$ such that $\|v\|_o=1$ and $\gamma(t) = \exp_o(tv)$ for all $t\in\mathbb{R}$ is
\begin{equation} \label{eq:busemann_proj}
    \Tilde{B}^v(x) = \exp_o\big(-B^v(x) v\big).
\end{equation}
% $\Tilde{B}^\gamma(x)=\exp_o\big(- B^\gamma(x) v\big)$ if $\gamma(t) = \exp_o(tv)$.

% Furthermore, the Busemann function enjoys many useful properties such as being convex or being 1-Lipschitz \citep[II. Proposition 8.22]{bridson2013metric}.

\begin{figure}[t]
    \centering
    \hspace*{\fill}
    \subfloat{\label{fig:parallel_proj_euc}\includegraphics[width={0.25\linewidth}]{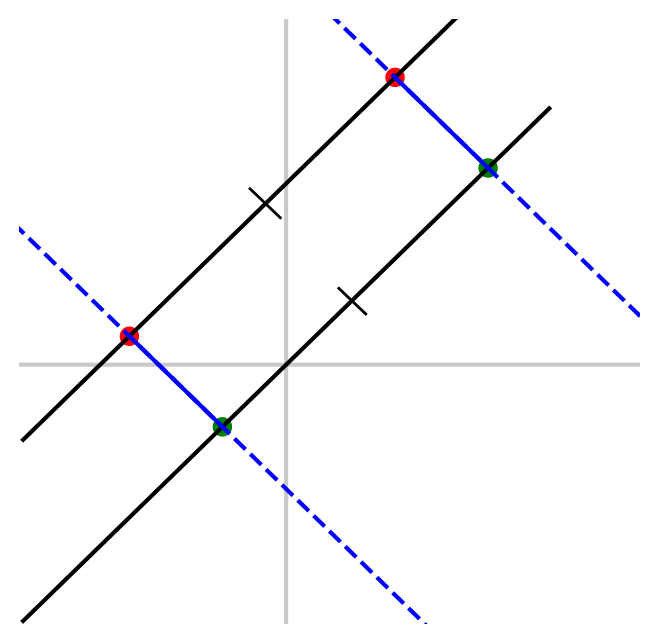}} \hfill
    \subfloat{\label{fig:parallel_proj_horo}\includegraphics[width={0.25\linewidth}]{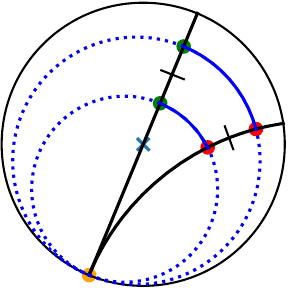}} \hfill
    \subfloat{\label{fig:parallel_proj_geod}\includegraphics[width={0.25\linewidth}]{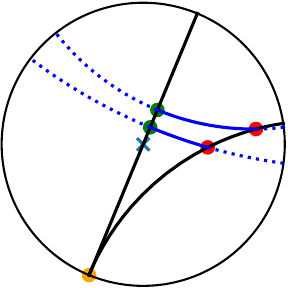}} \hfill
    \hspace*{\fill}
    \caption{({\bf Left}) On Euclidean spaces, the distance between the projections of two points belonging to a geodesic with the same direction is conserved. ({\bf Middle}) On Hyperbolic spaces, this is also the case using the horospherical projection as demonstrated in \citep[Proposition 3.4]{chami2021horopca}, but not for geodesic projections ({\bf Right}).}
    \label{fig:parallel_proj}
    % \vspace{-10pt}
\end{figure}

\textbf{Wasserstein Distance on Geodesics.} We saw that we can obtain projections on $\mathbb{R}$. Hence, it is analogous to the Euclidean case as we can use the one dimensional Wasserstein distance on the real line to compute it. In \Cref{prop:eq_wasserstein}, as a sanity check, we verify that the Wasserstein distance between the coordinates (on $\mathcal{P}_p(\mathbb{R})$) is as expected equal to the Wasserstein distance between the measures projected on geodesics (on $\mathcal{P}_p(\mathcal{M})$). This relies on the isometry property of $t^v$ derived in \Cref{prop:isometry}.

\begin{proposition} \label[proposition]{prop:eq_wasserstein}
    Let $(\mathcal{M},g)$ a Hadamard manifold, $p\ge 1$ and $\mu,\nu \in \mathcal{P}_p(\mathcal{M})$. Let $v\in T_o\mathcal{M}$ such that $\|v\|_o=1$ and $\mathcal{G}^v=\{\exp_o(tv),\ t\in\mathbb{R}\}$ the geodesic on which the measures are projected. Then,
    \begin{align}
        &W_p^p(\Tilde{P}^v_\#\mu, \Tilde{P}^v_\#\nu) = W_p^p(P^v_\#\mu, P^v_\#\nu), \\
        &W_p^p(\Tilde{B}^v_\#\mu, \Tilde{B}^v_\#\nu) = W_p^p(B^v_\#\mu, B^v_\#\nu),
    \end{align}
    where the Wasserstein distances are defined with the corresponding geodesic distance given the space, \emph{i.e.} with $d(x,y)$ the geodesic distance on $\mathcal{M}$ for the $W_p$ on the left, and $|t-s|$ for $W_p$ on the right.
\end{proposition}

From these properties, we can work equivalently in $\mathbb{R}$ and on the geodesics when using the Busemann projection (also called horospherical projection) or the geodesic projection of measures.

\begin{figure}
    \centering
    \includegraphics[width=0.7\linewidth]{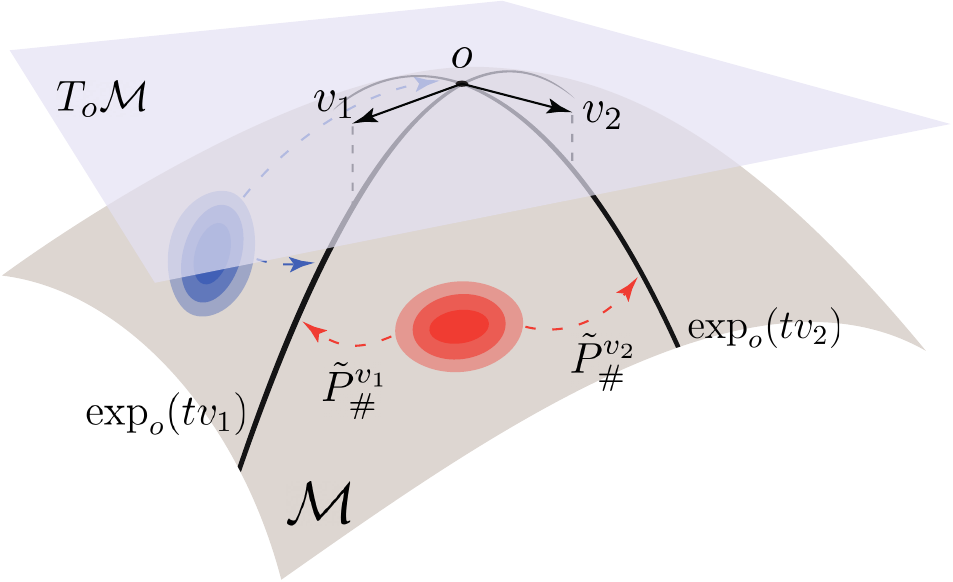}
    \caption{Illustration of the projection process of measures on geodesics $t\mapsto \exp_o(tv_1)$ and $t\mapsto \exp_o(tv_2)$.}
    \label{fig:illustration_rsw}
\end{figure}

\textbf{Sliced-Wasserstein on Hadamard Manifolds.} We are ready to define the Sliced-Wasserstein distance on Hadamard manifolds. For directions, we will sample from the uniform measure on $S_o = \{v\in T_o\mathcal{M},\ \|v\|_o=1\}$. Note that other distributions might be used such as a Dirac in the maximum direction similarly as max-SW \citep{deshpande2019max} for example or any variant using different slicing distributions such as in \citep{nguyen2020distributional, nguyen2020improving, nguyen2023energy, ohana2022shedding}. However, in order to define a strict generalization of SW, we choose to focus on the uniform one in this work. %Here, the projection $P^v$ will either be the geodesic projection or the Busemann projection.

% \begin{definition}[Cartan-Hadamard Sliced-Wasserstein]
%     Let $(\mathcal{M},g)$ a Hadamard manifold with $o$ its origin. Denote $\lambda$ the uniform distribution on $S_o=\{v\in T_o\mathcal{M},\ \|v\|_o=1\}$. Denote $P^v:\mathcal{M}\to\mathbb{R}$ a projection as previously discussed. Let $p\ge 1$, then we define the $p$-Cartan-Hadamard Sliced-Wasserstein distance between $\mu,\nu\in\mathcal{P}_p(\mathcal{M})$ as
%     \begin{equation}
%         \mathrm{CHSW}_p^p(\mu,\nu) = \int_{S_o} W_p^p(P^v_\#\mu, P^v_\#\nu)\ \mathrm{d}\lambda(v).
%     \end{equation}
% \end{definition}
% In practice, when we use the geodesic projection, we will call it Geodesic Cartan-Hadamard Sliced-Wasserstein $\mathrm{GCHSW}$, and when we use the horospherical projection, we will call it the Horospherical Cartan-Hadamard Sliced-Wasserstein $\mathrm{HCHSW}$. When we deal with both at the same time, we just use $\mathrm{CHSW}$.

\begin{definition}[Cartan-Hadamard Sliced-Wasserstein]
    Let $(\mathcal{M},g)$ a Hadamard manifold with $o$ its origin. Denote $\lambda_o$ the uniform distribution on $S_o=\{v\in T_o\mathcal{M},\ \|v\|_o=1\}$. Let $p\ge 1$, then we define the $p$-Geodesic Cartan-Hadamard Sliced-Wasserstein distance between $\mu,\nu\in\mathcal{P}_p(\mathcal{M})$ as
    \begin{equation}
        \gchsw_p^p(\mu,\nu) = \int_{S_o} W_p^p(P^v_\#\mu, P^v_\#\nu)\ \mathrm{d}\lambda_o(v).
    \end{equation}
    Likewise, we define the $p$-Horospherical Cartan-Hadamard Sliced-Wasserstein distance between $\mu,\nu\in\mathcal{P}_p(\mathcal{M})$ as
    \begin{equation}
        \hchsw_p^p(\mu,\nu) = \int_{S_o} W_p^p(B^v_\#\mu, B^v_\#\nu)\ \mathrm{d}\lambda_o(v).
    \end{equation}
\end{definition}
In the following, when we want to mention both $\gchsw$ and $\hchsw$, for example for properties satisfied by both, we will use the term Cartan-Hadamard Sliced-Wasserstein abbreviated as $\chsw$. Then, we will write without loss of generality
\begin{equation}
    \chsw_p^p(\mu, \nu) = \int_{S_o} W_p^p(P^v_\#\mu, P^v_\#\nu)\ \mathrm{d}\lambda_o(v),
\end{equation}
with $P^v$ either denoting the geodesic or the horospherical projection. We illustrate the projection process on \Cref{fig:illustration_rsw}.

\subsection{Related Works}

% \red{GSW + defining functions + Radon transforms, \citep{rustamov2020intrinsic}}

% \red{restricted to compact spaces + problem computation of spectrum impossible but in very few cases \citep[p 17]{pennec2006intrinsic} + instable in high dimension \citep[Appendix A]{dutordoir2020sparse}}

%% GSW, Rustamov

% \red{Mettre moins de détails, notamment sur ISW ?}
% \nc{Ok avec le niveau de détails, mais ce qui m'embête c'est le dernier paragraphe ou on a l'impression que c'est lié à Intrinsic slcied alors que c'est majoritairement de la sphère. Rajouter un titre comme proposé ?}
\textbf{Intrinsic Sliced-Wasserstein.} To the best of our knowledge, the first attempt to define a generalization of the Sliced-Wasserstein distance on Riemannian manifolds was made by \citet{rustamov2020intrinsic}. In this work, they restricted their analysis to compact spaces and proposed to use the eigendecomposition of the Laplace-Beltrami operator (see \citep[Definition 4.7]{gallot1990riemannian}). Let $(\mathcal{M}, g)$ be a compact Riemannian manifold. For $\ell\in\mathbb{N}$, denote $\lambda_\ell$ the eigenvalues and $\phi_\ell$ the eigenfunctions of the Laplace-Beltrami operator sorted by increasing eigenvalues. Then, we can define spectral distances as 
\begin{equation}
    \forall x,y\in\mathcal{M},\ d_\alpha(x,y) = \sum_{\ell\ge 0} \alpha(\lambda_\ell) \big(\phi_\ell(x)-\phi_\ell(y)\big)^2,
\end{equation}
where $\alpha:\mathbb{R}_+\to\mathbb{R}_+$ is a monotonically decreasing function. Then, they define the Intrinsic Sliced-Wasserstein (ISW) distance between $\mu,\nu\in\mathcal{P}_2(\mathcal{M})$ as 
\begin{equation}
    \isw_2^2(\mu,\nu) = \sum_{\ell\ge 0} \alpha(\lambda_\ell) W_2^2\big((\phi_\ell)_\#\mu, (\phi_\ell)_\#\nu\big).
\end{equation}
The eigenfunctions are used to map the measures to the real line, which make it very efficient to compute in practice. The eigenvalues are sorted in increasing order, and the series is often truncated by keeping only the $L$ smallest eigenvalues. This distance cannot be applied on Hadamard manifolds as these spaces are not compact. %On compact spaces such as the sphere, this provides an alternate sliced distance. 

\textbf{Sliced-Wasserstein on the Sphere.} \looseness=-1
\citet{bonet2023spherical} then proposed a Spherical Sliced-Wasserstein distance by integrating and projecting over all geodesics using the geodesic projection in an attempt to generalize intrinsically the Sliced-Wasserstein distance to the sphere $S^{d-1}$. We note that $\isw$ is more in the spirit of a max-K Sliced-Wasserstein distance \citep{dai2021sliced}, which projects over the $K$ maximal directions, than the Sliced-Wasserstein distance. More recently, \citet{quellmalz2023sliced, quellmalz2024parallelly} studied different Sliced-Wasserstein distances on $S^2$ by using spherical Radon transforms while \citet{tran2024stereographic} proposed to use the stereographic projection along the Generalized Sliced-Wasserstein distance \citep{kolouri2019generalized}, and \citet{garrett2024validating} proposed Sliced-Wasserstein distances over the space of functions on the sphere using a convolution slicer \emph{w.r.t} a kernel for the projection. Moreover, \citet{genest2024non} leveraged the Sliced-Wasserstein distance on manifolds to sample noise on non-Euclidean spaces such as meshes.

% However, on general geometries, the geodesic distance and the geodesic projection can be difficult to compute efficiently, as we may not always have closed-forms. In these situations, using the spectral distance can be beneficial as being more practical to compute but also more robust to noise and geometry aware \citep{lipman2010biharmonic, chen2023riemannian}. Nonetheless, we note that the computation of this spectrum is often impossible \citep{gallot1990riemannian, pennec2006intrinsic}, and that in particular cases where it is possible such as the sphere, computing the eigenfunctions can become numerically unstable in dimension $d\ge 10$ \citep[Appendix A]{dutordoir2020sparse}.

\textbf{Generalized Sliced-Wasserstein.} A somewhat related distance is the Generalized Sliced-Wasserstein distance (GSW) introduced by \citet{kolouri2019generalized}, and which uses nonlinear projections on the real lines. The main difference lies in the fact that GSW focuses on probability distributions lying in Euclidean space by projecting the measures along nonlinear hypersurfaces. That said, adapting the definition of GSW to handle probability measures on Riemannian manifolds, and the properties that need to be satisfied by the defining function $g$ such as the homogeneity, then we can write the $\chsw$ in the framework of GSW using $g:(x,v)\mapsto P^v(x)$. %We will discuss in the \Cref{section:chsw_properties} with more details the relations with the Radon transforms.

\section{Examples of Cartan-Hadamard Sliced-Wasserstein} \label{section:examples}

In this section, we specify the framework derived in full generality in \Cref{section:irsw} to particular Hadamard manifolds. More precisely, we first focus on manifolds endowed with a Pullback Euclidean metric, which are Hadamard manifolds with null curvature. Then, we take a look at Hyperbolic spaces which are manifolds of constant negative curvature. We also study the space of Symmetric Positive Definite matrices (SPD) endowed with metrics for which it is a Hadamard manifold. Finally, we discuss the case of the product manifold of Hadamard manifolds, which is itself a Hadamard manifold as products of manifolds of non-positive curvature are still of non-positive curvature \citep[Lemma 1]{gu2019learning}. We defer the proofs of this section to \Cref{proofs:section_examples}.

% \red{Other examples (to mention or not?): $\mathbb{R}_+^*$ with $d(x,y)=|\log(y/x)|$ \citep[Section 4.1.1]{cabanes2022apprentissage}, Poincaré polydisk \citep[Section 4.1.3]{cabanes2022apprentissage} (but product of manifolds), Siegel spaces (hard to compute?) \citep[Appendix E]{cabanes2022apprentissage}...}

\subsection{Pullback Euclidean Manifold} \label{section:pem}

% \red{Squared geodesic distances on $\mathbb{R}^d$ but with no explicit form (?) \citep{scarvelis2023riemannian, pooladian2023neural}. Réfléchir à injective with NFs ?}

Cartan-Hadamard manifolds include among others spaces of null curvature. As the % scalar 
curvature is conserved by the pullback operator, pullback Euclidean metrics are such spaces. We formally recall the definition of a pullback Euclidean metric along with its geodesic distance and exponential map following \citep[Theorem 3.3]{chen2023adaptive}.

\begin{theorem}[Pullback Euclidean Metric] \label{th:pem}
    Let $\mathcal{N}$ be an Euclidean space and denote $\langle\cdot,\cdot\rangle$ its inner product and $\|\cdot\|$ the associated norm. Let $\mathcal{M}$ be some space and $\phi:\mathcal{M}\to \mathcal{N}$ be a diffeomorphism. Then, defining for any $x\in\mathcal{M}$ and $u,v\in T_x\mathcal{M}$ the metric $g^\phi_x(u,v) = \langle \phi_{*,x}(u), \phi_{*,x}(v)\rangle$ where $\phi_{*,x}:T_x\mathcal{M}\to T_{\phi(x)}\mathcal{N}$ is the differential of $\phi$ at $x$, $(\mathcal{M},g^\phi)$ is a Riemannian manifold with geodesic distance
    \begin{equation}
        d_\mathcal{M}(x,y) = \|\phi(x) - \phi(y)\|.
    \end{equation}
    Moreover, the exponential map is 
    \begin{equation}
        \forall x\in \mathcal{M}, v\in T_x\mathcal{M},\ \exp_x(v) = \phi^{-1}\big(\phi(x) + \phi_{*,x}(v)\big).
    \end{equation}
\end{theorem}

Let $(\mathcal{M}, g^\phi)$ be such space. Denote $o$ the origin of $\mathcal{M}$. Geodesics passing through $o$ in direction $v\in T_o\mathcal{M}$ are of the form
\begin{equation}
    \forall t\in\mathbb{R},\ \gamma_v(t) = \phi^{-1}\big(\phi(o) + t \phi_{*,o}(v)\big).
\end{equation}
Moreover, tangent vectors $v\in T_o\mathcal{M}$ belong to the sphere $S_o$ if and only if $\|v\|_o^2 = \|\phi_{*,o}(v)\|^2 = 1$. Thus, using this formula, we can obtain both the geodesic and horospherical coordinates which actually coincide (up to a sign), as in the Euclidean case.

\begin{proposition} \label[proposition]{prop:proj_coord_pullback}
    Let $v\in S_o$, then the projection coordinate on $\mathcal{G}^v = \{\gamma_v(t),\ t\in\mathbb{R}\}$ is
    \begin{equation}
        \forall x\in \mathcal{M},\ P^v(x) = -B^v(x) = \langle \phi(x)-\phi(o), \phi_{*,o}(v)\rangle.
    \end{equation}
\end{proposition}

% \begin{proof}
    % See \Cref{proof:prop_proj_coord_pullback}.
% \end{proof}

For instance, the Euclidean space endowed with the Mahalanobis distance enters this framework for $\phi(x) = A^{\frac12} x$ with $A\in S_d^{++}(\mathbb{R})$ a positive definite matrix since in this case, for any $x,y\in\mathbb{R}^d$,
\begin{equation}
    d(x,y)^2 = (x-y)^TA(x-y) = \|A^{\frac12}x - A^{\frac12}y\|_2^2.
\end{equation}
In this case, we have $\phi(0) = 0$ and $\phi_{*,0}(v) = A^\frac12 v$. Thus, the projection is obtained by $P^v(x) = \langle A^{\frac12} x, A^{\frac12} v\rangle = x^T A v$ for $v\in S_0$, \emph{i.e.} which satisfies $\|v\|_0^2 = \|A^{\frac12}v\|_2^2 = 1$. In this situation, as expected, the directions and the data points are first mapped by the linear projection $x\mapsto A^{\frac12}x$ and then the usual orthogonal projections are performed as for the Euclidean Sliced-Wasserstein distance. % and we recover the Euclidean Sliced-Wasserstein distance.
\begin{definition}[Mahalanobis Sliced-Wasserstein]
    Let $p\ge 1$, $A\in S_d^{++}(\mathbb{R})$. The $p$-Mahalanobis Sliced-Wasserstein distance between $\mu,\nu\in\mathcal{P}_p(\mathbb{R}^d)$ is defined as
    \begin{equation}
        \sw_{p,A}^p(\mu,\nu) = \int_{S_0} W_p^p(P^v_\#\mu, P^v_\#\nu)\ \mathrm{d}\lambda_0(v),
    \end{equation}
    with $P^v(x) = x^T A v$ for $v\in S_0=\{v\in\mathbb{R}^d,\ v^T A v = 1\}$, $x\in\mathbb{R}^d$ and $\lambda_0$ the uniform distribution on $S_0$.
\end{definition}

The Mahalanobis distance is often learned in metric learning, which has been used for different applications in \emph{e.g.} computer visions, information retrieval or bioinformatics \citep{bellet2013survey}. In \Cref{sec:xp_mahalanobis}, we use the Mahalanobis Sliced-Wasserstein distance on a document classification task \citep{kusner2015word}, where the underlying metric $A$ is previously learned using metric learning methods \citep{huang2016supervised}.

% \red{Add discussion on metric learning + test with a distance learned? \citep{bellet2013survey, cuturi2014ground, de2020metric}}

More generally, this Pullback Euclidean framework includes any squared geodesic distance for which the metric is of the form $\langle u, v\rangle_x = u^T A(x) v$ with $A(x)\in S_d^{++}(\mathbb{R})$ for any $x\in\mathbb{R}^d$ \citep{scarvelis2023riemannian, pooladian2023neural}. For such a metric, we have $\phi_{*,x}(v) = A(x)^{\frac12} v$ and computing $\phi(x)$ in closed-form may not be straightforward.
It also includes many useful metrics used on the space of SPD matrices which we describe more thoroughly in \Cref{section:spdpem}.

\subsection{Hyperbolic Spaces} 

% \red{Ajouter une plus grosse intro sur les applis en ML ??? + connection with Fisher-Rao spaces ?}

\looseness=-1 Hyperbolic spaces are Riemannian manifolds of negative constant curvature $K<0$ \citep{lee2006riemannian} and are hence particular cases of Hadamard manifolds. They have recently received a surge of interest in machine learning as they allow embedding data with a hierarchical structure efficiently \citep{nickel2017poincare,nickel2018learning}. A thorough review of the recent use of hyperbolic spaces in machine learning can be found in \citep{peng2021hyperbolic, mettes2023hyperbolic}.

There are five usual parameterizations of a hyperbolic manifold \citep{peng2021hyperbolic}. They are equivalent (isometric) and one can easily switch from one formulation to the other. Hence, in practice, we use the one which is the most convenient, either given the formulae to derive or the numerical properties. In machine learning, the two most used models are the Poincaré ball and the Lorentz model (also known as the hyperboloid model). Each of these models has its own advantages compared to the other. %, see \emph{e.g.} \citep{mishne2023numerical}. 
For example, the Lorentz model has a distance which behaves better \emph{w.r.t.} numerical issues compared to the distance of the Poincaré ball. However, the Lorentz model is unbounded, contrary to the Poincaré ball. We introduce in the following these two models.

\textbf{Lorentz Model.} The Lorentz model of curvature $K<0$ is defined as
\begin{equation}
    \mathbb{L}^d_K = \left\{(x_0,\dots,x_d)\in\mathbb{R}^{d+1},\ \langle x, x\rangle_\mathbb{L} = \frac{1}{K},\ x_0>0\right\},
\end{equation}
where for $x,y\in\mathbb{R}^{d+1}$, $\langle x,y\rangle_\mathbb{L} = -x_0y_0 + \sum_{i=1}^d x_i y_i$ is the Minkowski pseudo inner-product. The Lorentz model can be seen as the upper sheet of a two-sheet hyperboloid. In the following, we will denote $x^0 = (\frac{1}{\sqrt{-K}},0,\dots,0)\in\mathbb{L}^d_K$ the origin of the hyperboloid. The geodesic distance in this manifold is defined as 
\begin{equation}
    \forall x,y\in\mathbb{L}^d_K,\ d_\mathbb{L}(x,y) = \frac{1}{\sqrt{-K}} \arccosh\big(K\langle x,y\rangle_\mathbb{L}\big).
\end{equation}
At any $x\in \mathbb{L}^d_K$, the tangent space is $T_x\mathbb{L}^d_K = \{v\in\mathbb{R}^{d+1},\ \langle x,v\rangle_\mathbb{L}=0\}$. Note that on $T_{x^0}\mathbb{L}^d_K$, the Minkowski inner product is the usual Euclidean inner product. Moreover, geodesics passing through $x$ in direction $v\in T_x\mathbb{L}^d_K$ are obtained as the intersection between the plane $\mathrm{span}(x, v)$ and the hyperboloid $\mathbb{L}_K^d$, and are of the form
\begin{equation}
    \forall t\in\mathbb{R},\ \exp_x(tv) = \cosh(\sqrt{-K} t \|v\|_\mathbb{L}) x + \frac{\sinh(\sqrt{-K}t\|v\|_\mathbb{L})}{\sqrt{-K}} \frac{v}{\|v\|_\mathbb{L}}.
\end{equation}
In particular, geodesics passing through the origin $x^0$ in direction $v\in S_{x^0}$ are 
\begin{equation}
    \forall t\in\mathbb{R},\ \gamma_{v}(t) = \exp_{x^0}(tv) = \cosh\big(\sqrt{-K} t\big) x^0 + \frac{\sinh\big(\sqrt{-K} t\big)}{\sqrt{-K}} v.
\end{equation}

% Moreover, geodesics passing through the origin $x^0$ in direction $v\in T_{x^0}\mathbb{L}^d_K\cap S^d$ are of the form 
% \begin{equation}
%     \forall t\in\mathbb{R},\ \gamma_{v}(t) = \exp_{x^0}(tv) = \cosh\big(\sqrt{-K} t\big) x^0 + \frac{\sinh\big(\sqrt{-K} t\big)}{\sqrt{-K}} v.
% \end{equation}

% \red{Recall \citep{bose2020latent, desai2023hyperbolic, skoped2019mixed}: $\mathbb{L}^d_K = \{(x_0,\dots,x_d)\in\mathbb{R}^{d+1},\ \langle x,x\rangle_\mathbb{L}=1/K,\ x_0>0\}$ for $K<0$, $d_\mathbb{L}(x,y) = \frac{1}{\sqrt{-K}}\arccosh(K\langle x,y\rangle_\mathbb{L})$, $\exp_x(v) = \cosh(\sqrt{-K}\|v\|_\mathbb{L})x + \sinh(\sqrt{-K}\|v\|_\mathbb{L})\frac{v}{\sqrt{-K} \|v\|_\mathbb{L}}$}

\textbf{Poincaré Ball.} The Poincaré ball of curvature $K<0$ is defined as
\begin{equation}
    \mathbb{B}^d_K = \left\{x\in\mathbb{R}^d,\ \|x\|_2^2 < -\frac{1}{K}\right\}.
\end{equation}
It can be seen as the stereographic projection of each point of $\mathbb{L}^d_K$ on the hyperplane $\{x\in\mathbb{R}^{d+1},\ x_0=0\}$. The origin of $\mathbb{B}^d_K$ is $0$ and the geodesic distance is defined as
\begin{equation}
    \forall x,y\in\mathbb{B}^d_K,\ d_\mathbb{B}(x,y) = \frac{1}{\sqrt{-K}} \arccosh\left( 1 - 2K \frac{\|x-y\|_2^2}{(1+K\|x\|_2^2)(1+K\|y\|_2^2)}\right).
\end{equation}
The tangent space is $\mathbb{R}^d$ and for any $\Tilde{v}\in S^{d-1}$, the geodesic passing through the origin is defined as 
\begin{equation}
    \forall t\in\mathbb{R},\ \gamma_{\Tilde{v}}(t) = \exp_0(t\Tilde{v}) = \frac{1}{\sqrt{-K}}\tanh\left(\frac{\sqrt{-K} t}{2}\right) \Tilde{v}.
\end{equation}

% \red{\citep{park2021unsupervised} $\mathbb{B}^d_K = \{x\in\mathbb{R}^d,\ \|x\|_2^2<-1/K\}$ for $K<0$. $d_\mathbb{B}(x,y) = \frac{1}{\sqrt{-K}}\arccosh\left(1-2K\frac{\|x-y\|_2^2}{(1+K\|x\|_2^2)(1+K\|y\|_2^2)}\right)$, $\gamma_p(t) = \exp_0^K(tp) = \tanh\left(\frac{\sqrt{-K}t}{2}\right) \frac{p}{\sqrt{-K}}$ for $p\in S^{d-1}$}

\textbf{Hyperbolic Sliced-Wasserstein.} To define Hyperbolic Sliced-Wasserstein distances, we first need to sample geodesics, which can be done in both models by simply sampling from a uniform measure on the sphere. Indeed, let $\Tilde{v} \in S^{d-1}$, then the direction of the geodesic in $\mathbb{L}_K^d$ is obtained as $v = (0,\Tilde{v}) \in T_{x^0}\mathbb{L}^d_K\cap S^d = S_{x^0}$ by concatenating $0$ to $\Tilde{v}$. On the Poincaré ball, $\Tilde{v}$ gives directly the direction to the geodesic, and is called an ideal point.

\begin{figure}[t]
    \centering
    \hspace*{\fill}
    \subfloat[Euclidean.]{\label{fig:proj_euc}\includegraphics[width={0.17\linewidth}]{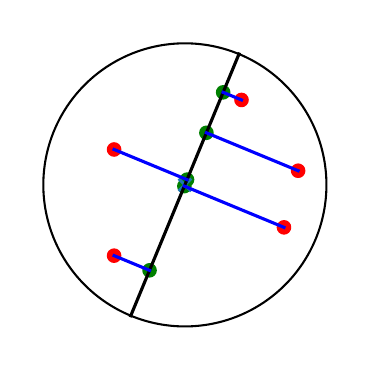}} \hfill
    \subfloat[Geodesics.]{\label{fig:proj_geods}\includegraphics[width={0.17\linewidth}]{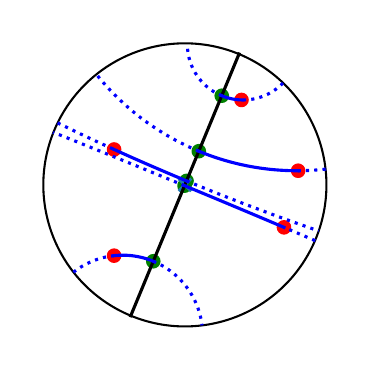}} \hfill
    \subfloat[Horospheres.]{\label{fig:proj_horo}\includegraphics[width={0.17\linewidth}]{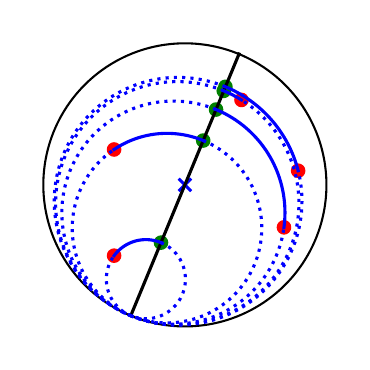}} \hfill
    \subfloat[Euclidean.]{\label{fig:proj_euc_l}\includegraphics[width=0.15\linewidth]{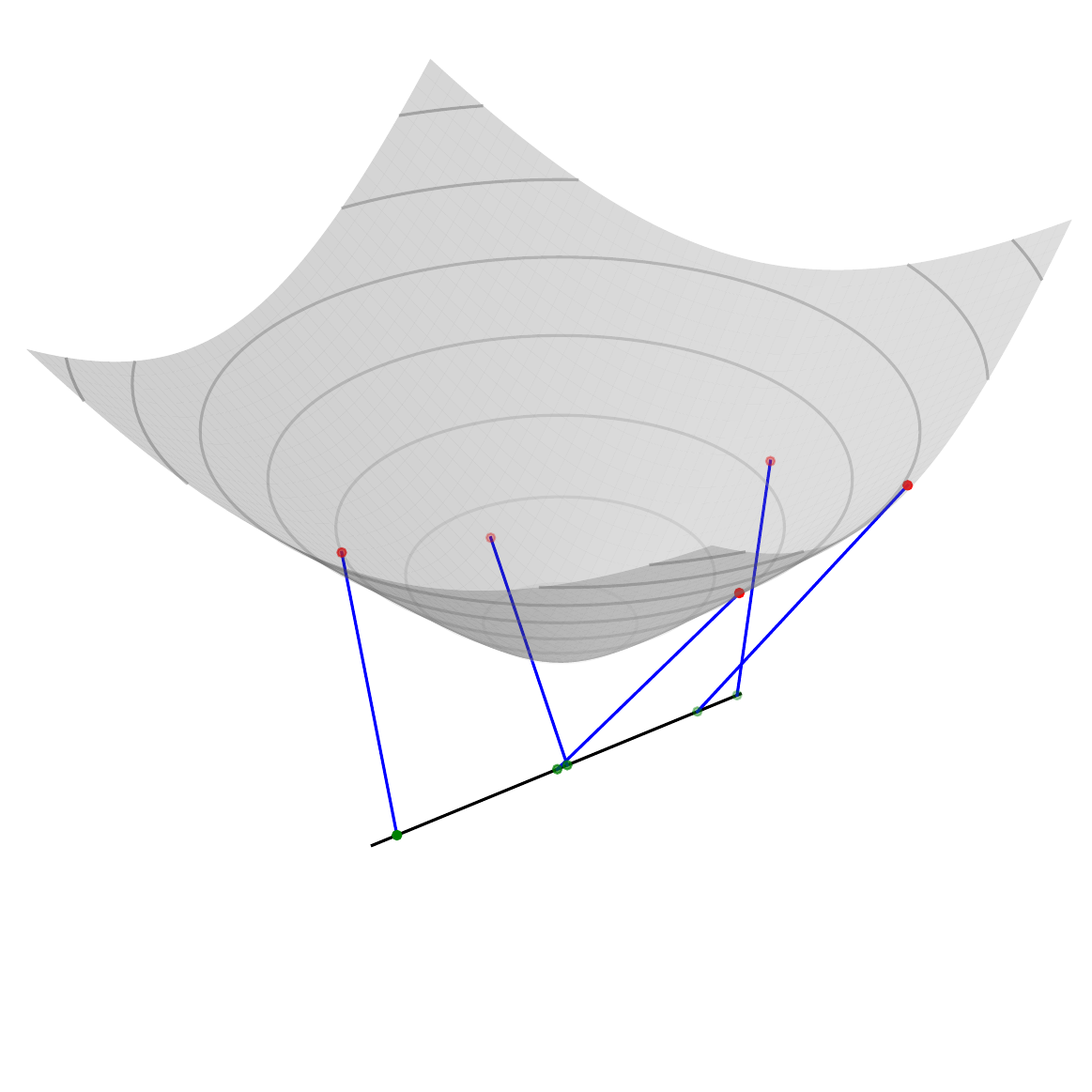}} 
    % \subfloat[Geodesics.]{\label{fig:proj_lorentz}\includegraphics[width=0.15\linewidth]{Figures/Illustrations/Projection_Lorentz_Poincare_v2.png}} \hfill
    \subfloat[Geodesics.]{\label{fig:proj_lorentz}\includegraphics[width=0.15\linewidth]{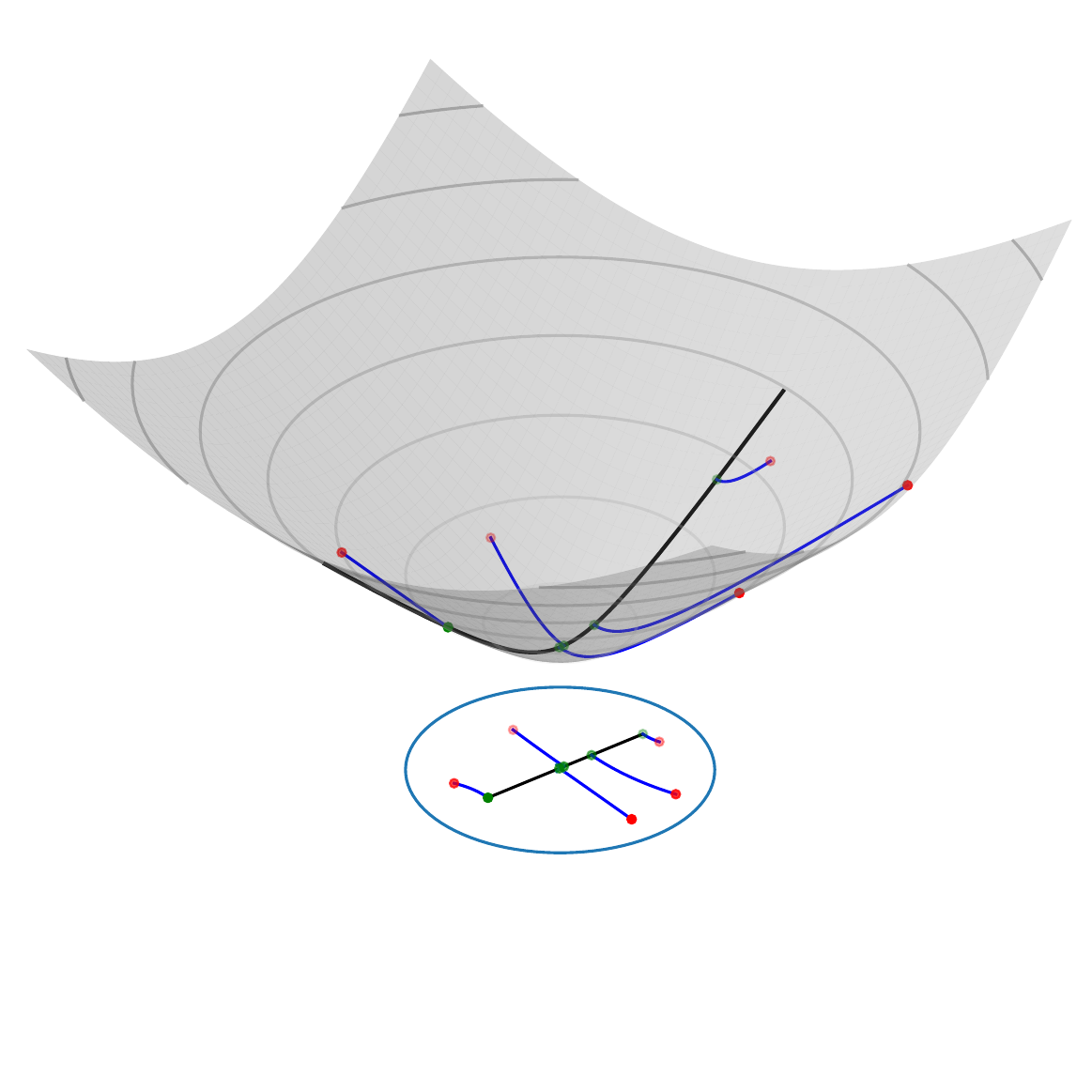}} \hfill
    % \subfloat[Horospheres.]{\label{fig:proj_lorentz_horo}\includegraphics[width=0.15\linewidth]{Figures/Illustrations/Projection_Lorentz_Poincare_Horosphere_v2.png}} \hfill
    \subfloat[Horospheres.]{\label{fig:proj_lorentz_horo}\includegraphics[width=0.15\linewidth]{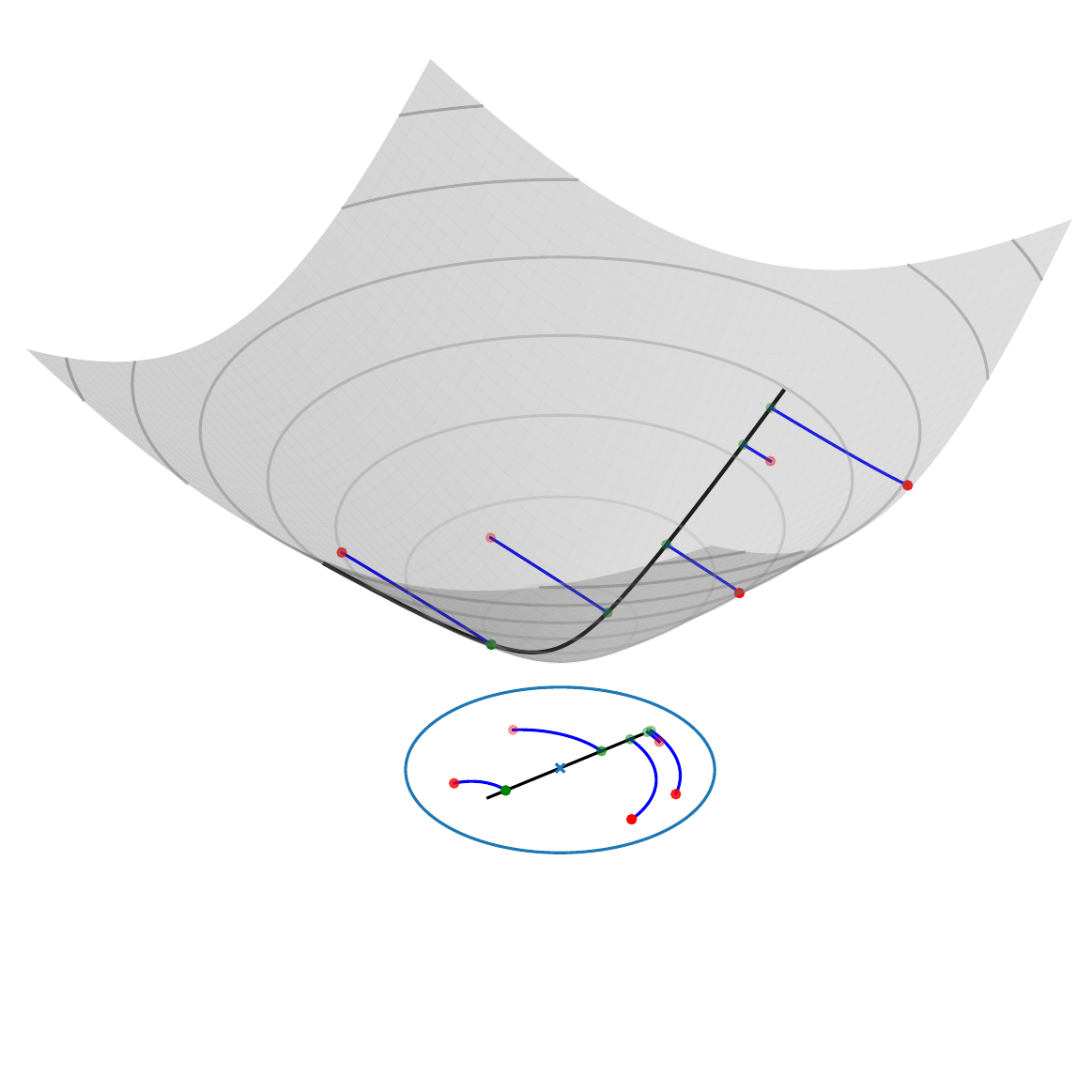}} \hfill
    \hspace*{\fill}
    \caption{Projection of (red) points on a geodesic (black line) in the Poincaré ball and in the Lorentz model along  Euclidean lines, geodesics or horospheres (in blue). Projected points on the geodesic are shown in green.}
    \label{fig:projections_hsw}
    % \vspace{-10pt}
\end{figure}

\looseness=-1 Thus, we only need to compute the projection coordinates on the geodesics in order to build the corresponding Geodesic and Horospherical Sliced-Wasserstein distances. We provide the closed-form of the geodesic projection and the Busemann function for both models in the following propositions. Additionally, we illustrate the projection process in \Cref{fig:projections_hsw}.

% \begin{proposition}[Coordinate of the geodesic projection] \label{prop:hsw_coord_geod_proj} \leavevmode
%     \begin{enumerate}
%         \item Let $\mathcal{G}^v = \mathrm{span}(x^0, v)\cap \mathbb{L}^d$ where $v\in T_{x^0}\mathbb{L}^d\cap S^d$. Then, the coordinate $P^v$ of the geodesic projection on $\mathcal{G}^v$ of $x\in \mathbb{L}^d$ is
%         \begin{equation}
%             P^v(x) = \arctanh\left(-\frac{\langle x, v\rangle_\mathbb{L}}{\langle x,x^0\rangle_\mathbb{L}}\right).
%         \end{equation}
%         \item Let $\Tilde{v}\in S^{d-1}$ be an ideal point. Then, the coordinate $P^{\Tilde{v}}$ of the geodesic projection on the geodesic characterized by $\Tilde{v}$ of $x\in \mathbb{B}^d$ is
%         \begin{equation}
%             P^{\Tilde{v}}(x) = 2 \arctanh\big(s(x)\big),
%         \end{equation}
%         where $s$ is defined as 
%         \begin{equation}
%             s(x) = \left\{\begin{array}{ll} \frac{1+\|x\|_2^2 - \sqrt{(1+\|x\|_2^2)^2 - 4 \langle x, \Tilde{v}\rangle^2}}{2 \langle x, \Tilde{v}\rangle} & \mbox{ if } \langle x,\Tilde{v}\rangle \neq 0 \\
%             0 & \mbox{ if } \langle x,\Tilde{v}\rangle = 0.
%             \end{array}\right.
%         \end{equation}
%     \end{enumerate}
% \end{proposition}

\begin{proposition}[Coordinate projections on Hyperbolic spaces] \label[proposition]{prop:hsw_coord_projs} \leavevmode
    \begin{enumerate}
        \item Let $v\in S_{x^0} = T_{x^0}\mathbb{L}_K^d\cap S^d$, the geodesic and horospherical projection coordinates on $\mathcal{G}^v = \mathrm{span}(x^0,v)\cap \mathbb{L}_K^d$ are for all $x\in\mathbb{L}_K^d$,
        \begin{align}
            &P^v(x) = \frac{1}{\sqrt{-K}}\arctanh\left(-\frac{1}{\sqrt{-K}}\frac{\langle x, v\rangle_\mathbb{L}}{\langle x,x^0\rangle_\mathbb{L}}\right), \\
            &B^v(x) = \frac{1}{\sqrt{-K}} \log\left(-\sqrt{-K}\left\langle x, \sqrt{-K} x^0 + v\right\rangle_\mathbb{L}\right).
        \end{align}
        \item Let $\Tilde{v}\in S^{d-1}$ an ideal point. Then the geodesic and horospherical projections coordinates on $\mathcal{G}^{\Tilde{v}}=\{\gamma_{\Tilde{v}}(t),\ t\in\mathbb{R}\}$ are for all $x\in\mathbb{B}_K^d$,
        \begin{align}
            &P^{\Tilde{v}}(x) = \frac{2}{\sqrt{-K}} \arctanh\big(\sqrt{-K} s(x)\big), \\
            &B^{\Tilde{v}}(x) = \frac{1}{\sqrt{-K}}\log\left(\frac{\|\Tilde{v}-\sqrt{-K}x\|_2^2}{1+K\|x\|_2^2}\right),
        \end{align}
        where $s$ is defined as 
        \begin{equation}
            s(x) = \left\{\begin{array}{ll} \frac{1-K\|x\|_2^2 - \sqrt{(1-K\|x\|_2^2)^2 + 4K \langle x, \Tilde{v}\rangle^2}}{-2K \langle x, \Tilde{v}\rangle} & \mbox{ if } \langle x,\Tilde{v}\rangle \neq 0 \\
            0 & \mbox{ if } \langle x,\Tilde{v}\rangle = 0.
            \end{array}\right.
        \end{equation}
    \end{enumerate}
\end{proposition}

This proposition allow to define hyperbolic Sliced-Wasserstein distances by specifying $\chsw$ with the right formulas.

\begin{definition}[Hyperbolic Sliced-Wasserstein] \leavevmode
    \begin{enumerate}
        \item Let $p\ge 1$, $\mu,\nu\in\mathcal{P}_p(\mathbb{L}^d_K)$. Then, the $p$-Geodesic Hyperbolic Sliced-Wasserstein distance and the $p$-Horospherical Hyperbolic Sliced-Wasserstein distance on the Lorentz model $\mathbb{L}^d_K$ are defined as
        \begin{align}
            &\ghsw_p^p(\mu,\nu) = \int_{T_{x^0}\mathbb{L}^d_K\cap S^d} W_p^p(P^v_\#\mu, P^v_\#\nu)\ \mathrm{d}\lambda(v) \\
            &\hhsw_p^p(\mu,\nu) = \int_{T_{x^0}\mathbb{L}^d_K\cap S^d} W_p^p(B^v_\#\mu, B^v_\#\nu)\ \mathrm{d}\lambda(v).
        \end{align}
        \item Let $p\ge 1$, $\Tilde{\mu}, \Tilde{\nu}\in\mathcal{P}_p(\mathbb{B}^d_K)$. Then, the $p$-Geodesic Hyperbolic Sliced-Wasserstein distance and the $p$-Horospherical Hyperbolic Sliced-Wasserstein distance on the Poincaré ball $\mathbb{B}^d_K$ are defined as
        \begin{align}
            &\ghsw_p^p(\Tilde{\mu}, \Tilde{\nu}) = \int_{S^{d-1}} W_p^p(P^{\Tilde{v}}_\#\Tilde{\mu}, P^{\Tilde{v}}_\#\Tilde{\nu})\ \mathrm{d}\lambda(\Tilde{v}) \\
            &\hhsw_p^p(\Tilde{\mu}, \Tilde{\nu}) = \int_{S^{d-1}} W_p^p(B^{\Tilde{v}}_\#\Tilde{\mu}, B^{\Tilde{v}}_\#\Tilde{\nu})\ \mathrm{d}\lambda(\Tilde{v}).
        \end{align}
    \end{enumerate}
\end{definition}

Note that we could also work on the other models such as the Klein model, the Poincaré half-plane model or the hemisphere model (see \emph{e.g.} \citep{cannon1997hyperbolic, loustau2020hyperbolic}) and derive the corresponding projections in order to define the Hyperbolic Sliced-Wasserstein distances in these models. Note also that these different Sliced-Wasserstein distances are actually equal from one model to the other when using the isometry mappings, which is a particular case of \Cref{prop:isometry_chsw}.

\begin{proposition} \label[proposition]{prop:isometry_chsw}
    Let $(\mathcal{M}, g^\mathcal{M})$ and $(\mathcal{N}, g^\mathcal{N})$ be two isometric Cartan-Hadamard manifolds, $\phi:\mathcal{M}\to\mathcal{N}$ an isometry and assume that $\lambda_{\phi(o)} = (\phi_{*,o})_\#\lambda_o$\footnote{We expect it to be true in general as $\phi$ is an isometry, but we did not find in the literature a formal proof. In practice, this fact was verified for each tested case.}. Let $p\ge 1$, $\mu,\nu\in\mathcal{P}_p(\mathcal{M})$ and $\Tilde{\mu} = \phi_\#\mu$, $\Tilde{\nu}=\phi_\#\nu$. Then,
    \begin{align}
        \chsw_p^p(\mu, \nu; \lambda_o) = \chsw_p^p(\Tilde{\mu}, \Tilde{\nu}; \lambda_{\phi(o)}),
        % \chsw_p^p(\mu, \nu; \lambda_o) = \chsw_p^p(\Tilde{\mu}, \Tilde{\nu}; (\phi_{*,o})_\#\lambda_o),
    \end{align}
    where we denote $\chsw_p^p(\mu,\nu;\lambda)$ the Cartan-Hadamard Sliced-Wasserstein distance with slicing distribution $\lambda$.%\nc{add assumption that $\lambda_o =  \lambda_{\phi(o)}$}
\end{proposition}

% \begin{proof}
%     See \Cref{proof:prof_isometry_chsw}.
% \end{proof}

\Cref{prop:isometry_chsw} includes as a particular case the Hyperbolic Sliced-Wasserstein distances (and in particular is more general than \citep[Proposition 3.4]{bonet2022hyperbolic}). 
% For $P_{\mathbb{B}\to\mathbb{L}}$ and $P_{\mathbb{L}\to\mathbb{B}}$ the isometries between the Poincaré ball and the Lorentz model, it can be further shown that $(P_{\mathbb{B}\to\mathbb{L}})_{*,0}(\Tilde{v}) = 2(0, \Tilde{v})$ and $(P_{\mathbb{L}\to\mathbb{B}})_{*,x^0}(v) = \frac12 v_{1:d}$, and thus we have an equality with the right slicing distributions, as observed in \citep[Proposition 3.4]{bonet2022hyperbolic}. 
This demonstrates that the Hyperbolic Sliced-Wasserstein distances are independent from the chosen model. Thus, we can work in the model which is the most convenient for us. Moreover, if we work on a model for which we do not have necessarily a closed-form, we can project the distributions on a model where we have the closed-forms such as the Lorentz model or the Poincaré ball for which we derived the closed-forms in this section.

% \red{\Cref{prop:isometry_chsw} à checker, peut-être qu'on a bien une uniforme..., pas sûr de comment le montrer (juste dire que c'est une isométrie ok ?)}

\looseness=-1 \citet{bonet2022hyperbolic} compared GHSW and HHSW on different tasks such as gradient flows or as regularizers for deep classification with prototypes. Moreover, they also verified empirically that GHSW and HHSW are independent with respect to the model while comparing evolutions of the distances between Wrapped Normal distributions. In particular, they observed that HHSW had values closer to the Wasserstein distance compared to GHSW.

\subsection{Symmetric Positive Definite Matrices} \label{section:spd}

% \red{Ajouter intro utilité ML ? + bonnes refs pour img ets...}

Let $S_d(\mathbb{R})$ be the set of symmetric matrices of $\mathbb{R}^{d \times d}$, and $S_d^{++}(\mathbb{R})$ be the set of SPD matrices of $\mathbb{R}^{d \times d}$, \emph{i.e.} matrices $M\in S_d(\mathbb{R})$ satisfying for all $x\in\mathbb{R}^d\setminus\{0\},\ x^T M x > 0$.
$S_d^{++}(\mathbb{R})$ is a Riemannian manifold \citep{bhatia2009positive} which can be endowed with different metrics. At each $M\in S_d^{++}(\mathbb{R})$, we can associate a tangent space $T_M S_d^{++}(\mathbb{R})$ which can be identified with the space of symmetric matrices $S_d(\mathbb{R})$.

\looseness=-1 SPD matrices have received a lot of attention in Machine Learning. On one hand, this is the natural space to deal with invertible covariance matrices, which are often used to represent M/EEG data \citep{blankertz2007optimizing, sabbagh2019manifold} or images \citep{tuzel2006region, pennec2020manifold}. Moreover, this space is more expressive than Euclidean spaces, and endowed with specific metrics such as the Affine-Invariant metric, it enjoys a non-constant non-positive curvature. This property was leveraged to embed different type of data \citep{harandi2014manifold, brooks2019exploring}. This motivated the development of different machine learning algorithms \citep{chevallier2017kernel, yair2019domain, zhuang2020spd, lei2021eeg, ju2022deep} and of neural networks architectures \citep{huang2017riemannian, brooks2019riemannian}.

% , zhuang2020spd, lei2021eeg,

We now introduce the Sliced-Wasserstein distance on the space of SPD matrices first endowed with the Affine-Invariant metric, and in a second part endowed with different pullback Euclidean metrics.

\subsubsection{Symmetric Positive Definite Matrices with Affine-Invariant Metric.} \label{section:spdai}

A classical metric used widely with SPDs is the geometric Affine-Invariant metric \citep{pennec2006riemannian}, where the inner product is defined as
\begin{equation}
    \forall M \in S_d^{++}(\mathbb{R}),\ A,B\in T_M S_d^{++}(\mathbb{R}),\ \langle A,B\rangle_M = \tr(M^{-1}AM^{-1}B).
\end{equation}
Denoting by $\tr$ the Trace operator, the corresponding geodesic distance $d_{AI}(\cdot,\cdot)$ is given by
\begin{equation}
    \forall X, Y \in S_d^{++}(\mathbb{R}),\ d_{AI}(X,Y) = \sqrt{\tr\big(\log(X^{-1}Y)^2\big)}.
\end{equation}
An interesting property justifying the use of the Affine-Invariant metric is that $d_{AI}$ satisfies the affine-invariant property: for any $g\in GL_d(\mathbb{R})$, where $GL_d(\mathbb{R})$ denotes the set of invertible matrices in $\mathbb{R}^{d\times d}$,
\begin{equation}
    \forall X,Y\in S_d^{++}(\mathbb{R}),\ d_{AI}(g\cdot X, g\cdot Y) = d_{AI}(X,Y),
\end{equation}
where $g\cdot X = gXg^T$. With this metric, $S_d^{++}(\mathbb{R})$ is of (non-constant) non-positive curvature and hence a Hadamard manifold.

The natural origin is the identity matrix $I_d$ and geodesics passing through $I_d$ in direction $A\in S_d(\mathbb{R})$ are of the form \citep[Section 3.6.1]{pennec2020manifold}
\begin{equation}
    \forall t\in\mathbb{R},\ \gamma_A(t) = \exp_{I_d}(tA) = \exp(tA),
\end{equation}
where $\exp$ denotes the matrix exponential.

For the Affine-Invariant case, to the best of our knowledge, there is no closed-form for the geodesic projection on $\mathcal{G}^A$, the difficulty being that the matrices do not necessarily commute. Hence, we will discuss here the horospherical projection which can be obtained with the Busemann function. For $A\in S_d(\mathbb{R})$ such that $\|A\|_F=1$, denoting $\gamma_A:t\mapsto \exp(tA)$ the geodesic line passing through $I_d$ with direction $A$, the Busemann function $B^A$ associated to $\gamma_A$ writes as 
\begin{equation}
    \forall M\in S_d^{++}(\mathbb{R}),\ B^A(M) = \lim_{t\to\infty}\ \big(d_{AI}\big(\exp(tA),M\big)-t\big).
\end{equation} 
We cannot directly compute this quantity by expanding the distance since $\exp(-tA)$ %\nc{why $- tA$ ?} 
and $M$ are not necessarily commuting. The main idea to solve this issue is to first find a group $G\subset GL_d(\mathbb{R})$ which will leave the Busemann function invariant. Then, we can find an element of this group which will project $M$ on the space of matrices commuting with $\exp(A)$. This part of the space is of null curvature, \emph{i.e.} it is isometric to an Euclidean space. In this case, we can compute the Busemann function as the matrices are commuting. Hence, the Busemann function is of the form
\begin{equation}
    B^A(M) = -\left\langle A, \log \big(\pi_A (M)\big)\right\rangle_F, 
\end{equation}
where $\pi_A$ is a projection on the space of commuting matrices which can be obtained in practice through a UDU or LDL decomposition. We detail more precisely in \Cref{appendix:busemann} how to obtain $\pi^A$. For more details about the Busemann function on the Affine-invariant space, we refer to \citet[Section II.10]{bridson2013metric} and \citet{fletcher2009computing, fletcher2011horoball}.

We note that computing the Busemann function on this space induces a heavy computational cost. Thus, we advocate using in practice Sliced-Wasserstein distances obtained using Pullback-Euclidean metrics on SPDs as described in the next section.

\begin{figure}[t]
    \centering
    \includegraphics[width=0.7\columnwidth]{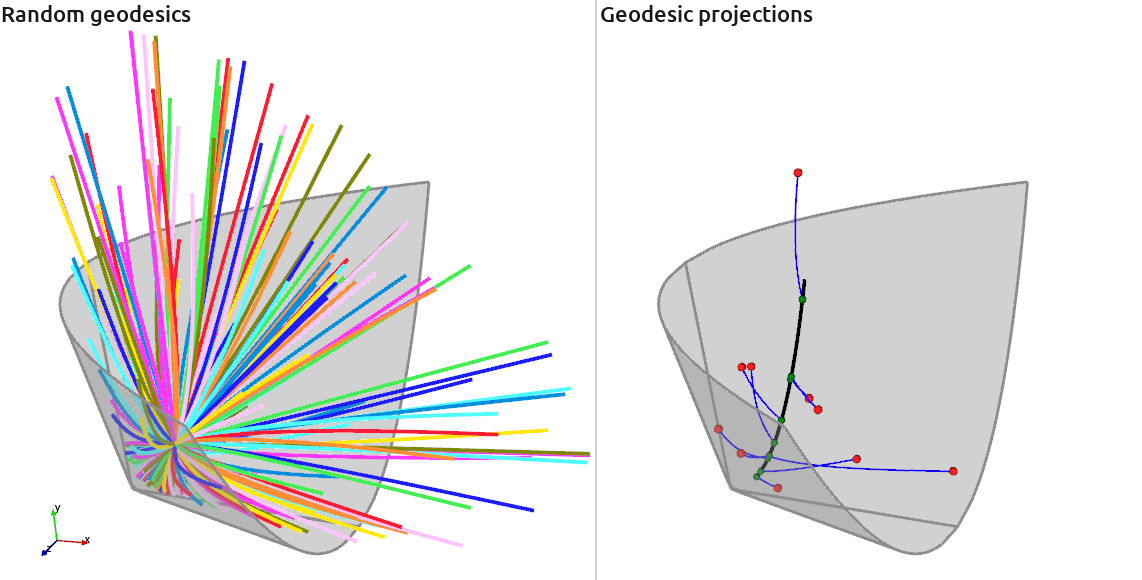}
    \caption{({\bf Left}) Random geodesics drawn in  $S_2^{++}(\mathbb{R})$. ({\bf Right}) Projections (green points) of covariance matrices (depicted as red points) over one geodesic (in black) passing through $I_2$ along the Log-Euclidean geodesics (blue lines).}
    \label{fig:spdsw_proj}
    % \vspace{-10pt}
\end{figure}

\subsubsection{Symmetric Positive Definite Matrices with Pullback Euclidean Metrics.} \label{section:spdpem}

We study here metrics endowing the space of SPD matrices which are pullback Euclidean metrics \citep{chen2023riemannian_multiclass, chen2023adaptive}, \emph{i.e.} metrics which are obtained through a diffeomorphism from $S_d^{++}(\mathbb{R})$ to $(S_d(\mathbb{R}), \langle\cdot,\cdot\rangle_F)$. Pullback Euclidean metrics and more generally pullback metrics allow inheriting properties from the mapped space \citep{chen2023adaptive}. The pullback Euclidean metrics studied here belong to the framework presented in \Cref{section:pem} with $\mathcal{M}=S_d^{++}(\mathbb{R})$ and $\mathcal{N}=S_d(\mathbb{R})$.
This framework includes many interesting metrics, such as the Log-Euclidean metric with $\phi=\log$ \citep{arsigny2005fast,arsigny2006log} which is a good first order approximation of the Affine-Invariant metric \citep{arsigny2005fast,pennec2020manifold}, the Log-Cholesky metric \citep{lin2019riemannian} or the recently proposed $O(n)$-invariant Log-Euclidean metric \citep{thanwerdas2023n, chen2023riemannian_multiclass} and Adaptative Riemannian metric \citep{chen2023adaptive}.

\textbf{Log-Euclidean Metric.} We first focus on the Log-Euclidean metric for which $\phi=\log$. To apply \Cref{prop:proj_coord_pullback}, we first need to compute its differential in the origin $I_d$. For completeness, we recall here the differential form of the matrix logarithm derived \emph{e.g.} in \citep{pennec2020manifold}.

\begin{lemma}[Section 3.2.2 in \citep{pennec2020manifold}] \label[lemma]{lemma:pennec_diff_log}
    Let $\phi:X\mapsto \log(X)$ and $X=UDU^T\in S_d^{++}(\mathbb{R})$ where $D=\mathrm{diag}(\lambda_1,\dots,\lambda_d)$. The differential operator of $\phi$ at $X$ is given by
    \begin{equation}
        \forall V\in T_X S_d^{++}(\mathbb{R}),\ \phi_{*,X}(V) = U \Sigma(V) U^T,
    \end{equation}
    where $\Sigma(V) = U^T V U \odot \Gamma$ and $\Gamma$ is the Loewner's matrix defined for all $i,j\in\{1,\dots,d\}$ as 
    \begin{equation}
        \Gamma_{ij} = \left\{\begin{array}{ll}
            \frac{\log \lambda_i - \log \lambda_j}{\lambda_i - \lambda_j} & \mbox{ if }i\neq j \\
            \frac{1}{\lambda_i} & \mbox{ if } i=j.
        \end{array}\right.
    \end{equation}
    % \begin{equation}
    %     \Gamma = \begin{pmatrix}
    %         \frac{1}{\lambda_1} & \frac{\log \lambda_1 - \log \lambda_2}{\lambda_1-\lambda_2} & \dots  & \\
    %         \frac{\log \lambda_2 - \log \lambda_1}{\lambda_2-\lambda_1} & \frac{1}{\lambda_2} & \ddots & \\
    %         \vdots & \ddots &  \ddots & \\
    %         & & & \frac{1}{\lambda_d}
    %     \end{pmatrix}.
    % \end{equation}
\end{lemma}
\begin{proof}
    Apply the Dalickii-Krein formula, see \emph{e.g.} \citep[Theorem 2.11]{noferini2017formula}.
\end{proof}
We note that for close eigenvalues \citep{pennec2020manifold},
\begin{equation}
    \frac{\log \lambda_i - \log\lambda_j}{\lambda_i-\lambda_j} = \frac{1}{\lambda_j} \left(1-\frac{\lambda_i-\lambda_j}{2\lambda_j} + \frac{(\lambda_i-\lambda_j)^2}{3\lambda_j^2} + O\big((\lambda_i-\lambda_j)^3\big)\right).    
\end{equation}
Furthermore, for $X=D=U=I_d$, since $[U^T V U]_{ij} = V_{ij}$, we find $\phi_{*, I_d}(V) = V$ for any $V$. Thus, as $\log(I_d) = 0$, we obtain the following projections.

\begin{proposition}
    Let $\phi=\log$. Then, for any $A\in S_d(\mathbb{R})$ such that $\|A\|_F = 1$, the coordinate projection is
    \begin{equation}
        \forall X\in S_d^{++}(\mathbb{R}),\ P^A(X) = -B^A(X) = \langle \log(X), A\rangle_F.
    \end{equation}
\end{proposition}

\begin{proof}
    Apply \Cref{prop:proj_coord_pullback} with $\phi(X)=\log(X)$ observing that $\phi(I_d)=0$ and $\phi_{*,I_d} = \id$.
\end{proof}

In particular, \citet{bonet2023sliced} used the Sliced-Wasserstein distance on the space of SPDs endowed with the Log-Euclidean metric, tamed SPDSW, and applied it to M/EEG data to perform brain-age prediction and domain adaptation for brain computational interfaces.

\textbf{O(n)-Invariant Log-Euclidean Metric.} The $O(n)$-invariant Log-Euclidean metric has been introduced by \citet{thanwerdas2023n} and further studied in \citep{chen2023riemannian_multiclass}. It is a pullback Euclidean metric with, for $X\in S_d^{++}(\mathbb{R})$ and $p,q\ge 0$, $\phi^{p,q}(X) = F^{p,q}\big(\log(X)\big)$ where $F^{p,q}(A) = q A + \frac{p-q}{d}\mathrm{Tr}(A)I_d$ for $A\in S_d(\mathbb{R})$. It can be seen as a generalization of the Log-Euclidean metric since for $p=q=1$, $F^{1,1}(A) = A$. Since $F^{p,q}$ is a linear function, the differential of $\phi^{p,q}$ at $X\in S_d^{++}(\mathbb{R})$ is $\phi_{*,X}^{p,q}(V) = F^{p,q}\big(\log_{*,X}(V)\big)$ for any $V\in S_d(\mathbb{R})$. Thus, we have $\phi^{p,q}(I_d) = 0$, $\phi^{p,q}_{*,I_d} = F^{p,q}$, and we can apply \Cref{prop:proj_coord_pullback}.
% and, applying \Cref{prop:proj_coord_pullback}, the projection for any $A\in S_d(\mathbb{R})$ such that $\|A\|_{I_d}^2 = \langle F^{p,q}(A), F^{p,q}(A)\rangle_F = 1$ is in this case
% \begin{equation}
%     \forall X\in S_d^{++}(\mathbb{R}),\ P^A(X) = \left\langle F^{p,q}\big(\log(X)\big), F^{p,q}(A)\right\rangle_F.
% \end{equation}

\begin{proposition}
    Let $p,q\ge 0$, $\phi^{p,q} = F^{p,q}\circ \log$ with $F^{p,q}(A)=q A + \frac{p-q}{d}\tr(A] I_d$ for $A\in S_d(\mathbb{R})$. Then, for any $A\in S_d(\mathbb{R})$ such that $\|A\|_{I_d}^2 = \langle F^{p,q}(A),F^{p,q}(A)\rangle_F =1$, the coordinate projection is
    \begin{equation}
        \forall X\in S_d^{++}(\mathbb{R}),\ P^A(X) = \left\langle F^{p,q}\big(\log(X)\big), F^{p,q}(A)\right\rangle_F.
    \end{equation}
\end{proposition}

\begin{proof}
    Apply \Cref{prop:proj_coord_pullback} with $\phi(X)=F^{p,q}\big(\log(X)\big)$ observing that $\phi(I_d)=0$ and $\phi_{*, I_d} = F^{p,q}$.
\end{proof}

% \textbf{\red{Adaptative Log-Euclidean Metric.}} \citep{chen2023adaptive}

\textbf{Log-Cholesky Metric.} \citet{lin2019riemannian} introduced the Log-Cholesky metric which is obtained as a pullback Euclidean metric with respect to $L_d(\mathbb{R})$, the space of lower triangular matrices, endowed with the Frobenius inner product. The diffeomorphism between $S_d^{++}(\mathbb{R})$ and $L_d(\mathbb{R})$ is of the form $\phi:X\mapsto \varphi(\mathcal{L}(X))$ with $\mathcal{L}:S_d^{++}(\mathbb{R})\to L_d^{++}(\mathbb{R})$ which returns the lower triangular matrix obtained by the Cholesky decomposition, \emph{i.e.} for $X=LL^T\in S_d^{++}(\mathbb{R})$, $\mathcal{L}(X) = L$, and $\varphi: L_d^{++}(\mathbb{R})\to L_d(\mathbb{R})$ defined as $\varphi(L) = \lfloor L\rfloor + \log\big(\mathrm{diag}(L)\big)$ with $\lfloor\cdot\rfloor$ the strictly lower triangular part of the matrix and $\mathrm{diag}$ its diagonal part. 

It is easy to see that $\phi(I_d) = 0$. We compute the differential of $\phi$ in \Cref{lemma:diff_log_cholesky} by using the chain rule and \citep[Proposition 4]{lin2019riemannian} which gives the differential of $\mathcal{L}$. Then, applying \Cref{prop:proj_coord_pullback}, we can compute the projection.

\begin{proposition} \label[proposition]{prop:proj_log_cholesky}
    Let $\phi = \varphi \circ \mathcal{L}$. Then, for any $A\in S_d(\mathbb{R})$ such that $\|A\|_{I_d}^2 = 1$, the coordinate projection is
    \begin{equation}
        \forall X=LL^T \in S_d^{++}(\mathbb{R}),\ P^A(X) = \big\langle \lfloor L\rfloor, \lfloor A\rfloor\big\rangle_F + \big\langle \log(\mathrm{diag}(L)), \frac12 \mathrm{diag}(A)\big\rangle_F.
    \end{equation}
\end{proposition}

\subsection{Product of Hadamard Manifolds.}

% \citep{bridson2013metric, gu2019learning}
% \red{Add intro product manifold + metric + geodesic distance}

\looseness=-1 In recent attempts to embed data into more flexible spaces, it was proposed to use products of manifolds \citep{gu2019learning, Skopek2020Mixed-curvature, borde2023neural, borde2023latent} instead of constant curvature spaces, as the data may not be uniformly curved. Products of constant curvature spaces are not necessarily of constant curvature, thus it allows to have more flexibility in order to embed the data by better capturing the curvature of the underlying manifold. Since product of Hadamard manifolds are still Hadamard manifolds \citep{gu2019learning}, product of hyperbolic spaces are Hadamard manifolds, and can be used to obtain flexible spaces \emph{e.g.} by learning the curvature of the different spaces. Another example of product of Hadamard manifolds is the Poincaré polydisk \citep{cabanes2022apprentissage} which is the product manifold of $\mathbb{R}_+^*$ with distance $d(x,y)=|\log(y/x)|$ and the Poincaré disk, and which has received attention for radar applications \citep{le2017probability}.
Note also that Gaussian distributions with diagonal covariances matrices endowed with the Fisher information matrix form a product of hyperbolic spaces \citep{cho2022gm}. Therefore, it is of interest to provide tools to compare probability distributions on products of Hadamard manifolds.

Let $\big((\mathcal{M}_i, g_i)\big)_{i=1}^n$ be $n$ Hadamard manifolds and define the product manifold $\mathcal{M} = \mathcal{M}_1\times\dots\times\mathcal{M}_n$. Then, at $x=(x_1,\dots,x_n)\in\mathcal{M}$, the tangent space is simply the inner product $T_x\mathcal{M} = T_{x_1}\mathcal{M}_1\times\dots\times T_{x_n}\mathcal{M}_n$, and $\mathcal{M}$ is equipped with the metric tensor $g = \sum_{i=1}^n g_i$. Moreover, for $v=(v_1,\dots,v_n)\in T_o\mathcal{M}$, the geodesic passing through the origin $o=(o_1,\dots,o_n)$ in direction $v$ is
\begin{equation}
    \forall t\in \mathbb{R},\ \gamma_o(t) = \big(\gamma_{o_1}(t),\dots,\gamma_{o_n}(t)\big),
\end{equation}
where $\gamma_{o_i}$ is a geodesic in $\mathcal{M}_i$ passing through $o_i$ in direction $v_i$. Moreover, the squared geodesic distance can be simply obtained as \citep{gu2019learning}
\begin{equation}
    \forall x,y \in \mathcal{M},\ d_\mathcal{M}(x,y)^2 = \sum_{i=1}^n d_{\mathcal{M}_i}(x_i,y_i)^2.
\end{equation}

% In \citep{lin2023hyperbolic}, they do not use the geodesic distance?

Deriving the closed-form for the geodesic projection
\begin{equation} \label{eq:geod_proj_product}
    t^* = \argmin_{t\in\mathbb{R}}\ \sum_{i=1}^n d_{\mathcal{M}_i}\big(\gamma_{o_i}(t), y_i\big)^2
\end{equation}
might depend on the context and might not be straightforward. 
% Note that \eqref{eq:geod_proj_product} is a convex problem as a sum of convex problems. \notsure{A solution might be to find the projection by gradient descent, using the following scheme,
% \begin{equation}
%     \forall k\ge 0,\ t_{k+1} = t_k + 2 \tau \sum_{i=1}^n \langle \log_{\gamma_{o_i}(t)}(y_i), \gamma_{o_i}'(t_k)\rangle_{\gamma_{o_i}(t_k)},
% \end{equation}
% where we used that for $f(x)=d_\mathcal{M}(x,y)^2$, $\mathrm{grad}_\mathcal{M}f(x) = -2\log_x(y)$ (see \Cref{lemma:derivative_geodesic_dist}). Provided we can compute efficiently the log map and the derivative of the geodesics curves, this would involve a linear computation overhead with respect to the number of steps, but which could still be worth for large sample scenarios compared to the Wasserstein distance.}
Nonetheless, deriving the Busemann function on a product of Hadamard manifolds is simply the weighted sum of the Busemann function on each geodesic line, and is thus easy to compute provided we know in closed-form the Busemann function on each manifold $\mathcal{M}_i$. This was first observed in \citep[Section II. 8.24]{bridson2013metric} in the case of two manifolds, and we generalize the result to an arbitrary number of manifolds.

\begin{proposition}[Busemann function on product Hadamard manifold] \label[proposition]{prop:busemann_product}
    Let $(\mathcal{M}_i)_{i=1}^n$ be $n$ Hadamard manifolds and $\mathcal{M} = \mathcal{M}_1\times\dots\times\mathcal{M}_n$ the product manifold. Let $\lambda_1,\dots,\lambda_n$ be such that $\sum_{i=1}^n \lambda_i^2 = 1$. For any $i\in\{1,\dots,n\}$, let $\gamma_i$ be a geodesic line on $\mathcal{M}_i$ and define $\gamma:t\mapsto \big(\gamma_1(\lambda_1 t),\dots,\gamma_n(\lambda_n t)\big)$ a geodesic line on $\mathcal{M}$. Then,
    \begin{equation}
        \forall x=(x_1,\dots,x_n)\in \mathcal{M},\ B^\gamma(x) = \sum_{i=1}^n \lambda_i B^{\gamma_i}(x_i).
    \end{equation}
\end{proposition}

% \begin{proof}
%     See \Cref{proof:prop_busemann_product}.
% \end{proof}

In \Cref{xp:comparison_datasets}, we leverage this projection and the corresponding Sliced-Wasserstein distance to compare datasets viewed as distributions on $\mathbb{R}^{d_x}\times \mathbb{H}^{d_y}$.

\section{Properties} \label{section:chsw_properties}

In this section, we derive theoretical properties of the Cartan-Hyperbolic Sliced-Wasserstein distance. First, we will study its topology and the conditions required to have that $\chsw$ is a true distance. In particular, we will first focus on the general case, and then on the specific case of pullback Euclidean metrics. Then, we will study some of its statistical properties. The proofs of this section are postponed to \Cref{proofs:section_chsw_properties}. %\nc{uniformité avec le reste dans la citation des annexes ?}

\subsection{Topology}

\textbf{Distance Property.} First, we are interested in the distance properties of $\chsw$. From the properties of the Wasserstein distance and of the slicing process, we can show that it is a pseudo-distance, \emph{i.e.} that it satisfies the positivity, the symmetry and the triangular inequality.

\begin{proposition} \label[proposition]{prop:chsw_pseudo_distance}
    Let $p\ge 1$, then $\chsw_p$ is a finite pseudo-distance on $\mathcal{P}_p(\mathcal{M})$.
\end{proposition}

% \begin{proof}
%     See \Cref{proof:prop_chsw_pseudo_distance}.
% \end{proof}

\looseness=-1 For now, the lacking property is the one of indiscernibility, \emph{i.e.} that $\chsw_p(\mu,\nu)=0$ implies that $\mu=\nu$. We conjecture that it holds but we have not been able to prove it yet in full generality. In the following, we derive a sufficient condition on a related Radon transform for this property to hold. These derivations are inspired from \citep{boman2009support, bonneel2015sliced}.

Let $f\in L^1(\mathcal{M})$, and let us define, analogously to the Euclidean Radon transform, the Cartan-Hadamard Radon transform $\chr: L^1(\mathcal{M}) \to L^1(\mathbb{R}\times S_o)$ which integrates the function $f$ over a level set of the projection $P^v$:
\begin{equation}
    \forall t\in\mathbb{R},\ \forall v\in S_o,\ \chr f(t,v) = \int_\mathcal{M} f(x) \mathbb{1}_{\{t = P^v(x)\}}\ \mathrm{d}\vol(x).
\end{equation}
% \red{Here, $\mathrm{d}x$ denotes the surface measure on the submanifold $\{x\in\mathcal{M},\ t=P^v(x)\}$.}
Then, we can also define its dual operator $\chr^*:C_0(\mathbb{R}\times S_o)\to C_b(\mathcal{M})$ for $g\in C_0(\mathbb{R}\times S_o)$ where $C_0(\mathbb{R}\times S_o)$ is the space of continuous functions from $\mathbb{R}\times S_o$ to $\mathbb{R}$ that vanish at infinity and $C_b(\mathcal{M})$ is the space of continuous bounded functions from $\mathcal{M}$ to $\mathbb{R}$, as
\begin{equation}
    \forall x\in \mathcal{M},\ \chr^*g(x) = \int_{S_o} g(P^v(x), v)\ \mathrm{d}\lambda_o(v).
\end{equation}
\begin{proposition} \label[proposition]{prop:chsw_dual_rt}
    $\chr^*$ is the dual operator of $\chr$, \emph{i.e.} for all $f\in L^1(\mathcal{M})$, $g\in C_0(\mathbb{R}\times S_o)$,
    \begin{equation}
        \langle \chr f, g\rangle_{\mathbb{R}\times S_o} = \langle f, \chr^* g\rangle_\mathcal{M}.
    \end{equation}
\end{proposition}
% \begin{proof}
%     See \Cref{proof:prop_chsw_dual_rt}.
% \end{proof}
$\chr^*$ maps $C_0(\mathbb{R}\times S_o)$ to $C_b(\mathcal{M})$ because $g$ is necessarily bounded as a continuous function which vanishes at infinity. Note that $\chr^*$ actually maps $C_0(\mathbb{R}\times S_o)$ to $C_0(\mathcal{M})$.
\begin{proposition} \label[proposition]{prop:chr_vanish}
    Let $g\in C_0(\mathbb{R}\times S_o)$, then $\chr^*g \in C_0(\mathcal{M})$.
\end{proposition}
% \begin{proof}
%     See \Cref{proof:prop_chr_vanish}.
% \end{proof}

Let us now recall the disintegration theorem.
\begin{definition}[Disintegration of a measure] \label{def:disintegration}
    Let $(Y,\mathcal{Y})$ and $(Z,\mathcal{Z})$ be measurable spaces, and~$(X,\mathcal{X})=(Y\times Z,\mathcal{Y}\otimes\mathcal{Z})$ the product measurable space. Then, for~$\mu\in\mathcal{P}(X)$, we denote the marginals as $\mu_Y = \pi^Y_\#\mu$ and $\mu_Z=\pi^Z_\#\mu$, where $\pi^Y$ (respectively $\pi^Z$) is the projection on $Y$ (respectively Z). Then, a~family $\big(K(y,\cdot)\big)_{y\in\mathcal{Y}}$ is a disintegration of $\mu$ if for all $y\in Y$, $K(y,\cdot)$ is a measure on $Z$, for~all $A\in\mathcal{Z}$, $K(\cdot,A)$ is measurable and:
    \begin{equation*}
        \forall g\in C(X),\ \int_{Y\times Z} g(y,z)\ \mathrm{d}\mu(y,z) = \int_Y\int_Z g(y,z)K(y,\mathrm{d}z)\ \mathrm{d}\mu_Y(y),
    \end{equation*}
    where $C(X)$ is the set of continuous functions on $X$. We can note $\mu=\mu_Y\otimes K$. $K$ is a probability kernel if for all $y\in Y$, $K(y,Z)=1$.
\end{definition}
The disintegration of a measure actually corresponds to conditional laws in the context of probabilities. In the case where $X=\mathbb{R}^d$, we have existence and uniqueness of the disintegration (see \citep[Box 2.2]{santambrogio2015optimal} or \citep[Chapter 5]{ambrosio2005gradient} for the more general case).

Using the dual operator, we can define the Radon transform of a measure $\mu$ in $\mathcal{M}$ as the measure $\chr \mu$ satisfying 
\begin{equation}
    \forall g\in C_0(\mathbb{R}\times S_o),\ \int_{\mathbb{R}\times S_o} g(t,v)\ \mathrm{d}(\chr\mu)(t, v) = \int_\mathcal{M} \chr^*g(x)\ \mathrm{d}\mu(x).
\end{equation}
$\chr \mu$ being a measure on $\mathbb{R}\times S_o$, we can disintegrate it \emph{w.r.t.} the uniform measure on $S_o$ as $\chr\mu = \lambda_o \otimes K_\mu$ where $K_\mu$ is a probability kernel on $S_o\times \mathcal{B}(\mathbb{R})$. % (see \emph{e.g.} \citep[Box 2.2]{santambrogio2015optimal} for references on the disintegration theorem). 
In the following proposition, we show that for $\lambda_o$-almost every $v\in S_o$, $K(v,\cdot)$ coincides with $P^v_\#\mu$.
\begin{proposition} \label[proposition]{prop:chsw_disintegration}
    Let $\mu$ be a measure on $\mathcal{M}$, and $K_\mu$ a probability kernel on $S_o\times \mathcal{B}(\mathbb{R})$ such that $\chr\mu = \lambda_o\otimes K_\mu$. Then for $\lambda_o$-almost every $v\in S_o$, $K_\mu(v,\cdot) = P^v_\#\mu$.
\end{proposition}
% \begin{proof}
%     See \Cref{proof:prop_chsw_disintegration}.
% \end{proof}

All these derivations allow to link the Cartan-Hadamard Sliced-Wasserstein distance with the Radon transform defined with the corresponding projection (geodesic or horospherical). Then, $\chsw_p(\mu,\nu)=0$ implies that for $\lambda_o$-almost every $v\in S_o$, $P^v_\#\mu=P^v_\#\nu$. Showing that the Radon transform is injective would allow to conclude that $\mu=\nu$.

Actually, here we derived two different Cartan-Hadamard Radon transforms. Using $P^v$ as the geodesic projection, the Radon transform integrates over geodesic subspaces of dimension $\mathrm{dim}(\mathcal{M})-1$. Such spaces are totally geodesic subspaces, and are related to the more general geodesic Radon transform \citep{rubin2003notes}. In the case where the geodesic subspace is of dimension one, \emph{i.e.} it integrates only over geodesics, this coincides with the X-ray transform, and it has been studied \emph{e.g.} in \citep{lehtonen2018tensor}. Here, we are interested in the case of dimension $\mathrm{dim}(\mathcal{M})-1$, which, to the best of our knowledge, has only been studied in \citep{lehtonen2016geodesic} in the case where $\mathrm{dim}(\mathcal{M})=2$ and hence when the geodesic Radon transform and the X-ray transform coincide. However, no results on the injectivity over the sets of measures is yet available. In the case where $P^v$ is the Busemann projection, the set of integration is a horosphere. General horospherical Radon transforms on Cartan-Hadamard manifolds have not yet been studied to the best of our knowledge.

% \red{Note that the Cartan-Hadamard Sliced-Wasserstein distance can be embedded by the isometry $\mu\mapsto \lambda_o\otimes K_\mu$ in the space 
% \begin{equation}
%     \mathcal{P}_p^{\lambda_o}(S_o\times\mathbb{R}) =\{\gamma\in \mathcal{P}(S_o\times\mathbb{R}),\ \gamma=\lambda_o\otimes K \text{ and for $\lambda_o$-a.e. $v\in S_o$},\ \int |t|^p\ K(v,\mathrm{d}t) < \infty\}    
% \end{equation}
% endowed with the disintegrated Monge-Kantorovich metric
% \begin{equation}
%     DMK_p^{\lambda_o}(\mu,\nu) = \int_{S_o} W_p^p\big(K_\mu(v,\cdot), K_\nu(v,\cdot)\big)\ \mathrm{d}\lambda_o(v),
% \end{equation}
% recently introduced by \citet{kitagawa2023new}, which happens to be a geodesic space in contrast with the Sliced-Wasserstein space \citep{candau_tilh, park2023geometry, kitagawa2023new}.
% }

% \red{Connection with Radon transform + distance if RT distance.}

% \red{Refs Radon transform on 2D Cartan-Hadamard manifolds : \citep{lehtonen2016geodesic}, on X-Ray transforms (integrate over geodesic but not over totally geodesic of dim d-1) \citep{lehtonen2018tensor}}

\textbf{Link with the Wasserstein Distance.} An important property of the Sliced-Wasserstein distance on Euclidean spaces is that it is topologically equivalent to the Wasserstein distance, \emph{i.e.} it metrizes the weak convergence. Such results rely on properties of the Fourier transform which do not translate straightforwardly to manifolds. Hence, deriving such results will require further investigations. We note that a possible lead for the horospherical case is the connection between the Busemann function and the Fourier-Helgason transform \citep{biswas2018fourier, sonoda2022fully}. Using that the projections are Lipschitz functions, we can still show that $\chsw$ is a lower bound of the geodesic Wasserstein distance.

\begin{proposition} \label[proposition]{prop:chsw_upperbound}
    Let $\mu,\nu\in \mathcal{P}_p(\mathcal{M})$, then
    \begin{equation}
        \chsw_p^p(\mu,\nu) \le W_p^p(\mu,\nu).
    \end{equation}
\end{proposition}

% \begin{proof}
%     See \Cref{proof:prop_chsw_upperbound}.
% \end{proof}

\looseness=-1 This property means that it induces a weaker topology compared to the Wasserstein distance, which can be computationally beneficial but which also comes with less discriminative powers \citep{nadjahi2020statistical}.

\textbf{Hilbert Embedding.} $\chsw$ also comes with the interesting properties that it can be embedded in Hilbert spaces. This is in contrast with the Wasserstein distance which is known to not be Hilbertian \citep[Section 8.3]{peyre2019computational} except in one dimension where it coincides with its sliced counterpart.

\begin{proposition} \label[proposition]{prop:chsw_hilbertian}
    Let $p\ge 1$ and $\mathcal{H}=L^p([0,1]\times S_o, \mathrm{Leb}\otimes \lambda_o)$. We define $\Phi$ as 
    \begin{equation}
        \begin{aligned}
        \Phi : \ &\mathcal{P}_p(\mathcal{M}) \rightarrow \mathcal{H}\\
        &\mu \mapsto \big( (q, v) \mapsto F^{-1}_{P^v_{\#}\mu}(q) \big),
        \end{aligned}
    \end{equation}
    where $F^{-1}_{P^v_\#\mu}$ is the quantile function of $P^v_\#\mu$. Then $\chsw_p$ is Hilbertian and for all $\mu,\nu\in \mathcal{P}_p(\mathcal{M})$,
    \begin{equation}
        \chsw_p^p(\mu,\nu) = \|\Phi(\mu)-\Phi(\nu)\|_\mathcal{H}^p.
    \end{equation}
\end{proposition}

% \begin{proof}
%     See \Cref{proof:prop_chsw_hilbertian}.
% \end{proof}

This is a nice property which allows to define a valid positive definite kernel for measures such as the Gaussian kernel \citep[Theorem 6.1]{jayasumana2015kernel}, and hence to use kernel methods \citep{hofmann2008kernel}. This can allow for example to perform distribution clustering, classification \citep{kolouri2016sliced, carriere2017sliced} or regression \citep{meunier2022distribution}.

\begin{proposition} \label[proposition]{prop:chsw_gaussian_kernel}
    Define the kernel $K:\mathcal{P}_2(\mathcal{M})\times \mathcal{P}_2(\mathcal{M})\to \mathbb{R}$ as $K(\mu,\nu) = \exp\big(-\gamma \chsw_2^2(\mu,\nu)\big)$ for $\gamma>0$. Then $K$ is a positive definite kernel. 
    % \red{TODO: $K(\mu_i,\mu_j) = \exp(-\gamma \chsw_2^2(\mu_i,\mu_j))$ is a positive definite kernel. See \citep[Th 5]{kolouri2016sliced}}
\end{proposition}

\begin{proof}
    Apply \citep[Theorem 6.1]{jayasumana2015kernel}.
\end{proof}

\citet{bonet2023sliced} notably used this property to perform brain-age regression by first representing M/EEG data as a probability distribution of SPD matrices. And then by plugging the Gaussian kernel constructed with the Cartan-Hadamard Sliced-Wasserstein on the space of SPDs endowed with the Log-Euclidean metric, that we presented in \Cref{section:spdpem}, into the kernel Ridge regression method.

Note that to show that the Gaussian kernel is universal, \emph{i.e.} that the resulting Reproducing Kernel Hilbert Space (RKHS) is powerful enough to approximate any continuous function \citep{meunier2022distribution}, we would need additional results such as that it metrizes the weak convergence and that $\chsw_2$ is a distance, as shown in \citep[Proposition 7]{meunier2022distribution}.

\subsection{Topology for Pullback Euclidean Manifolds}

In this section, we focus on particular Hadamard manifolds for which the metric is a pullback Euclidean metric, which allows inheriting the properties of Euclidean spaces, and deriving additional properties of the corresponding Sliced-Wasserstein distance. This covers for example the space of SPD matrices with Pullback Euclidean metrics studied in \Cref{section:spdpem} as well as the Mahalanobis manifold introduced in \Cref{section:pem}.

% Formally, let $\mathcal{N}$ be an Euclidean space and denote $\langle\cdot,\cdot\rangle$ its inner product and $\|\cdot\|$ the associated norm. Let $\mathcal{M}$ be some space and $\phi:\mathcal{M}\to\mathcal{N}$ be a diffeomorphism. Then, defining for any $x\in \mathcal{M}$ and $u,v\in T_x\mathcal{M}$ the metric $g^\phi(u,v) = \langle \phi_{*,x}(u), \phi_{*,x}(v)\rangle$ where $\phi_{*,x}:T_x\mathcal{M}\to T_{\phi(x)}\mathcal{N}$ is the differential of $\phi$ at $x$, $(\mathcal{M}, g^\phi)$ is a Riemannian manifold with geodesic distance $d_\mathcal{M}(x,y)=\|\phi(x)-\phi(y)\|$ (see \Cref{th:pem}).

% This covers for example the space of SPD matrices with Pullback Euclidean metrics studied in \Cref{section:spdpem}, $\mathbb{R}^d$ endowed with the Mahalanobis distance with $\phi(x) = A^{\frac12}x$ and $A\in S_d^{++}(\mathbb{R})$ or more generally any squared geodesic distance on $\mathbb{R}^d$ for which $\phi_{*,x}(u) = A(x)^{\frac12}u$ with $A:\mathbb{R}^d\to S_d^{++}(\mathbb{R})$ \citep{scarvelis2023riemannian, pooladian2023neural}. 

Let $\mathcal{N}$ be an Euclidean space with $\langle\cdot,\cdot\rangle$ its inner product and $\|\cdot\|$ the associated norm. Let $\phi:\mathcal{M}\to \mathcal{N}$ be a diffeomorphism and denote $(\mathcal{M}, g^\phi)$ the resulting Riemannian manifold (see \Cref{th:pem} for more details).
We recall that by \Cref{prop:proj_coord_pullback}, the projection of $x\in \mathcal{M}$ on the geodesic characterized by the direction $v\in S_o$ is of the form
\begin{equation}
    P^v(x) = \langle \phi(x) - \phi(o), \phi_{*,o}(v)\rangle.
\end{equation}
In this case, given the formula, we can link $\chsw$ with the Euclidean $\sw$ with the integration made on $S_{\phi(o)}=\{v\in T_{\phi(o)}\mathcal{N},\ \|v\|_{\phi(o)}=1\}$ with respect to the measure $\lambda_{\phi(o)}$. %$(\phi_{*,o})_\#\lambda$.

\begin{lemma} \label[lemma]{lemma:chsw_pullback}
    Let $(\mathcal{M}, g^\phi)$ a pullback Euclidean Riemannian manifold and assume that $\lambda_{\phi(o)}=(\phi_{*,o})_\#\lambda_o$. Let $p\ge 1$ and $\mu,\nu\in\mathcal{P}_p(\mathcal{M})$. Then,
    \begin{equation}
        \chsw_p^p(\mu,\nu) = \int_{S_{\phi(o)}} W_p^p(Q^v_\#\phi_\#\mu, Q^v_\#\phi_\#\nu)\ \mathrm{d}\lambda_{\phi(o)}(v) = \sw_p^p(\phi_\#\mu, \phi_\#\nu),
    \end{equation}
    with $Q^v(x) = \langle x, v\rangle$ and $\sw_p$ the Euclidean Sliced-Wasserstein distance. %$\sw_p^p(\mu,\nu;\lambda) = \int_{S_o} W_p^p(Q^v_\#\mu,Q^v_\#\nu)\ \mathrm{d}\lambda(v)$ the Euclidean Sliced-Wasserstein distance with slicing distribution $\lambda$.
\end{lemma}

% \begin{lemma} \label[lemma]{lemma:chsw_pullback}
%     Let $(\mathcal{M}, g^\phi)$ a pullback Euclidean Riemannian manifold. Let $p\ge 1$ and $\mu,\nu\in\mathcal{P}_p(\mathcal{M})$. Then,
%     \begin{equation}
%         \chsw_p^p(\mu,\nu) = \int_{S_{\phi(o)}} W_p^p(Q^v_\#\phi_\#\mu, Q^v_\#\phi_\#\nu)\ \mathrm{d}\big((\phi_{*,o}\big)_\#\lambda_o)(v) = \sw_p^p(\phi_\#\mu, \phi_\#\nu; (\phi_{*,o})_\#\lambda_o),
%     \end{equation}
%     with $Q^v(x) = \langle x, v\rangle$ and $\sw_p^p(\mu,\nu;\lambda) = \int_{S_o} W_p^p(Q^v_\#\mu,Q^v_\#\nu)\ \mathrm{d}\lambda(v)$ the Euclidean Sliced-Wasserstein distance with slicing distribution $\lambda$.
% \end{lemma}
% \looseness=-1 Note that $\phi_{*,o}$ is a bijection and thus the support of $(\phi_{*,o})_\#\lambda_o$ is $S_{\phi(o)}$. However, if it is not an isometry, the corresponding slicing distribution $(\phi_{*,o})_\#\lambda$ on $S_{\phi(o)}$ might not be uniform.

Using this simple lemma, we can leverage results known for the Euclidean Sliced-Wasserstein distance to $\chsw$ on these particular spaces. First, we show that we recover the distance property by additionally showing the indiscernible property.

\begin{proposition} \label[proposition]{prop:distance_chsw_pullback}
    Let $(\mathcal{M}, g^\phi)$ a pullback Euclidean Riemannian manifold. Let $p\ge 1$, then $\chsw_p$ is a finite distance on $\mathcal{P}_p(\mathcal{M})$.
\end{proposition}

% \begin{proof}
%     See \Cref{proof:prop_distance_chsw_pullback}.
% \end{proof}

We can also obtain the important property that $\chsw$ metrizes the weak convergence as the Wasserstein distance \citep{villani2009optimal}. This property was first shown for arbitrary measures in \citep{nadjahi2019asymptotic} for the regular Euclidean $\sw$.

\begin{proposition} \label[proposition]{prop:weak_cv_chsw_pullback}
    Let $(\mathcal{M}, g^\phi)$ a pullback Euclidean Riemannian manifold of dimension $d$. Let $p\ge 1$, $(\mu_n)_n$ a sequence in $\mathcal{P}_p(\mathcal{M})$ and $\mu\in\mathcal{P}_p(\mathcal{M})$. Then, $\lim_{n\to\infty} \chsw_p(\mu_n,\mu) = 0$ if and only if $(\mu_n)_n$ converges weakly towards $\mu$.
\end{proposition}

% \begin{proof}
%     See \Cref{proof:prop_weak_cv_chsw_pullback}.
% \end{proof}

With these additional properties, we can also show that the corresponding Gaussian kernel is universal applying \citep[Theorem 4]{meunier2022distribution}. 
In addition to \Cref{prop:chsw_upperbound}, we show that we can lower bound $\chsw$ with the Wasserstein distance when the measures are compactly supported.

\begin{proposition} \label[proposition]{prop:lowerbound_chsw_pullback}
    Let $(\mathcal{M}, g^\phi)$ a pullback Euclidean Riemannian manifold of dimension $d$. Let $p\ge 1$, $r>0$ a radius and $B(o,r)=\{x\in\mathcal{M},\ d_\mathcal{M}(x,o)\le r\}$ a closed ball. Then there exists a constant $C_{d,p,r}$ such that for all $\mu,\nu\in\mathcal{P}_p\big(B(o,r)\big)$,
    \begin{equation}
        W_p^p(\mu,\nu) \le C_{d,p,r} \chsw_p(\mu,\nu)^{\frac{1}{d+1}}.
    \end{equation}
\end{proposition}

% Assume now that $\phi_{*,o} = \id$. This is for example the case for the Log-Euclidean metric on SPDs. Then, in addition to \Cref{prop:chsw_upperbound}, we show that we can lower bound $\chsw$ with the Wasserstein distance when the measures are compactly supported.

% \begin{proposition} \label[proposition]{prop:lowerbound_chsw_pullback}
%     Let $(\mathcal{M}, g^\phi)$ a pullback Euclidean Riemannian manifold of dimension $d$ with $\phi_{*,o}=\id$. Let $p\ge 1$, $r>0$ a radius and $B(o,r)=\{x\in\mathcal{M},\ d_\mathcal{M}(x,o)\le r\}$ a closed ball. Then there exists a constant $C_{d,p,r}$ such that for all $\mu,\nu\in\mathcal{P}_p\big(B(o,r)\big)$,
%     \begin{equation}
%         W_p^p(\mu,\nu) \le C_{d,p,r} \chsw_p(\mu,\nu)^{\frac{1}{d+1}}.
%     \end{equation}
% \end{proposition}

% \begin{proof}
%     See \Cref{proof:prop_lowerbound_chsw_pullback}.
% \end{proof}

\subsection{Statistical Properties}

\textbf{Sample Complexity.} \looseness=-1 In practical settings, we usually cannot directly compute the closed-form between $\mu,\nu\in\mathcal{P}_p(\mathcal{M})$, but we have access to samples $x_1,\dots,x_n\sim \mu$ and $y_1,\dots,y_n\sim \nu$. Then, it is common practice to estimate the discrepancy with the plug-in estimator $\chsw_p^p(\hat{\mu}_n,\hat{\nu}_n)$ \citep{manole2021plugin,manole2022minimax,niles2022estimation} where $\hat{\mu}_n = \frac{1}{n}\sum_{i=1}^n \delta_{x_i}$ and $\hat{\nu}_n = \frac{1}{n}\sum_{i=1}^n \delta_{y_i}$ are empirical estimations of the measures. We are interested in characterizing the speed of convergence of the plug-in estimator towards the true distance. Relying on the proof of \citet{nadjahi2020statistical}, we derive in \Cref{prop:chsw_sample_complexity} the sample complexity of $\chsw$. As in the Euclidean case, we find that the sample complexity does not depend on the dimension, which is an important and appealing property of sliced divergences \citep{nadjahi2020statistical} compared to the Wasserstein distance, which has a sample complexity in $O(n^{-1/d})$ \citep{niles2022estimation}.

\begin{proposition} \label[proposition]{prop:chsw_sample_complexity}
    Let $p\ge 1$, $q>p$ and $\mu,\nu\in\mathcal{P}_p(\mathcal{M})$. Denote $\hat{\mu}_n$ and $\hat{\nu}_n$ their counterpart empirical measures and $M_q(\mu) = \int_\mathcal{M} d(x,o)^q\ \mathrm{d}\mu(x)$ their moments of order $q$. Then, there exists $C_{p,q}$ a constant depending only on $p$ and $q$ such that 
    \begin{equation}
        \mathbb{E}\big[|\chsw_p(\hat{\mu}_n, \hat{\nu}_n) - \chsw_p(\mu,\nu)|\big] \le 2\alpha_{n,p,q}C_{p,q}^{1/p} \big(M_q(\mu)^{1/q} + M_q(\nu)^{1/q}\big), 
    \end{equation}
    where
    \begin{equation}
        \alpha_{n,p,q} = \left\{ 
        \begin{array}{ll}
          n^{-1/(2p)} & \mbox{ if } q > 2p, \\
          n^{-1/(2p)} \log(n)^{1/p} & \mbox{ if } q = 2p, \\
          n^{-(q-p)/(pq)} & \mbox{ if } q \in (p,2p).
        \end{array}
        \right.
    \end{equation}
\end{proposition}

% \begin{proof}
%     See \Cref{proof:prop_chsw_sample_complexity}.
% \end{proof}

This property is very appealing in practical settings as it allows to use the same number of samples while having the same convergence rate in any dimension. In practice though, we cannot compute exactly $\chsw_p(\hat{\mu}_n, \hat{\nu}_n)$ as the integral on $S_o$ \emph{w.r.t.} the uniform measure $\lambda_o$ is intractable.

% Add \citep[Proposition 4]{xi2022distributional}?

\textbf{Projection Complexity.} Thus, to compute it in practice, we usually rely on a Monte-Carlo approximation, by drawing $L\ge 1$ directions $v_1,\dots,v_L$ and approximating the distance by $\widehat{\chsw}_{p,L}$ defined between $\mu,\nu\in\mathcal{P}_p(\mathcal{M})$ as
\begin{equation}
    \widehat\chsw_{p,L}^p(\mu,\nu) = \frac{1}{L} \sum_{\ell=1}^L W_p^p(P^{v_\ell}_\#\mu, P^{v_\ell}_\#\nu).
\end{equation}
In the following proposition, we derive the Monte-Carlo error of this approximation, and we show that we recover the classical rate of $O(1/\sqrt{L})$.

\begin{proposition} \label[proposition]{prop:chsw_proj_complexity}
    Let $p\ge 1$, $\mu,\nu\in\mathcal{P}_p(\mathcal{M})$. Then, the error made by the Monte-Carlo estimate of $\chsw_p$ with $L$ projections can be bounded as follows
    \begin{equation}
        \mathbb{E}_v\left[|\widehat{\chsw}_{p,L}^p(\mu,\nu) - \chsw_p^p(\mu,\nu)|\right]^2 \le \frac{1}{L} \mathrm{Var}_v\left[W_p^p(P^v_\#\mu, P^v_\#\nu)\right].
    \end{equation}
\end{proposition}

% \begin{proof}
%     See \Cref{proof:prop_chsw_proj_complexity}.
% \end{proof}

We note that here the dimension actually intervenes in the term of variance. % $\mathrm{Var}_v\left[W_p^p(P^v_\#\mu, P^v_\#\nu)\right]$.

\textbf{Computational Complexity.} As we project on the real line, the complexity of computing the Wasserstein distances between each projected distribution is in $O(Ln\log n)$. Then, we add the complexity of computing the projections, which will depend on the spaces and whether or not we have access to a closed-form.

\section{Application of Cartan-Hadamard Sliced-Wasserstein Distances} \label{section:applications}

\looseness=-1 In this section, we provide some illustrations of Cartan-Hadamard Sliced-Wasserstein distances on manifolds which were not yet studied in previous works. We note that \citet{bonet2022hyperbolic} used $\hsw$ to perform deep classification with prototypes on Hyperbolic spaces, while \citet{bonet2023sliced} used $\spdsw$ to perform domain adaptation for Brain Computer Interface and to perform Brain-Age regression by leveraging the Gaussian kernel from \Cref{prop:chsw_gaussian_kernel} and plugging it into Kernel Ridge regression. Here, we first provide an experiment using the Mahalanobis Sliced-Wasserstein distance to classify documents, and then an experiment on a product of Cartan-Hadamard manifolds to compare datasets.

\subsection{Document classification with Mahalanobis Sliced-Wasserstein} \label{sec:xp_mahalanobis}

\begin{table}[t]
    \centering
    % \vspace{-15pt}
    \caption{Accuracy on Document Classification}
    \resizebox{0.6\linewidth}{!}{
        \begin{tabular}{ccccc}
            & \textbf{BBCSport} & \textbf{Movies} & \textbf{Goodreads genre} & \textbf{Goodreads like} \\ \toprule
            $W_2$ & 94.55 & 74.44 & 56.18 & 71.00 \\
            $W_A$ & 98.36 & 76.04 & 56.81  & 68.37 \\
            $\sw_2$ & $89.42_{\pm 0.89}$ & $67.27_{\pm 0.69}$ & $50.01_{\pm 1.21}$ & $65.90_{\pm 0.17}$ \\
            $\sw_{2,A}$ & $97.58_{\pm 0.04}$ & $76.55_{\pm 0.11}$ & $57.03_{\pm 0.68}$ & $67.54_{\pm 0.14}$ \\
            \bottomrule
        \end{tabular}
    }
    % \vspace{-10pt}
    \label{tab:table_acc}
\end{table}

We propose here to perform an experiment of document classification. Suppose that we have $N$ documents $D_1,\dots,D_N$. Following the work of \citet{kusner2015word}, we represent each document $D_k$ as a distribution over words. More precisely, denote $x_1\dots,x_n\in \mathbb{R}^d$ the set of words, embedded using \texttt{word2vec} \citep{mikolov2013distributed} in dimension $d=300$. Then, $D_k$ is represented by the probability distribution $D_k = \sum_{i=1}^{n} w_i^k \delta_{x_i}$, where $w_i^k$ represents the frequency of the word $x_i$ in $D_k$ normalized such that $\sum_{i=1}^{n} w_i^k = 1$. 

Then, following \citep{huang2016supervised}, we learn a matrix $A\in S_d^{++}(\mathbb{R})$ using the Neighborhood Component Analysis (NCA) method \citep{goldberger2004neighbourhood} combined with the Word Centroid Distance (WCD), defined as $\mathrm{WCD}_A(D_k, D_\ell)^2 = (Xw^k - Xw^\ell)^T A (Xw^k - Xw^\ell)$ with $X= (x_1, \dots, x_n)\in\mathbb{R}^{d\times n}$. We use the \texttt{pytorch-metric-learning} library \citep{musgrave2020pytorch} to learn $A$.

\begin{table}[t]
    \centering
    % \vspace{-15pt}
    \caption{Runtimes on Document Classification}
    \resizebox{0.6\linewidth}{!}{
        \begin{tabular}{ccccc}
            & & \textbf{BBCSport} & \textbf{Movies} & \textbf{Goodreads} \\ \toprule
            \multirow{2}{*}{$W_A$} & Average ($\cdot10^{-3}$ s) & $3.29_{\pm 1.61}$ & $6.78_{\pm 2.74}$ & $440.30_{\pm 259}$ \\
            & Full (s) & 891 & 13544 & 221252 \\
            \multirow{2}{*}{$\sw_{2,A}$} & Average ($\cdot 10^{-3}$s) & $2.45_{\pm 0.008}$ & $2.47_{\pm 0.04}$ & $2.5_{\pm 0.12}$ \\
            & Full (s) & 665 & 4931 & 1256 \\
            \bottomrule
        \end{tabular}
    }
    % \vspace{-10pt}
    \label{tab:table_runtime}
\end{table}

\begin{figure}[t]
    \centering
    \hspace*{\fill}
    % \subfloat{\includegraphics[width=0.3\linewidth]{Figures/Classif_Docs/runtime_BBC.png}} \hfill
    % \subfloat{\includegraphics[width=0.3\linewidth]{Figures/Classif_Docs/runtime_movies.png}} \hfill
    % \subfloat{\includegraphics[width=0.3\linewidth]{Figures/Classif_Docs/runtime_goodreads.png}} \hfill
    \subfloat{\includegraphics[width=0.3\linewidth]{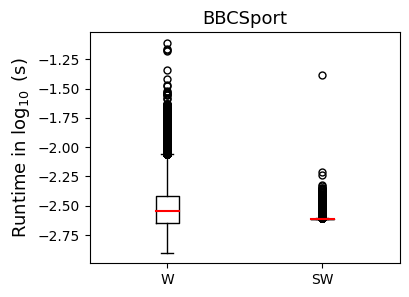}} \hfill
    \subfloat{\includegraphics[width=0.3\linewidth]{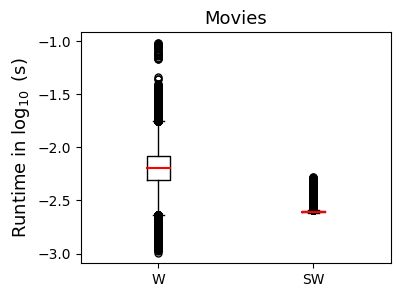}} \hfill
    \subfloat{\includegraphics[width=0.3\linewidth]{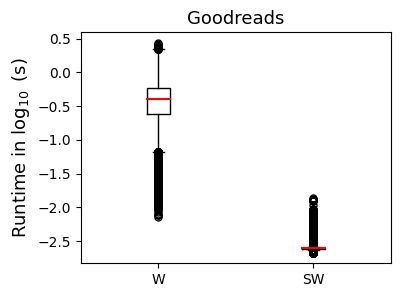}}
    \hspace*{\fill}
    \caption{Runtime between each pair of documents.}
    \label{fig:boxplot_runtimes}
\end{figure}

Once $A$ is learned, we compute the distance between documents using the Wasserstein distance or the Sliced-Wasserstein distance with Mahalanobis ground cost distance, \emph{i.e.} $d_A(x,y)^2 = (x-y)^T A (x-y)$. Once we compute the distance between each documents $\big(d(D_k,D_\ell)\big)_{k,\ell}$, we use a $k$-nearest neighbor classifier. On \Cref{tab:table_acc}, we report the results for the BBCSport dataset \citep{kusner2015word}, the Movies reviews dataset \citep{pang2002thumbs} and the Goodread dataset \citep{maharjan2017multi}. All the datasets are split in 5 different train/test sets. The number of neighbors is found using a cross validation. We compare the results when using the regular Wasserstein and Sliced-Wasserstein distances, \emph{i.e.} with $A=I_d$, and when learning $A$ using NCA with the WCD metric. The Wasserstein distance is computed using the Python Optimal Transport library \texttt{POT} \citep{flamary2021pot}. The results for SW are averaged over 3 runs and $\sw$ is approximated with $L=500$ projections.

With this simple initialization, we observe that the results obtained with the Mahalanobis Sliced-Wasserstein distance become very competitive with the ones obtained using the Wasserstein distance with the Mahalanobis ground cost. We note that the results might be further improved by performing then a NCA with $W_A$ or $\sw_{2,A}$ as distances in the same spirit of \citep{huang2016supervised}. Here, we just use an initialization through WCD as a proof of concept to demonstrate how much it can already improve the results when using $\sw$ with a carefully chosen groundcost distance.

\looseness=-1 We showcase the computational benefits of using the Sliced-Wasserstein distance compared to the Wasserstein distance on \Cref{fig:boxplot_runtimes} by plotting the runtime of comparing each pair of documents and on \Cref{tab:table_runtime} with the full runtimes. We note that the Wasserstein distance is computed on CPU while the Sliced-Wasserstein distance is implemented in Pytorch and uses GPU. We used as CPU an Intel Xeon 4214 and as GPU a Titan RTX. We observe a computational gain even on small scale datasets where the documents contain few words, and therefore for which the underlying representative distributions contain few samples. For datasets with distributions with a larger number of samples such as goodreads, the computational benefits are pretty big. We sum up the statistics of the different datasets in \Cref{tab:summary_docs}.

\subsection{Datasets Comparisons with Sliced-Wasserstein on a Product Manifold} \label{xp:comparison_datasets}

\begin{figure}[t]
    \centering
    \hspace*{\fill}
    \subfloat[SW.]{\label{fig:otdd_sw}\includegraphics[width={0.45\linewidth}]{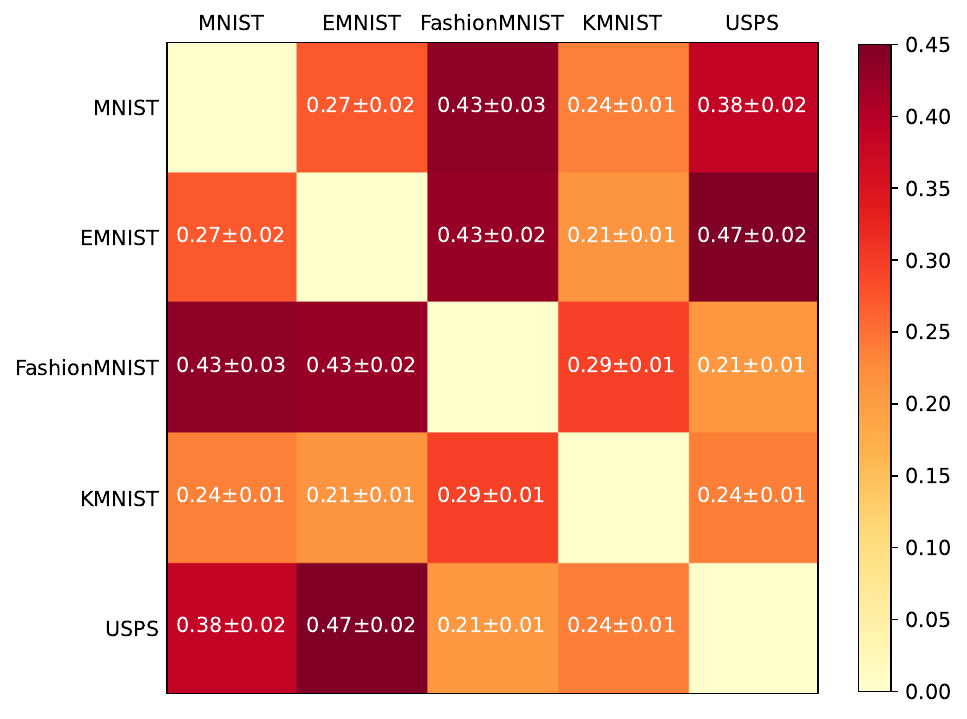}} \hfill
    \subfloat[Product HCHSW.]{\label{fig:otdd_psw}\includegraphics[width={0.45\linewidth}]{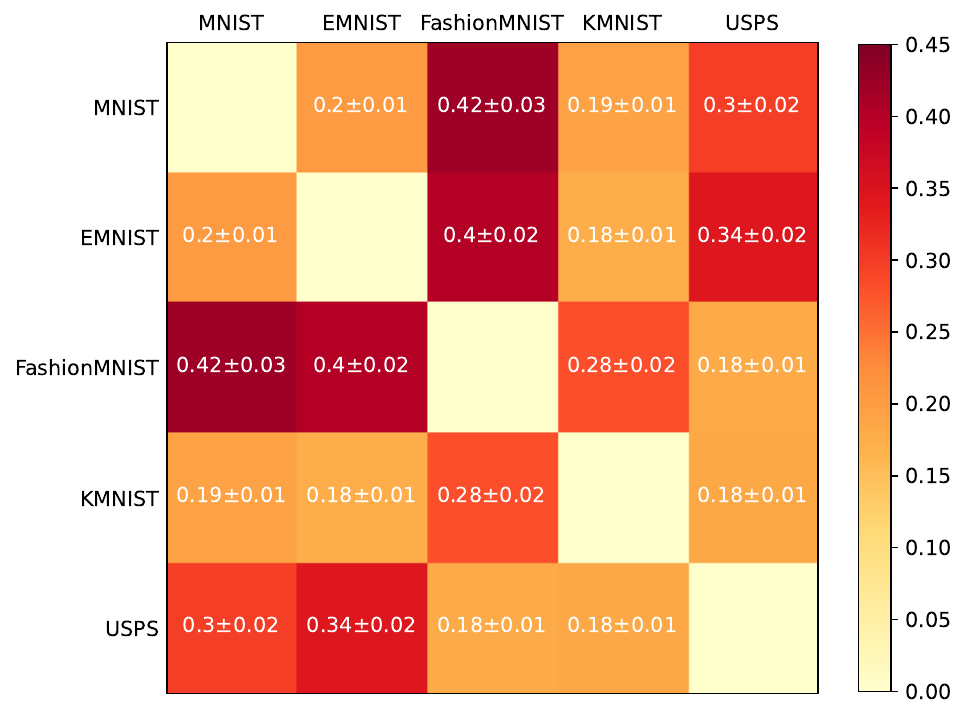}} 
    \hspace*{\fill}
    \caption{Comparison between $\sw$ between the datasets and $\chsw$ between the datasets embedded on $\mathbb{R}^{d_x} \times \mathbb{L}^{d_y}$. Results are averaged over 100 draws of projections.}
    \label{fig:otdd}
\end{figure}

Assume we have datasets defined as sets of feature-label pairs $(x,y)\in \mathcal{X}\times\mathcal{Y}$ \citep{alvarez2020geometric}, where the samples are in $\mathbb{R}^{d_x}$ and the labels are embedded in a Hyperbolic space $\mathbb{H}^{d_y}$. Then, a  
% We use an embedding of the labels in a Hyperbolic space. Then, following \citep{alvarez2020geometric}, we consider datasets as sets of feature-label pairs $(x,y)\in \mathbb{R}^{d_x}\times\mathbb{H}^{d_y}$. A
dataset $D_i$ can then be seen as a probability distribution on $\mathbb{R}^{d_x}\times \mathbb{H}^{d_y}$ which we can compare using $\chsw$ on product manifolds.

We assume that the datasets are already embedded in such spaces. In practice, such embedding could come up for instance when we are given image-text pairs, which could be embedded both in Hyperbolic spaces \emph{e.g.} using \citep{desai2023hyperbolic}, or for more classical datasets using label embeddings methods \citep{akata2015label}. 

Here, to get a dataset represented in $\mathcal{P}(\mathbb{R}^{d_x}\times \mathbb{H}^{d_y})$, we follow \citep{liu2022wasserstein} and use a multidimensional scaling (MDS) method in hyperbolic spaces \citep{walter2004h,cvetkovski2011multidimensional} to get an embedding $\psi:\mathcal{P}(\mathbb{R}^{d_x})\to \mathbb{H}^{d_y}$ into the hyperbolic space such that, for $\nu_y$ denoting the conditional probability distribution of samples in $\mathbb{R}^{d_x}$ with labels $y\in\mathcal{Y}$,
\begin{equation}
    W_2^2(\nu_y, \nu_{y'}) \approx \alpha \cdot d_\mathbb{H}\big(\psi(\nu_y), \psi(\nu_{y'})\big)^2,
\end{equation}
with $\alpha$ some scaling parameter. To find this embedding, we minimize the absolute different squared loss \citep{cvetkovski2011multidimensional} defined as, for an original distance matrix $\Delta = (\delta_{i,j})_{i,j}$ and a scaling factor $\sqrt{\alpha}$, 
\begin{equation}
    \forall z_1,\dots,z_n\in \mathbb{L}^{d_y},\ \mathcal{L}(z) = \sum_{i=1}^n \sum_{j=i+1}^n \big( d_\mathbb{L}(z_i, z_j) - \sqrt{\alpha} \delta_{ij}\big)^2.
\end{equation}
To improve the numerical stability, we perform the optimization in the tangent space following \citep{mishne2023numerical} using the parametrization 
\begin{equation}
    z_i = \exp_{x^0}\big((0, \Tilde{z}_i)\big) = \left(\cosh(\|\Tilde{z}_i\|), \sinh(\|\Tilde{z}_i\|)\frac{\Tilde{z}_i}{\|\Tilde{z}_i\|}\right)    
\end{equation}
for $\Tilde{z}_i\in\mathbb{R}^{d_y-1}$, and then performing the optimization in the Euclidean space.

\looseness=-1 We focus here on \emph{*NIST} datasets which include MNIST \citep{lecun-mnisthandwrittendigit-2010}, EMNIST \citep{cohen2017emnist}, FashionMNIST \citep{xiao2017fashion}, KMNIST \citep{clanuwat2018deep} and USPS \citep{hull1994database}. We plot on \Cref{fig:otdd} the matrix distance obtained between the \emph{*NIST} datasets either using $\sw$ between the datasets seen only through their features, \emph{i.e.} with $D _i \in \mathcal{P}(\mathbb{R}^{d_x})$, and using $\chsw$ on the space $\mathcal{P}(\mathbb{R}^{d_x}\times \mathbb{L}^{d_y})$ where the labels were embedded on $\mathbb{L}^{d_y}$ using the method described in the previous paragraph with a scaling of $\sqrt{\alpha}=0.1$ and $d_y=10$. We observe that when the labels are not taken into account, the USPS and MNIST datasets have a huge discrepancy between them. However, when taking into account the labels, we recover that these two datasets are in fact more similar as they both represent numbers. Thus, we argue that using the sliced distance on the product dataset in order to take into account the labels provides better comparisons of the datasets. Furthermore, from a computation point of view, $\chsw$ on the product manifold is much cheaper compared to \emph{e.g.} computing the Wasserstein distance. On our experiments, computing the full distance matrix with $\chsw$ took in average 0.05s against 120s to compute the Wasserstein distance, where we used here only 10000 samples of the datasets.

\section{Cartan-Hadamard Sliced-Wasserstein Flows} \label{section:chswf}

We propose here to derive the Wasserstein gradient flows of the CHSW distances along with approximated non-parametric particle schemes. We provide first the results on general Hadamard manifolds and then we specify them to Mahalanobis manifolds, Hyperbolic spaces and SPDs endowed with the Log-Euclidean metric. The proofs of this section are postponed to \Cref{proofs:section_chswf}.

\subsection{Wasserstein Gradient Flows}

\textbf{First Variations.} \looseness=-1 Being discrepancies on Hadamard manifolds, CHSW discrepancies can be used to learn parametric or empirical distributions through their minimization. %minimizing it
A possible solution is to leverage Wasserstein gradient flows \citep{ambrosio2005gradient, santambrogio2017euclidean} of $\mathcal{F}(\mu) = \frac12 \chsw_2^2(\mu,\nu)$ where $\nu$ is some target distribution. 
% To generate new samples from $\nu$, we can for instance use the forward Euler scheme, as done previously in \citep{liutkus2019sliced} for the Euclidean SW, which requires to compute its first variation. 
Approximating this flow would then allow to provide new samples from $\nu$. Computing such a flow requires first computing the first variations of the given functional.
% The first variation can also be used to analyze theoretically the convergence of the Wasserstein gradient flow. 
As a first step towards computing Wasserstein gradient flows of $\chsw$ on Hadamard spaces, and analyzing them, we derive in \Cref{prop:chsw_1st_variation} the first variation of $\mathcal{F}$.

\begin{proposition} \label[proposition]{prop:chsw_1st_variation}
    Let $K$ be a compact subset of $\mathcal{M}$, $\mu,\nu\in\mathcal{P}_2(K)$ with $\mu\ll \vol$. Let $v\in S_o$, denote $\psi_v$ the Kantorovich potential between $P^v_\#\mu$ and $P^v_\#\nu$ for the cost $c(x,y)=\frac12 |x-y|^2$ for $x,y\in\mathbb{R}$. Let $\xi$ be a diffeomorphic vector field on $K$ and denote for all $\epsilon\ge 0$, $T_\epsilon : K \to \mathcal{M}$ defined as $T_\epsilon(x) = \exp_x\big(\epsilon \xi(x)\big)$ for all $x\in K$. Then,
    \begin{equation}
        \begin{aligned}
            &\lim_{\epsilon\to 0^+}\ \frac{\chsw_2^2\big((T_\epsilon)_\#\mu,\nu\big) - \chsw_2^2(\mu,\nu)}{2\epsilon} \\ &= \int_{S_o} \int_{\mathcal{M}} \psi_v'\big(P^v(x)\big)\langle \mathrm{grad}_\mathcal{M} P^v(x),\xi(x)\rangle_x \ \mathrm{d}\mu(x)\ \mathrm{d}\lambda_o(v).
        \end{aligned}
    \end{equation}
\end{proposition}
% \begin{proof}
%     See \Cref{proof:prop_chsw_1st_variation}.
% \end{proof}
In the Euclidean case, we recover well the first variation formula for $\sw$ first derived in \citep[Proposition 5.1.7]{bonnotte2013unidimensional} as in this case, for $x\in \mathbb{R}^d$, $T_\epsilon(x) = x + \epsilon \xi(x)$, and for $\theta\in S^{d-1}$, $P^\theta(x) = \langle x,\theta\rangle$ and thus $\mathrm{grad}P^\theta(x) =\nabla P^\theta(x) = \theta$, and we recover
\begin{equation}
    \lim_{\epsilon\to 0^+}\ \frac{\sw_2^2\big((\mathrm{Id}+\epsilon\xi)_\#\mu, \nu\big)- \sw_2^2(\mu,\nu)}{2\epsilon} = \int_{S^{d-1}} \int_{\mathbb{R}^d} \psi_\theta'\big(P^\theta(x)\big) \big\langle \theta, \xi(x)\big\rangle\ \mathrm{d}\mu(x)\ \mathrm{d}\lambda(\theta).
\end{equation}

\textbf{Cartan-Hadamard Sliced-Wasserstein Flow.} Given the first variation, we can derive the Wasserstein gradient flow of $\mathcal{F}(\mu)=\frac12\chsw_2^2(\mu,\nu)$ as the continuity equation governed by the vector field $v_t$ obtained through the Wasserstein gradient
\begin{equation}
    \forall x\in \mathcal{M},\ v_t(x) = -\nabla_{W_2}\mathcal{F}(\mu_t)(x) = -\int_{S_o} \psi_{t,v}'\big(P^v(x)\big)\mathrm{grad}_\mathcal{M}P^v(x)\ \mathrm{d}\lambda_o(v),
\end{equation}
with $\psi_{t,v}$ the Kantorovich potential between $P^v_\#\mu_t$ and $P^v_\#\nu$ such that $\psi_{t,v}'(x) = x - F_{P^v_\#\mu_t}^{-1}\big(F_{P^v_\#\nu}(x)\big)$,
\emph{i.e.} the Wasserstein gradient flow $(\mu_t)_{t\ge 0}$ of $\mathcal{F}$ is a solution (in the distributional sense) of
\begin{equation}
    \partial_t\mu_t + \mathrm{div}(\mu_t v_t) = 0.
\end{equation}

\textbf{Forward Euler Scheme.} To provide an algorithm to sample from $\nu$ by minimizing $\mathcal{F}(\mu)=\frac12\chsw_2^2(\mu,\nu)$ while following its Wasserstein gradient flow, there are several possible strategies of discretization of the flow. For instance, a solution could be to compute the backward Euler scheme, also known as the Jordan-Kinderlehrer-Otto (JKO) scheme from the seminal work of \citet{jordan1998variational}. 
This strategy has for example been used to minimize the Sliced-Wasserstein distance in \citep{bonet2022efficient}. Here, we propose instead to use the forward Euler scheme, which allows defining a particle scheme approximating the trajectory of the Wasserstein gradient flow. Such a strategy has been used to minimize different functionals such as the MMD \citep{arbel2019maximum}, the Kernel Stein Discrepancy \citep{korba2021kernel} or the KL divergence \citep{fang2021kernel, wang2022projected}. For SW, \citet{liutkus2019sliced} proposed to minimize $\sw$ with an entropy term, which required to use a McKean Vlasov SDE.

Let $\mu_0\in\mathcal{P}_p(\mathcal{M})$ and $\tau>0$. On a Riemannian manifold, analogously to the Riemannian gradient descent \citep{bonnabel2013stochastic}, the forward Euler scheme becomes
\begin{equation}
    \forall k\ge 0,\ \mu_{k+1} = \exp_{\id}\big(-\tau \nabla_{W_2}\mathcal{F}(\mu_k)\big)_\#\mu_k,
\end{equation}
where $\nabla_{W_2}\mathcal{F}(\mu_k)(x)=-v_k(x) = \int_{S_o} \psi_{k,v}'\big(P^v(x)\big)\mathrm{grad}_\mathcal{M}P^v(x)\ \mathrm{d}\lambda_o(v)$ for $x\in\mathcal{M}$ is the Wasserstein gradient. In the Euclidean case, we recover the usual forward Euler scheme $\mu_{k+1} = \big(\id-\tau\nabla_{W_2}\mathcal{F}(\mu_k)\big)_\#\mu_k$.

In practice, we approximate the Wasserstein gradient by first sampling $v_1,\dots,v_L\sim \lambda_o$ and using
\begin{equation} \label{eq:formula_velocity}
    \forall x\in \mathcal{M},\ \hat{v}_k(x) = -\frac{1}{L} \sum_{\ell=1}^L \psi_{v_\ell,k}'\big(P^{v_\ell}(x)\big)\mathrm{grad}_\mathcal{M}P^{v_\ell}(x),
\end{equation}
where 
\begin{equation}
    \psi_{v, k}'\big(P^v(x)\big) = P^v(x) - F_{P^v_\#\nu}^{-1}\big(F_{P^v_\#\mu_k}(P^v(x))\big).
\end{equation}
Following \citep{liutkus2019sliced}, the cumulative distribution functions and the quantiles are approximated using linear interpolations between the true points\footnote{using \url{https://github.com/aliutkus/torchinterp1d}}. Finally, the particle scheme is given by,
\begin{equation}
    \forall k\ge 0, i\in\{1,\dots,n\},\ x_i^{k+1} = \exp_{x^k_i}\big(\tau\hat{v}_k(x^k_i)\big).
\end{equation}
We sum up the procedure in \Cref{algo:gradient_flows_chsw}.

\begin{algorithm}[tb]
   \caption{Wasserstein gradient flows of $\chsw$}
   \label{algo:gradient_flows_chsw}
    \begin{algorithmic}
       \STATE {\bfseries Input:} $(y_j)_{j=1}^n\sim \nu$, $\mu_0$, $L$ the number of projections, $N$ the number of steps
       \STATE Sample $(x^0_i)_{i=1}^n\sim \mu_0$
       % \STATE Compute $\hat{y}_j = P^v(y_j)$
       \FOR{$k=0$ {\bfseries to} $N-1$}
       \STATE Draw $v_1,\dots,v_L\sim\lambda_o$
       \STATE Compute $\hat{x}_{i,\ell}^k=P^{v_\ell}(x_i^k)$, $\hat{y}_{j,\ell} = P^{v_\ell}(y_j)$ for all $\ell\in\{1,\dots,L\}$
       \STATE Define $P^{v_\ell}_\#\hat{\nu} = \frac{1}{n}\sum_{j=1}^n \delta_{\hat{y}_{j,\ell}}$, $P^{v_\ell}_\#\hat{\mu}_k = \frac{1}{n} \sum_{i=1}^n \delta_{\hat{x}_{i,\ell}^k}$
       \STATE Compute $\hat{z}_{i,\ell}^k = \hat{x}_{i,\ell}^k - F_{P^{v_\ell}_\#\hat{\nu}}^{-1}(F_{P^{v_\ell}_\#\hat{\mu}_k}(\hat{x}_{i,\ell}^k))$
       \STATE Compute $g_\ell(x_i^k) = \mathrm{grad}_\mathcal{M}P^{v_\ell}(x_i^k)$
       \STATE Compute $\hat{v}_k(x_i^k) = \frac{1}{L} \sum_{\ell=1}^L (\hat{x}_{i,\ell}^k-\hat{z}_{i,\ell}^k) g_\ell(x_i^k)$
       \STATE For all $i\in\{1,\dots,n\}$, $x_i^{k+1} = \exp_{x_i^k}\big(\tau \hat{v}_k(x_i^k)\big)$
       \ENDFOR
    \end{algorithmic}
\end{algorithm}

% \red{Theoretical results to look at(?): Compute Hessian to have an idea of the cv? cf MMDWGF, show that $\frac{\mathrm{d}}{\mathrm{d}t}\chsw_2^2(\mu_t,\nu) = - \int \|v_t\|_{\mu_t}\ \mathrm{d}t$...}

\subsection{Application to the Mahalanobis Manifold}

% \begin{figure}[t]
%     \centering
%     \includegraphics[width=0.5\linewidth]{Figures/SWFs/Trajectories_swf_vs_mswf.pdf}
%     \caption{Trajectory of the Sliced-Wasserstein flows with the Euclidean SW and the Mahalanobis SW with a randomly chosen $A\in S_d^{++}(\mathbb{R})$, for 15 particles starting from the same positions.}
%     \label{fig:traj_swfs_vs_mswfs}
% \end{figure}

% \begin{figure}[t]
%     \centering
%     \includegraphics[width=\linewidth]{Figures/SWFs/Trajectories_mahalanobis.pdf}
%     \caption{\textbf{(First row)} Trajectories of Mahalanobis Sliced-Wasserstein flows using for SPD matrices along the geodesic between $I_2$ and a randomly chosen $A\in S_d^{++}(\mathbb{R})$. \textbf{(Second row)} Corresponding ellipse representing matrix $A_t$.}
%     \label{fig:traj_mahalanobis_swfs}
% \end{figure}

\begin{figure}[t]
    \centering
    \includegraphics[width=\linewidth]{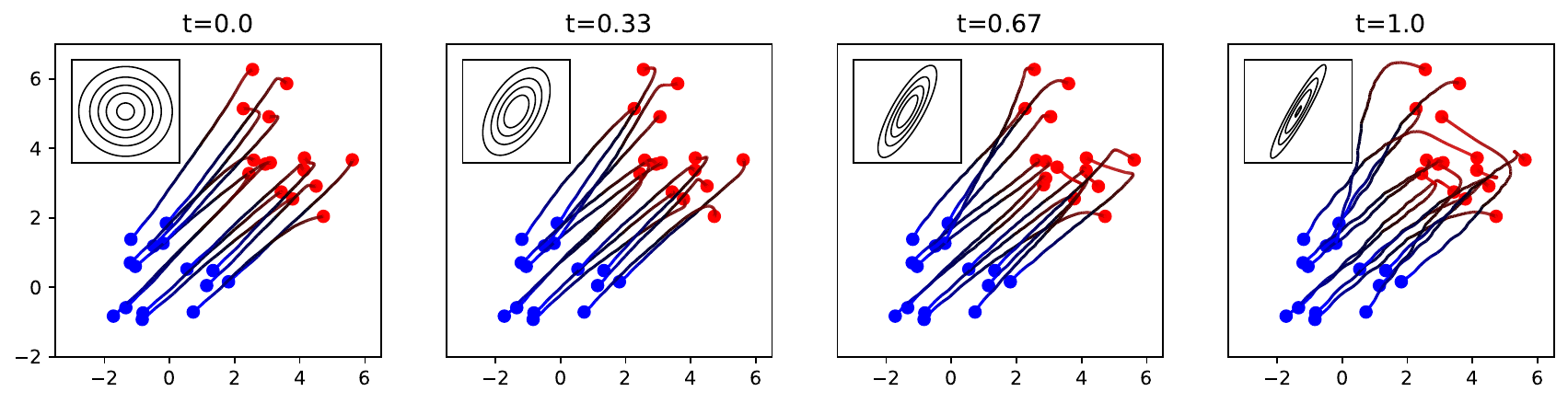}
    \caption{Trajectories of Mahalanobis Sliced-Wasserstein flows using four SPD matrices $A_t$ along the geodesic between $I_2$ and a randomly chosen $A\in S_d^{++}(\mathbb{R})$. Ellipses represent the matrices $A_t$.}
    \label{fig:traj_mahalanobis_swfs}
\end{figure}

For pullback Euclidean metrics, the Riemannian gradient can be obtained by using the inverse of the differential operator as stated in the following lemma.

\begin{lemma}[Lemma 4 in \citep{chen2023riemannian_multiclass}] \label[lemma]{lemma:grad_pem}
    Let $(\mathcal{M}, g^\phi)$ be a Pullback Euclidean Riemannian manifold. For $f:\mathcal{M}\to\mathbb{R}$ a smooth map, the gradient is of the form
    \begin{equation}
        \forall x\in\mathcal{M},\ \mathrm{grad}_\mathcal{M}f(x) = \phi_{*, x}^{-1}\big(\phi_{*, x}^{-*}\big(\nabla f(x)\big)\big).
    \end{equation}
\end{lemma}

For the Mahalanobis distance, \emph{i.e.} for $\phi(x) = A^{\frac12}x$ for any $x\in\mathbb{R}^d$ with $A\in S_d^{++}(\mathbb{R})$, the inverse of the differential is simply $\phi_{*,x}^{-1}(v) = A^{-\frac12}v$, and we recall that the projection is $P^v(x)=x^TAv$ for $v\in S_o$. Thus the Riemannian gradient of the projection $P^v$ for $v\in S_o$ is 
\begin{equation}
    \mathrm{grad}_\mathcal{M}P^v(x) = A^{-\frac12}\big(A^{-\frac12}(Av)\big) = v.
\end{equation}
\looseness=-1 We recover the same gradient. But the matrix $A$ is still involved in the formula of the projection, which can change the trajectory of the particles. Choosing well the matrix $A$ can help improving the convergence of flows for ill conditioned problems, see \emph{e.g.} \citep{duchi2011adaptive, dong2023particlebased}.

\looseness=-1 We illustrate on \Cref{fig:traj_mahalanobis_swfs} the effect on the trajectory when using a randomly sampled SPD matrix $A$ to specify the Mahalanobis distance compared to the classical Euclidean metric. We plot the trajectories for different SPDs obtained on the geodesic between $I_2$ and $A$, which is of the form $A_t = \exp\big(t\log (A)\big)$ for $t\in [0,1]$ when using the Affine-Invariant metric.
% We use the same learning rate, but since the convergence speed changes because of the scaling of the gradient, the Mahalanobis version took more iterations to converge. We plot the results at convergence for every $t$.

\subsection{Application to Hyperbolic Spaces}

% Now, we focus on the case of Hyperbolic spaces. To the best of our knowledge, there are very few non-parametric algorithms which were proposed on such a space, with the notable exception of \citep{said2021bayesian} in which a MCMC algorithm on Hadamard manifolds is proposed and discussed.

Here, we propose to minimize the Hyperbolic Sliced-Wasserstein distances in order to derive a new non-parametric scheme allowing to learn a distribution given its samples. We first recall how to compute the gradient on the Lorentz model.

\begin{proposition} \label[proposition]{prop:grad_lorentz}
    Let $f:\mathbb{L}^d_K\to\mathbb{R}$ and note $\Bar{f}:\mathbb{R}^{d+1}\to\mathbb{R}$ a smooth extension on $\mathbb{R}^{d+1}$. Then, the gradient of $f$ at $x\in\mathbb{L}_K^d$ is
    \begin{equation}
        \mathrm{grad}_{\mathbb{L}^d_K}f(x) = \mathrm{Proj}_x^K\big(-KJ\nabla\Bar{f}(x)\big),
    \end{equation}
    where $J=\mathrm{diag}(-1,1,\dots,1)$ and
    \begin{equation}
        \mathrm{Proj}_x^K(z) = z - K\langle x,z\rangle_\mathbb{L} x.
    \end{equation}
\end{proposition}

\begin{proof}
    We extend \citep[Proposition 7.7]{boumal2023introduction} to $\mathbb{L}^d_K$.
\end{proof}

Then, leveraging \Cref{prop:grad_lorentz}, we derive the closed-forms of the gradients of the geodesic and horospherical projections, which allows deriving the forward Euler scheme of this functional, by plugging the different formulas in \Cref{eq:formula_velocity}.

\begin{proposition} \label[proposition]{prop:grad_lorentz_projs}
    Let $v\in T_{x^0}\mathbb{L}^d_K\cap S^d$ and $x\in\mathbb{L}^d_K$, then
    \begin{align}
        &\mathrm{grad}_{\mathbb{L}^d_K}B^v(x) = K\sqrt{-K} \left( Kx - \frac{\sqrt{-K}x^0 + v}{\langle x,\sqrt{-K}x^0 + v\rangle_\mathbb{L}} \right), \\
        &\mathrm{grad}_{\mathbb{L}^d_K}P^v(x) = \frac{K^2\big(\langle x,x^0\rangle_\mathbb{L} v - \langle x,v\rangle_\mathbb{L} x^0\big)}{\langle x,v\rangle_\mathbb{L}^2 + K\langle x,x^0\rangle_\mathbb{L}^2}.
    \end{align}
\end{proposition}

On $\mathbb{B}^d$, the gradient can be obtained by rescaling the Euclidean gradient with the inverse Poincaré ball metric \citep{nickel2017poincare} which is $\left(\frac{1+K\|x\|_2^2}{2}\right)^2$ \citep{park2021unsupervised}. Thus, we can also derive the corresponding formulas on the Poincaré ball. For example, for the Busemann function, we have 
\begin{equation}
    \nabla B^{\Tilde{v}}(x) = 2\left(\frac{x}{1-\|x\|_2^2} - \frac{\Tilde{v}-x}{\|\Tilde{v}-x\|_2^2}\right),
\end{equation}
and therefore its Riemannian gradient is 
\begin{equation}
    \mathrm{grad}_{\mathbb{B}^d_K}B^{\Tilde{v}}(x) = \left(\frac{1+K\|x\|_2^2}{2}\right)^2 \nabla B^{\Tilde{v}}(x).
\end{equation}

% \begin{example}[HHSW on Poincaré ball]

%     For example, we have 
%     \begin{equation}
%         \nabla B^{\Tilde{v}}(x) = 2\left(\frac{x}{1-\|x\|_2^2} - \frac{\Tilde{v}-x}{\|\Tilde{v}-x\|_2^2}\right).
%     \end{equation}
%     Thus,
%     \begin{equation}
%         \mathrm{grad}_{\mathbb{B}^d_K}B^{\Tilde{v}}(x) = \left(\frac{1+K\|x\|_2^2}{2}\right)^2 \nabla B^{\Tilde{v}}(x).
%     \end{equation}
% \end{example}

\begin{figure}[t]
    \centering
    \includegraphics[width=\linewidth]{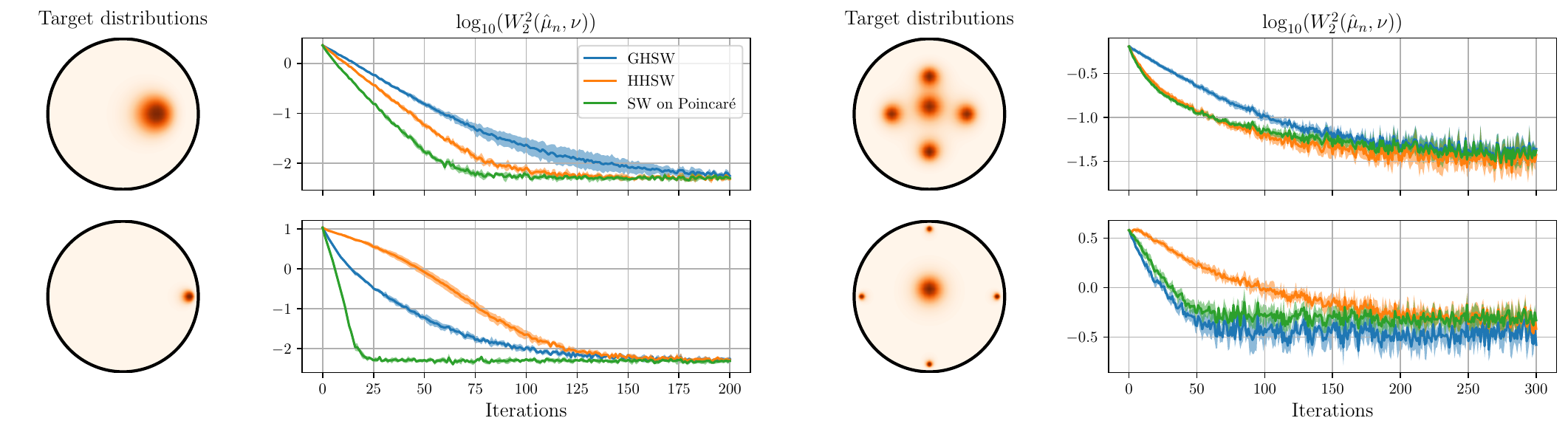}
    \caption{Log 2-Wasserstein between the target distribution and particles obtained from HSWFs (averaged over 5 runs).}
    \label{fig:comparison_hswfs}
\end{figure}

On \Cref{fig:comparison_hswfs}, we plot the 2-Wasserstein distance between the target distribution and samples from the Hyperbolic Sliced-Wasserstein Flows on Hyperbolic space of curvature $K=-1$. We compare the evolution between $\ghsw$, $\hhsw$ and $\sw$ (on the Poincaré ball for $\sw$) on 4 different scenarios. The two first ones involve a target distribution which is a Wrapped Normal Distribution (WND) located either close to the center or to the border of the disk. The second ones involve a mixture of WND, with some modes either close to the border or to the center. $\hhsw$ and $\ghsw$ can be done both on the Lorentz model or the Poincaré ball. Using either model give similar results. As hyperparameters, we chose $n=500$ particles, a learning rate of $\tau=0.1$ with $N=200$ epochs for centered targets, and $\tau=0.5$ and $N=300$ epochs for bordered targets. We note that the three gradient flows perform likewise, with an advantage of speed for $\sw$, which might be due to the fact that the minimization is done in the space of Euclidean probabilities, and thus does not take into account that the modes are actually on the border. %However, we note that for mixtures on the border, $\ghsw$ actually performs better, but not all samples converged to the modes. \nc{?? how do you see that from the Figure ?}

We add on \Cref{fig:trajectories_wnd_border} and \Cref{fig:trajectories_mwnd_border} trajectories for the border scenarios. When minimizing $\ghsw$, particles tend to go to the modes by following the shortest path, while when minimizing $\hhsw$, they tend to first go at the border before converging to the modes. As the distance on the border of the Poincaré disk are bigger than to the center, this may explain the observed slower convergence of $\hhsw$ in \Cref{fig:comparison_hswfs}.

\begin{figure}[t]
    \centering
    \hspace*{\fill}
    \subfloat[WND]{\label{fig:trajectories_wnd_border}\includegraphics[width={0.45\linewidth}]{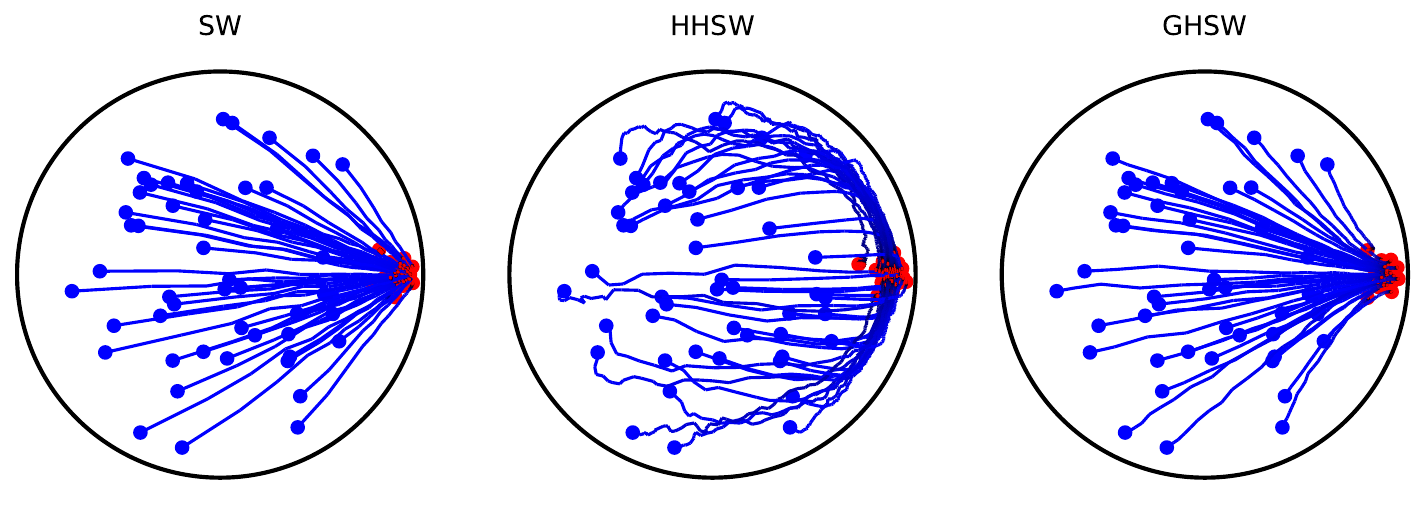}} \hfill
    \subfloat[MWND]{\label{fig:trajectories_mwnd_border}\includegraphics[width={0.45\linewidth}]{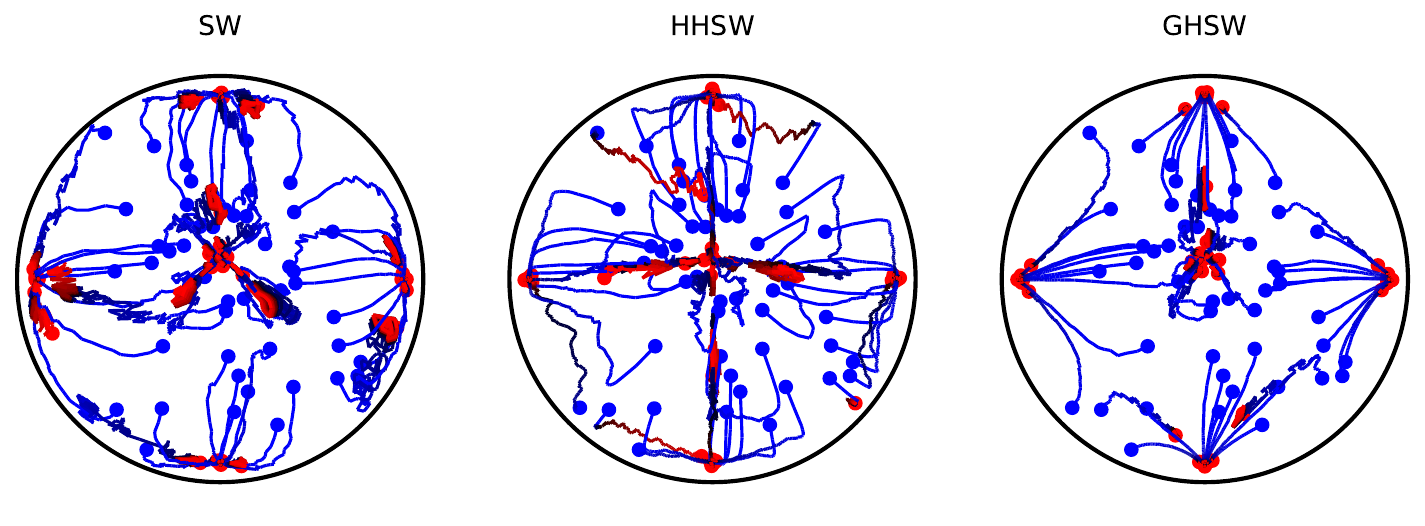}} \hfill
    \hspace*{\fill}
    \caption{Trajectories of 50 particles when the target is the WND on the border or the Mixture of WND on the border.}
    \label{fig:trajectories_hswfs}
\end{figure}

\subsection{Application to SPD matrices with the Log-Euclidean Metric}

% \red{See \citep[Lemma 4]{chen2023riemannian_multiclass} for the results for Pullback Euclidean metric: $\mathrm{grad}_\mathcal{M} f(X) = \phi_{*,X}^{-1}(\phi_{*,X}^{-*})(\nabla f(X))$ and \citep[Table 1]{chen2023riemannian_multiclass} for AIM: $\mathrm{grad}_\mathcal{M}f(X) = X \frac{\nabla f(X) + \nabla f(X)^T}{2} X$. For LEM, it is also precised in \citep[Proposition C.1]{chen2023adaptive} and in \citep[Section 3.2.2]{pennec2020manifold}. See also \citep{al2013computing}?}

% \red{Let $X\in S_d^{++}(\mathbb{R})$ with eigenvectors $u_1,\dots,u_d$ and eigenvalues $\lambda_1,\dots,\lambda_d$ and such that $X=UDU^T$. For $\phi(X)=\log(X)$, $\phi_{*,X}(V) = Q+Q^T+W$ where $Q = D_U \log(D) U^T$ with $D_U = \big((\lambda_1 I_d - X)^\dag V u_1,\dots,(\lambda_d I_d - S)^\dag V u_d\big)$, $\cdot^\dag$ the Moore-Penrose inverse, and $W = U \mathrm{diag}\big(u_1^T V u_1 / \lambda_1,\dots, u_d^T Vu_d / \lambda_d\big)U^T$ \citep[Proposition C.1]{chen2023adaptive}. Need to invert this for the Riemannian gradient? + compute adjoint?}

% \red{See also \texttt{geomstats} for the code \citep{miolane2020geomstats}.}

% \red{$P^A_{*,X}(V) = \tr\big(A^T(Q+Q^T+W)\big)$}

For $\spdsw$ with the Log-Euclidean metric, the formula of the gradient can be derived using the inverse of the differential as stated in \Cref{lemma:grad_pem}. We report the inverse of the differential form of the log in \Cref{lemma:inverse_differential_log}.

\begin{lemma} \label[lemma]{lemma:inverse_differential_log}
    Let $\phi:X\mapsto \log(X)$ and $X=UDU^T \in S_d^{++}(\mathbb{R})$ where $D=\mathrm{diag}(\lambda_1,\dots,\lambda_d)$. Then, we have
    \begin{equation}
        \forall W\in T_{\phi(X)}S_d^{++}(\mathbb{R}),\ \phi_{*,X}^{-1}(W) = U \Tilde{\Sigma}(W) U^T,
    \end{equation}
    where $\Tilde{\Sigma}(W) = U^T W U \oslash \Gamma$ with $\Gamma$ defined as in \Cref{lemma:pennec_diff_log}.
    % \begin{equation}
    %     \Gamma = \begin{pmatrix}
    %         \frac{1}{\lambda_1} & \frac{\log \lambda_1 - \log \lambda_2}{\lambda_1-\lambda_2} & \dots  & \\
    %         \frac{\log \lambda_2 - \log \lambda_1}{\lambda_2-\lambda_1} & \frac{1}{\lambda_2} & \ddots & \\
    %         \vdots & \ddots &  \ddots & \\
    %         & & & \frac{1}{\lambda_d}
    %     \end{pmatrix}.
    % \end{equation}
\end{lemma}

% \begin{proof}
%     See \Cref{proof:lemma_inverse_differential_log}
% \end{proof}

Finally, in \Cref{lemma:grad_proj_le}, we report the gradient of the projection obtained with the Log-Euclidean metric, which can be obtained using that the differential of the matrix log satisfies $\langle A, \log_{*,X}(V)\rangle_F = \langle \log_{*,X}(A), V\rangle_F$ for any $A, V\in S_d(\mathbb{R})$.

\begin{lemma} \label[lemma]{lemma:grad_proj_le}
    Let $A\in S_d(\mathbb{R})$ and $X=UDU^T\in S_d^{++}(\mathbb{R})$ with $U=\mathrm{diag}(\lambda_1,\dots,\lambda_d)$. Then,
    \begin{equation}
        \nabla P^A(X) = U\Sigma(A)U^T.
    \end{equation}
\end{lemma}

% \begin{proof}
%     See \Cref{proof:lemma_grad_proj_le}.
% \end{proof}

\looseness=-1 We now have all the tools to apply \Cref{algo:gradient_flows_chsw} for the particular case of $\spdsw$. In \Cref{fig:spdswfs}, trajectories plotted inside the $S_2^{++}(\mathbb{R})$ cone depict the evolution of the matrices along the gradient flow. The noisy behavior of some of them can be mostly explained by numerical instabilities arising from the different matrix operators used in the process, which require to use small step sizes.

% \red{We plot some trajectories on \Cref{fig:spdswfs}. 
% We experienced some numerical instabilities on this example, which required to reduce a lot the step size, which explains the Brownian type trajectories.} \nc{alternative: In \Cref{fig:spdswfs} trajectories  plotted inside the $S_2^{++}(\mathbb{R})$ cone depict the evolution of the matrices along the gradient flow. The noisy behavior of some of them can be mostly explained by numerical instabilities arising from the different matrix operators used in the process.}

% \red{TODO: ajouter comparaisons, plot, trajectoires ?}

% \red{MCMC in SPD space \citep{holbrook2018geodesic, yu2023riemannian, bharath2023sampling}}

\begin{figure}[t]
    \centering
    \includegraphics[width=0.4\linewidth]{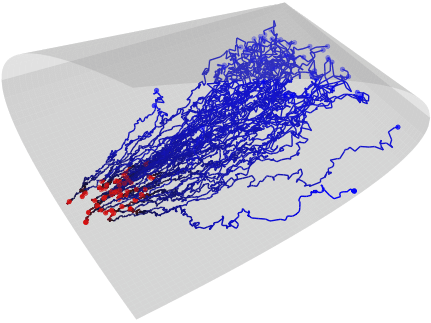}
    \caption{Trajectories of particles following the Wasserstein gradient flow of $\spdsw$.}
    \label{fig:spdswfs}
\end{figure}

\section{Future Works and Discussions}

In this article, we introduced formally a way to generalize the Sliced-Wasserstein distance on Riemannian manifolds of non-positive curvature and we specified this construction to different particular cases: pullback Euclidean metrics, Hyperbolic spaces, the space of Symmetric Positive-Definite matrices and product of Hadamard manifolds. These new discrepancies can be computed very efficiently and scale to distributions composed of a large of number of samples in contrast to the computation of the Wasserstein distance. We also analyzed these constructions theoretically while providing new applications and non-parametric schemes to minimize them using Wasserstein gradient flows. %\nc{peut être redire pourquoi SW c'est intéressant wrt. temps de calcul}

Further works might include studying other Hadamard manifolds for which we do not necessarily have a closed-form for the projections such as Siegel spaces \citep{cabanes2022apprentissage} or extending this construction to more general manifolds such as Riemannian manifolds of non-negative curvature, Finsler manifolds \citep{shen2001lectures} which have recently received some attention in Machine Learning \citep{lopez2021symmetric, pouplin2022identifying, lin2023hyperbolic}, or more generally metric spaces.

% Further works might include constructing SW type distances on geodesically complete Riemannian manifolds of non-negative curvature. Such spaces have more complicated geometries which makes it harder to build a general construction. Hence, we will focus in \Cref{chapter:ssw} on the particular case of the hypersphere, which is a space of positive constant curvature.

% Besides constructing SW distances on Riemannian manifolds, one could also be interested in extending the constructions on more general metric spaces. A particular class of such space with appealing properties, and which encloses Hadamard manifolds, are CAT(0) spaces \citep{bridson2013metric}. Optimal transport on these classes of metric spaces have recently received some attention \citep{berdellima2023existence}. We could also study generalization of Riemannian manifolds such as Finsler manifolds \citep{shen2001lectures} which have recently received some attention in Machine Learning \citep{lopez2021symmetric, pouplin2022identifying}.

For the projections, we studied two natural generalizations of the projection used in Euclidean spaces. We could also study other projections which do not follow geodesics subspaces or horospheres, but are well suited to Riemannian manifolds, in the same spirit of the Generalized Sliced-Wasserstein. Other subspaces could also be used, such as Hilbert curves \citep{bernton2019approximate, li2022hilbert} adapted to manifolds, or higher dimensional subspaces \citep{paty2019subspace, chami2021horopca}. Finally, we could also define other variations of $\chsw$ such as $\text{max-CHSW}$ for instance and more generally adapt many of the variants which have been proposed for SW to the case of Riemannian manifolds. Note also that we could plug these constructions into the framework introduced by \citet{sejourne2023unbalanced} in order to compare positive measures on Hadamard manifolds.

% Note also that the Busemann function is an example of a more broad class of functions called horofunctions. On Hadamard manifolds, horofunctions are necessarily Busemann functions, but it might not be the case on more general metric spaces.

On the theoretical side, we still need to show that these Sliced-Wasserstein discrepancies are proper distances by showing the indiscernible property. It might also be interesting to study whether statistical properties for the Euclidean SW distance derived in \emph{e.g.} \citep{nietert2022statistical, manole2022minimax, goldfeld2022statistical, xu2022central, xi2022distributional} still hold more generally for $\chsw$ on any Cartan-Hadamard manifold, or to study the properties of the space of probabilities endowed with these distances, such as geodesic properties or the gradient flows in this space, as it was recently done in \citep{candau_tilh, bonet2022efficient, park2023geometry, kitagawa2023new} for the Euclidean Sliced-Wasserstein distance.

% \red{ Future works? Gradient flows (add first variations and examples if it done?), distances (link with Fourier-Helgason?), other metrics, Busemann on spaces not geodesically complete, other spaces such as Grassmannian (plutot à la fin), product of manifolds (ça aussi)... CAT(0) spaces... Horofunctions (=Busemann functions on Hadamard manifold, but not necessary the case more generally)? More general metric spaces}

% \red{Tori: positive, negative and null sectional curvatures?}

% \red{Look at \citep{bacak2014convex}: Tangent cones on geodesic metrics spaces to characterize directions}

% \red{In particular, \citet{lin2023hyperbolic} proposed a method to embed hierarchical data into a product of hyperbolic spaces. (? put in discussion and conclusion because use $\ell^1$-distance to get the hierarchy, and thus is not a Riemannian manifold?)}

% \red{Mentionner \citep{kitagawa2023new, park2023geometry}}

% Acknowledgements and Disclosure of Funding should go at the end, before appendices and references

% \acks{All acknowledgements go at the end of the paper before appendices and references.
% Moreover, you are required to declare funding (financial activities supporting the
% submitted work) and competing interests (related financial activities outside the submitted work).
% More information about this disclosure can be found on the JMLR website.}

\pagebreak

\acks{This research was funded by project DynaLearn from Labex CominLabs and Region Bretagne ARED DLearnMe, and by the project OTTOPIA ANR-20-CHIA-0030 of the French National Research Agency (ANR). Clément Bonet’s research was partially supported by the center Hi! PARIS.}


\newpage

\appendix
\section{Useful Lemmas}

% \subsection{Lemmas}

We derive here some lemmas which will be useful for the proofs.

\begin{lemma}[Lemma 6 in \citep{paty2019subspace}] \label[lemma]{lemma:paty}
    Let $\mathcal{M}$, $\mathcal{N}$ be two Riemannian manifolds. Let $f:\mathcal{M}\to \mathcal{N}$ be a measurable map and $\mu,\nu \in \mathcal{P}(\mathcal{M})$. Then,
    \begin{equation}
        \Pi(f_\#\mu,f_\#\nu) = \{(f\otimes f)_\#\gamma,\ \gamma\in\Pi(\mu,\nu)\}.
    \end{equation}
\end{lemma}

\begin{proof}
    This is a straightforward extension of \citep[Lemma 6]{paty2019subspace}.
\end{proof}

\begin{lemma} \label[lemma]{lemma:lipschitz}
    Let $(\mathcal{M}, g)$ be a Hadamard manifold with origin $o$. Let $v\in T_o\mathcal{M}$, then
    \begin{enumerate}
        \item the geodesic projection $P^v$ is 1-Lipschitz.
        \item the Busemann function $B^v$ is 1-Lipschitz.
    \end{enumerate}
\end{lemma}

\begin{proof} \leavevmode
    \begin{enumerate}
        \item By \Cref{prop:isometry}, we know that
        \begin{equation}
            \forall x, y\in \mathcal{M},\ |P^v(x)-P^v(y)| = d\big(\Tilde{P}^v(x),\Tilde{P}^v(y)\big).
        \end{equation}
        Moreover, by \citep[Page 9]{ballmann2006manifolds}, $\Tilde{P}^v$ is 1-Lipschitz, so is $P^v$.
        \item The Busemann function is 1-Lipschitz, see \emph{e.g.} \citep[II. Proposition 8.22]{bridson2013metric}.
    \end{enumerate}
\end{proof}

\begin{lemma} \label[lemma]{lemma:inequality_distance}
    Let $d$ be a metric on $\mathcal{M}$. Then, for any $p\ge1$,
    \begin{equation}
        \forall x,y \in\mathcal{M},\ d(x,y)^p \le 2^{p-1} \big( d(x,o)^p + d(o,y)^p\big).
    \end{equation}
\end{lemma}

\begin{lemma}[Lemma 1 in \citep{rakotomamonjy2021statistical}] \label{lemma:fournier} %adapted from Theorem 2 in \citep{fournier2015rate}
    Let $p\ge 1$ and $\eta\in\mathcal{P}_p(\mathbb{R})$. Denote $\Tilde{M}_q(\eta)=\int |x|^q\ \mathrm{d}\eta(x)$ the moments of order $q$ and assume that $M_q(\eta)<\infty$ for some $q>p$. Then, there exists a constant $C_{p,q}$ depending only on $p,q$ such that for all $n\ge 1$,
    \begin{equation}
        \mathbb{E}[W_p^p(\hat{\eta}_n,\eta)] \le C_{p,q} \Tilde{M}_q(\eta)^{p/q}\left(n^{-1/2}\mathbb{1}_{\{q>2p\}} + n^{-1/2}\log(n) \mathbb{1}_{\{q=2p\}} + n^{-(q-p)/q} \mathbb{1}_{\{q\in(p,2p)\}}\right).
    \end{equation}
\end{lemma}

% We start by a lemma on the derivation of the geodesic distance on Riemannian manifolds. For references about this lemma, see \emph{e.g.} \citep[Appendix A]{chewi2020gradient} or \citep{goto2021approximated}.

\begin{lemma} \label[lemma]{lemma:derivative_geodesic_dist}
    Let $y\in \mathcal{M}$ and denote for all $x\in\mathcal{M}$, $f(x) = d(x,y)^2$. Then, $\mathrm{grad}_{\mathcal{M}} f(x) = -2\log_x(y)$.
\end{lemma}

For references about \Cref{lemma:derivative_geodesic_dist}, see \emph{e.g.} \citep[Appendix A]{chewi2020gradient} or \citep{goto2021approximated}.

\section{Proofs of \Cref{section:irsw}} \label{proofs:section_irsw}

\subsection{Proof of \Cref{prop:isometry}} \label{proof:prop_isometry}

\begin{proof}\textbf{of \Cref{prop:isometry}} %% Put in appendix (?)
    Let $x,y\in \mathcal{G}^v$. Then, there exists $s,t\in\mathbb{R}$ such that $x=\exp_o(sv)$ and $y=\exp_o(tv)$. By a simple calculation, we have on one hand that
    \begin{equation}
        \begin{aligned}
            \sign(\langle \log_o(x),v\rangle_o) &= \sign(\langle \log_o(\exp_o(sv)), v\rangle_o) \\
            &= \sign(s \|v\|_o^2) \\
            &= \sign(s),
        \end{aligned}
    \end{equation}
    using that $\log_o\circ \exp_o = \id$. And similarly, $\sign(\langle \log_o(y), v\rangle_o) = \sign(t)$.
    
    Then, by noting that $o=\exp_o(0)$, and recalling that $d(x,y) = d(\exp_o(tv), \exp_o(sv)) = |t-s|$,
    \begin{equation}
        \begin{aligned}
            |t^v(x)-t^v(y)| &= |\sign(\langle \log_o(x), v\rangle_o) d(x,o) - \sign(\langle \log_o(y), v\rangle_o d(y,o)| \\
            &= \big|\sign(s) d(\exp_o(sv),\exp_o(0)) - \sign(t) d(\exp_o(tv), \exp_o(0))\big| \\
            &= \big| \sign(s) |s| - \sign(t) |t|\big| \\
            &= |s-t| \\
            &= d(x,y).
        \end{aligned}    
    \end{equation}
\end{proof}

\subsection{Proof of \Cref{prop:charac_geod_proj}} \label{proof:prop_charac_geod_proj}

\begin{proof}\textbf{of \Cref{prop:charac_geod_proj}}
    We want to solve:
    \begin{equation}
        P^v(x) = \argmin_{t\in\mathbb{R}}\ d\big(\gamma(t), x\big)^2,
    \end{equation}
    where $\gamma(t) = \exp_o(tv)$.
    For $t\in\mathbb{R}$, let $g(t) = d\big(\gamma(t), x\big)^2 = f\big(\gamma(t)\big)$ where $f(x) = d(x,y)^2$ for $x,y\in\mathcal{M}$. Then, by \Cref{lemma:derivative_geodesic_dist}, we have for any $t\in\mathbb{R}$,
    \begin{equation}
        \begin{aligned}
            g'(t) = 0 &\iff \langle \gamma'(t), \mathrm{grad}_\mathcal{M} f\big(\gamma(t)\big)\rangle_{\gamma(t)} = 0 \\
            &\iff \langle \gamma'(t), -2\log_{\gamma(t)}(x)\rangle_{\gamma(t)} = 0.
        \end{aligned}
    \end{equation}
\end{proof}

\subsection{Proof of \Cref{prop:eq_wasserstein}} \label{proof:prop_eq_wasserstein}

\begin{proof}\textbf{of \Cref{prop:eq_wasserstein}}
    First, we note that $P^v = t^v\circ\Tilde{P}^v$. Then, by using \Cref{lemma:paty} which states that $\Pi(f_\#\mu, f_\#\nu)=\{(f\otimes f)_\#\gamma,\ \gamma\in\Pi(\mu,\nu)\}$ for any $f$ measurable, as well as that by \Cref{prop:isometry}, $|t^v(x)-t^v(y)| = d(x,y)$, we have:
    \begin{equation}
        \begin{aligned}
            W_p^p(P^v_\#\mu,P^v_\#\nu) &= \inf_{\gamma\in \Pi(P^v_\#\mu,P^v_\#\nu)}\ \int_{\mathbb{R}\times \mathbb{R}} |x-y|^p\ \mathrm{d}\gamma(x,y) \\
            &= \inf_{\gamma\in \Pi(\mu,\nu)}\ \int_{\mathbb{R}\times \mathbb{R}} |x-y|^p \ \mathrm{d}(P^v\otimes P^v)_\#\gamma(x,y) \\
            &= \inf_{\gamma\in \Pi(\mu,\nu)}\ \int_{\mathcal{M}\times \mathcal{M}}\ |P^v(x)-P^v(y)|^p\ \mathrm{d}\gamma(x,y) \\
            &= \inf_{\gamma\in \Pi(\mu,\nu)}\ \int_{\mathcal{M} \times \mathcal{M}} |t^v(\Tilde{P}^v(x))-t^v(\Tilde{P}^v(y))|^p\ \mathrm{d}\gamma(x,y) \\
            &= \inf_{\gamma\in \Pi(\mu,\nu)}\ \int_{\mathcal{M} \times \mathcal{M}} d\big(\Tilde{P}^v(x), \Tilde{P}^v(y)\big)^p\ \mathrm{d}\gamma(x,y) \\
            &= \inf_{\gamma\in \Pi(\mu,\nu)}\ \int_{\mathcal{M} \times \mathcal{M}} \ d(x, y)^p\ \mathrm{d}(\Tilde{P}^v\otimes \Tilde{P}^v)_\#\gamma(x,y) \\
            &= \inf_{\gamma\in\Pi(\Tilde{P}^v_\#\mu, \Tilde{P}^v_\#\nu)}\ \int_{\mathcal{G}^v\times \mathcal{G}^v} d(x,y)^p\ \mathrm{d}\gamma(x,y) \\
            &= W_p^p(\Tilde{P}^v_\#\mu, \Tilde{P}^v_\#\nu).
        \end{aligned}
    \end{equation}

    Now, let us show the results when using the Busemann projection. Let $v\in T_o\mathcal{M}$ such that $\|v\|_o = 1$, and recall that $\Tilde{B}^v(x) = \exp_o(-B^v(x) v)$. First, let us compute $t^v\circ \Tilde{B}^v$:
    \begin{equation}
        \begin{aligned}
            \forall x\in \mathcal{M},\ t^v\big(\Tilde{B}^v(x)\big) &= \sign(\langle \log_o(\Tilde{B}^v(x)), v \rangle_o)\ d(\Tilde{B}^v(x), o) \\
            &= \sign\big(\langle \log_o(\exp_o(-B^v(x)v)), v\rangle_o\big)\ d\big(\exp_o(-B^v(x)v), \exp_o(0)\big) \\
            &= \sign(-B^v(x) \|v\|_o^2)\ d(\exp_o(-B^v(x)v), \exp_o(0)) \\
            &= \sign(-B^v(x)) |-B^v(x)| \\
            &= - B^v(x).
        \end{aligned}
    \end{equation}
    Then, using the same computation as before, we get
    \begin{equation}
        W_p^p(B^v_\#\mu, B^v_\#\nu) = W_p^p(\Tilde{B}^v_\#\mu, \Tilde{B}^v_\#\nu).
    \end{equation}
\end{proof}

% \subsection{Proof of \Cref{prop:eq_wasserstein_busemann}} \label{proof:prop_eq_wasserstein_busemann}

% \begin{proof}\textbf{of \Cref{prop:eq_wasserstein_busemann}}
%     Let $v\in T_o\mathcal{M}$ such that $\|v\|_o = 1$, and recall that $\Tilde{B}^v(x) = \exp_o(-B^v(x) v)$. First, let us compute $t^v\circ \Tilde{B}^v$:
%     \begin{equation}
%         \begin{aligned}
%             \forall x\in \mathcal{M},\ t^v\big(\Tilde{B}^v(x)\big) &= \sign(\langle \log_o(\Tilde{B}^v(x)), v \rangle_o)\ d(\Tilde{B}^v(x), o) \\
%             &= \sign\big(\langle \log_o(\exp_o(-B^v(x)v)), v\rangle_o\big)\ d\big(\exp_o(-B^v(x)v), \exp_o(0)\big) \\
%             &= \sign(-B^v(x) \|v\|_o^2)\ d(\exp_o(-B^v(x)v), \exp_o(0)) \\
%             &= \sign(-B^v(x)) |-B^v(x)| \\
%             &= - B^v(x).
%         \end{aligned}
%     \end{equation}
%     Then, using the same computation as in the proof of \Cref{prop:eq_wasserstein}, we get
%     \begin{equation}
%         W_p^p(B^v_\#\mu, B^v_\#\nu) = W_p^p(\Tilde{B}^v_\#\mu, \Tilde{B}^v_\#\nu).
%     \end{equation}
% \end{proof}

\section{Proofs of \Cref{section:examples}} \label{proofs:section_examples}

\subsection{Proof of \Cref{prop:proj_coord_pullback}} \label{proof:prop_proj_coord_pullback}

\begin{proof}\textbf{of \Cref{prop:proj_coord_pullback}} \leavevmode

    \noindent \textit{Geodesic projection.} Let $x\in \mathcal{M}$. Denote $f:\mathbb{R}\to\mathbb{R}$ such that 
    \begin{equation}
        \begin{aligned}
            f(t) &= d_\mathcal{M}\big(\gamma(t), x\big)^2 \\
            &= d_\mathcal{M}\big(\phi^{-1}(\phi(o) + t \phi_{*,o}(v)), x\big)^2 \\
            &= \left\|\phi\big(\phi^{-1}(\phi(o) + t \phi_{*,o}(v))\big) - \phi(x)\right\|^2 \\
            &= t^2 \|\phi_{*,o}(v)\|^2 -2t\langle \phi(x)-\phi(o), \phi_{*,o}(v)\rangle + \|\phi(o)-\phi(x)\|^2 \\
            &= t^2 -2t\langle \phi(x)-\phi(o), \phi_{*,o}(v)\rangle + \|\phi(o)-\phi(x)\|^2,
        \end{aligned}
    \end{equation}
    using in the last line that $\|\phi_{*,o}(v)\|^2 = 1$ since $v\in S_o$.
    Then, 
    \begin{equation}
        f'(t) = 0 \iff t = \langle \phi(x)-\phi(o), \phi_{*,o}(v)\rangle.
    \end{equation}
    Therefore,
    \begin{equation}
        P^v(x) = \argmin_{t\in\mathbb{R}}\ f(t) = \langle \phi(x)-\phi(o), \phi_{*,o}(v)\rangle.
    \end{equation}

    \noindent \textit{Busemann function.} First, following \citep{bridson2013metric}, we have for all $x\in \mathcal{M}$,
    \begin{equation}
        B^v(x) = \lim_{t\to\infty}\ \big( d_\mathcal{M}(\gamma_v(t),x) - t\big) = \lim_{t\to\infty}\ \frac{d_\mathcal{M}(\gamma_v(t),x)^2 - t^2}{2t},
    \end{equation}
    denoting $\gamma_v:t\mapsto \phi^{-1}\big(\phi(o) + t \phi_{*,o}(v)\big)$ the geodesic line associated to $\mathcal{G}^v$. Then, we get
    \begin{equation}
        \begin{aligned}
            \frac{d_\mathcal{M}(\gamma_v(t),x)^2 - t^2}{2t} &= \frac{1}{2t} \big(\|\phi(\gamma_v(t))^2 - \phi(x)\|^2 - t^2\big) \\
            &= \frac{1}{2t} \big(\|\phi(o) + t\phi_{*,o}(v) - \phi(x)\|^2 - t^2\big) \\
            &= \frac{1}{2t} \big(t^2\|\phi_{*,o}(v)\|^2 -2t\langle \phi_{*,o}(v), \phi(x)-\phi(o)\rangle + \|\phi(x)-\phi(o)\|^2 - t^2\big) \\
            &= -\langle \phi_{*,o}(v), \phi(x)-\phi(o)\rangle + \frac{1}{2t}\|\phi(x)-\phi(o)\|^2,
        \end{aligned}
    \end{equation}
    using that $\|v\|_o = \|\phi_{*,o}(v)\| = 1$. Then, by passing to the limit $t\to\infty$, we find
    \begin{equation}
        B^v(x) = -\langle \phi_{*,o}(v), \phi(x)-\phi(o)\rangle.
    \end{equation}
\end{proof}

\subsection{Proof of \Cref{prop:hsw_coord_projs}}

We start by giving the proof of the coordinate geodesic projection which we recall in \Cref{prop:hsw_coord_geod_proj}.

\begin{proposition}[Coordinate of the geodesic projection on Hyperbolic space] \label[proposition]{prop:hsw_coord_geod_proj} \leavevmode
    \begin{enumerate}
        \item Let $\mathcal{G}^v = \mathrm{span}(x^0, v)\cap \mathbb{L}^d_K$ where $v\in T_{x^0}\mathbb{L}^d_K\cap S^d$. Then, the coordinate $P^v(x)$ of the geodesic projection on $\mathcal{G}^v$ of $x\in \mathbb{L}^d_K$ is
        \begin{equation}
            P^v(x) = \frac{1}{\sqrt{-K}}\arctanh\left(-\frac{1}{\sqrt{-K}}\frac{\langle x, v\rangle_\mathbb{L}}{\langle x,x^0\rangle_\mathbb{L}}\right).
        \end{equation}
        \item Let $\Tilde{v}\in S^{d-1}$ be an ideal point. Then, the coordinate $P^{\Tilde{v}}(x)$ of the geodesic projection on the geodesic characterized by $\Tilde{v}$ of $x\in \mathbb{B}^d_K$ is
        \begin{equation}
            P^{\Tilde{v}}(x) = \frac{2}{\sqrt{-K}} \arctanh\big(\sqrt{-K} s(x)\big),
        \end{equation}
        where $s$ is defined as 
        \begin{equation}
            s(x) = \left\{\begin{array}{ll} \frac{1-K\|x\|_2^2 - \sqrt{(1-K\|x\|_2^2)^2 + 4K \langle x, \Tilde{v}\rangle^2}}{-2K \langle x, \Tilde{v}\rangle} & \mbox{ if } \langle x,\Tilde{v}\rangle \neq 0 \\
            0 & \mbox{ if } \langle x,\Tilde{v}\rangle = 0.
            \end{array}\right.
        \end{equation}
    \end{enumerate}
\end{proposition}

First, we will compute in \Cref{prop:hsw_geodesic_proj} the geodesic projections.

\begin{proposition}[Geodesic projection] \label[proposition]{prop:hsw_geodesic_proj} \leavevmode
    \begin{enumerate}
        % \item Let $\mathcal{G}^v=\mathrm{span}(x^0,v)\cap \mathbb{L}^d$ where $v\in T_{x^0}\mathbb{L}^d\cap S^d$. Then, the geodesic projection $\Tilde{P}^v$ on $\mathcal{G}^v$ of $x\in\mathbb{L}^d$ is
        % \begin{equation}
        %     \begin{aligned}
        %         \Tilde{P}^v(x) &= \frac{1}{\sqrt{\langle x,x^0\rangle_\mathbb{L}^2-\langle x, v\rangle_\mathbb{L}^2}} \big(-\langle x,x^0\rangle_\mathbb{L}x^0 + \langle x,v\rangle_\mathbb{L} v\big) \\
        %         &= \frac{P^{\mathrm{span}(x^0, v)}(x)}{\sqrt{-\langle P^{\mathrm{span}(x^0, v)}(x), P^{\mathrm{span}(x^0, v)}(x)\rangle_\mathbb{L}}},
        %     \end{aligned}
        % \end{equation}
        % where $P^{\mathrm{span}(x^0,v^0)}$ is the linear orthogonal projection on the subspace $\mathrm{span}(x^0,v)$.
        \item Let $\mathcal{G}^v=\mathrm{span}(x^0,v)\cap \mathbb{L}^d_K$ where $v\in T_{x^0}\mathbb{L}^d_K\cap S^d$. Then, the geodesic projection $\Tilde{P}^v$ on $\mathcal{G}^v$ of $x\in\mathbb{L}^d_K$ is
        \begin{equation}
            \begin{aligned}
                \Tilde{P}^v(x) &= \frac{1}{\sqrt{-K\langle x,x^0\rangle_\mathbb{L}^2-\langle x, v\rangle_\mathbb{L}^2}} \big(-\sqrt{-K}\langle x,x^0\rangle_\mathbb{L}x^0 + \langle x,v\rangle_\mathbb{L} v\big).
                % &= \frac{P^{\mathrm{span}(x^0, v)}(x)}{\sqrt{-\langle P^{\mathrm{span}(x^0, v)}(x), P^{\mathrm{span}(x^0, v)}(x)\rangle_\mathbb{L}}},
            \end{aligned}
        \end{equation}
        
        % where $P^{\mathrm{span}(x^0,v^0)}$ is the linear orthogonal projection on the subspace $\mathrm{span}(x^0,v)$.
        % \item Let $\Tilde{v}\in S^{d-1}$ be an in ideal point. Then, the geodesic projection $\Tilde{P}^{\Tilde{v}}$ on the geodesic characterized by $\Tilde{v}$ of $x\in \mathbb{B}^d$ is
        % \begin{equation}
        %     \Tilde{P}^{\Tilde{v}}(x) = s(x) \Tilde{v},
        % \end{equation}
        % where
        % \begin{equation}
        %     s(x) = \left\{\begin{array}{ll} \frac{1+\|x\|_2^2 - \sqrt{(1+\|x\|_2^2)^2 - 4 \langle x, \Tilde{v}\rangle^2}}{2 \langle x, \Tilde{v}\rangle} & \mbox{ if } \langle x,\Tilde{v}\rangle \neq 0 \\
        %     0 & \mbox{ if } \langle x,\Tilde{v}\rangle = 0.
        %     \end{array}\right.
        % \end{equation}
        \item Let $\Tilde{v}\in S^{d-1}$ be an in ideal point. Then, the geodesic projection $\Tilde{P}^{\Tilde{v}}$ on the geodesic characterized by $\Tilde{v}$ of $x\in \mathbb{B}^d_K$ is
        \begin{equation}
            \Tilde{P}^{\Tilde{v}}(x) = s(x) \Tilde{v},
        \end{equation}
        where
        \begin{equation}
            s(x) = \left\{\begin{array}{ll} \frac{1-K\|x\|_2^2 - \sqrt{(1-K\|x\|_2^2)^2 + 4 K \langle x, \Tilde{v}\rangle^2}}{-2 K\langle x, \Tilde{v}\rangle} & \mbox{ if } \langle x,\Tilde{v}\rangle \neq 0 \\
            0 & \mbox{ if } \langle x,\Tilde{v}\rangle = 0.
            \end{array}\right.
        \end{equation}
    \end{enumerate}
\end{proposition}

\begin{proof}\textbf{of \Cref{prop:hsw_geodesic_proj}} \label{proof:geodesic_proj} \leavevmode
    \begin{enumerate}

        \item \textit{Lorentz model.} Any point $y$ on the geodesic obtained by the intersection between $E=\mathrm{span}(x^0,v)$ and $\mathbb{L}^d_K$ can be written as
        \begin{equation}
            y = \cosh(\sqrt{-K}t) x^0 + \sinh(\sqrt{-K}t) \frac{v}{\sqrt{-K}},
        \end{equation}
        where $t\in\mathbb{R}$. Moreover, as $\arccosh$ is an increasing function, we have
        \begin{equation}
            \begin{aligned}
                \Tilde{P}^v(x) &= \argmin_{y\in E\cap\mathbb{L}^d_K}\ d_\mathbb{L}(x,y) \\
                &= \argmin_{y\in E\cap \mathbb{L}^d_K}\ \arccosh(K\langle x,y\rangle_\mathbb{L}) \\
                &= \argmin_{y\in E\cap \mathbb{L}^d_K}\ K\langle x,y\rangle_\mathbb{L}.
            \end{aligned}
        \end{equation}
        This problem is equivalent with solving
        \begin{equation}
            \argmin_{t\in\mathbb{R}}\ K\cosh(\sqrt{-K}t)\langle x,x^0\rangle_\mathbb{L} + K\frac{\sinh(\sqrt{-K}t)}{\sqrt{-K}}\langle x,v\rangle_\mathbb{L}.
        \end{equation}
        Let $g(t)=\cosh(\sqrt{-K}t)\langle x,x^0\rangle_\mathbb{L} + \frac{\sinh(\sqrt{-K}t)}{\sqrt{-K}}\langle x,v\rangle_\mathbb{L}$, then
        \begin{equation} \label{eq:hsw_coord_opt}
            g'(t) = 0 \iff \tanh(\sqrt{-K}t) = - \frac{1}{\sqrt{-K}} \frac{\langle x,v\rangle_\mathbb{L}}{\langle x,x^0\rangle_\mathbb{L}}.
        \end{equation}
        Finally, using that $1-\tanh^2(t)=\frac{1}{\cosh^2(t)}$ and $\cosh^2(t)-\sinh^2(t)=1$, and observing that necessarily, $\langle x,x^0\rangle_\mathbb{L} \le 0$, we obtain
        \begin{equation}
            \cosh(\sqrt{-K}t) = \frac{1}{\sqrt{1-\left(-\frac{1}{\sqrt{-K}}\frac{\langle x,v\rangle_\mathbb{L}}{\langle x,x^0\rangle_\mathbb{L}}\right)^2}} = \frac{- \sqrt{-K} \langle x,x^0\rangle_\mathbb{L}}{\sqrt{-K\langle x,x^0\rangle_\mathbb{L}^2 - \langle x,v\rangle_\mathbb{L}^2}},
        \end{equation}
        and
        \begin{equation}
            \sinh(\sqrt{-K}t) = \frac{\langle x, v\rangle_\mathbb{L}}{\sqrt{-K\langle x,x^0\rangle_\mathbb{L}^2 - \langle x,v\rangle_\mathbb{L}^2}}.
        \end{equation}

        \item \textit{Poincaré ball.} A geodesic passing through the origin on the Poincaré ball is of the form $\gamma(t) = tp$ for an ideal point $p\in S^{d-1}$ and $t\in ]-\frac{1}{\sqrt{-K}},\frac{1}{\sqrt{-K}}[$. Using that $\arccosh$ is an increasing function, we find
        \begin{equation}
            \begin{aligned}
                \Tilde{P}^p(x) &= \argmin_{y \in \mathrm{span}(\gamma)}\ d_\mathbb{B}(x,y) \\
                &= \argmin_{tp}\ \frac{1}{\sqrt{-K}}\arccosh\left(1 - 2K \frac{\|x-\gamma(t)\|_2^2}{(1+K\|x\|_2^2)(1+K\|\gamma(t)\|_2^2)}\right) \\
                &= \argmin_{tp}\ \log\big(\|x-\gamma(t)\|_2^2\big) - \log\big(1+ K\|x\|_2^2\big) - \log\big(1+K\|\gamma(t)\|_2^2\big) \\
                &= \argmin_{tp}\ \log\big(\|x-tp\|_2^2\big) - \log\big(1+Kt^2\big).
            \end{aligned}
        \end{equation}
        Let $g(t) = \log\big(\|x-tp\|_2^2\big) - \log\big(1+Kt^2\big)$. Then,
        % \begin{equation}
        %     \begin{aligned}
        %         g'(t) = 0 \iff t^2 - \frac{1+\|x\|_2^2}{\langle x,p\rangle} t + 1 = 0.
        %     \end{aligned}
        % \end{equation}
        \begin{equation}
            g'(t) = 0 \iff  \left\{\begin{array}{ll}
          t^2 + \frac{1-K\|x\|_2^2}{K\langle x,p\rangle} t - \frac{1}{K} = 0 & \mbox{ if } \langle p,x\rangle \neq 0, \\
          t = 0 & \mbox{ if } \langle p,x\rangle = 0.
          \end{array}\right.
        \end{equation}
        Finally, if $\langle x, p\rangle \neq 0$, the solution is
        \begin{equation}
            t = -\frac{1-K\|x\|_2^2}{2K\langle x,p\rangle} \pm \sqrt{\left(\frac{1-K\|x\|_2^2}{2K\langle x,p\rangle}\right)^2 + \frac{1}{K}}.
        \end{equation}

        Now, let us suppose that $\langle x, p\rangle > 0$. Then, 
        \begin{equation}
            \begin{aligned}
                \frac{1-K\|x\|_2^2}{-2K\langle x,p\rangle} + \sqrt{\left(\frac{1-K\|x\|_2^2}{2K\langle x,p\rangle}\right)^2 + \frac{1}{K}} &\ge \frac{1-K\|x\|_2^2}{2K\langle x,p\rangle} \\
                & \ge \frac{1}{\sqrt{-K}},
            \end{aligned}
        \end{equation}
        because $\|\sqrt{-K}x-p\|_2^2 \ge 0$ implies that $\frac{1-K\|x\|_2^2}{2\sqrt{-K}\langle x,p\rangle} \ge 1$ which implies that $\frac{1-K\|x\|_2^2}{-2K\langle x,p\rangle} \ge \frac{1}{\sqrt{-K}}$, and therefore the solution is
        \begin{equation}
            t = -\frac{1-K\|x\|_2^2}{2K\langle x,p\rangle} - \sqrt{\left(\frac{1-K\|x\|_2^2}{2K\langle x,p\rangle}\right)^2 + \frac{1}{K}}.
        \end{equation}

        Similarly, if $\langle x, p\rangle < 0$, then
        \begin{equation}
            \begin{aligned}
                \frac{1-K\|x\|_2^2}{-2K\langle x,p\rangle} - \sqrt{\left(\frac{1-K\|x\|_2^2}{2K\langle x,p\rangle}\right)^2 + \frac{1}{K}} &\le  \frac{1-K\|x\|_2^2}{-2K\langle x,p\rangle} \\
                &\le -\frac{1}{\sqrt{-K}},
            \end{aligned}
        \end{equation}
        because $\|\sqrt{-K}x+p\|_2^2 \ge 0$ implies $\frac{1-K\|x\|_2^2}{2\sqrt{-K}\langle x,p\rangle} \le -1$, which implies that $\frac{1-K\|x\|_2^2}{-2K\langle x,p\rangle} \le -\frac{1}{\sqrt{-K}}$ and the solution is
        \begin{equation}
            \frac{1-K\|x\|_2^2}{-2K\langle x,p\rangle} + \sqrt{\left(\frac{1-K\|x\|_2^2}{2K\langle x,p\rangle}\right)^2 + \frac{1}{K}}.
        \end{equation}
        Thus,
        \begin{equation}
            \begin{aligned}
                s(x) &= \begin{cases}
                    \frac{1-K\|x\|_2^2}{-2K\langle x, p\rangle} - \sqrt{\left(\frac{1-K\|x\|_2^2}{2K\langle x,p\rangle}\right)^2 + \frac{1}{K}} \quad \text{if } \langle x, p\rangle > 0 \\
                    \frac{1-K\|x\|_2^2}{-2K\langle x, p\rangle} + \sqrt{\left(\frac{1-K\|x\|_2^2}{2K\langle x,p\rangle}\right)^2 +\frac{1}{K}} \quad \text{if } \langle x, p\rangle < 0.
                \end{cases} \\
                &= \frac{1-K\|x\|_2^2}{-2K\langle x, p\rangle} - \mathrm{sign}(\langle x,p\rangle) \sqrt{\left(\frac{1-K\|x\|_2^2}{2K\langle x,p\rangle}\right)^2 + \frac{1}{K}} \\
                &= \frac{1-K\|x\|_2^2}{-2K\langle x, p\rangle}  - \frac{\mathrm{sign}(\langle x,p\rangle)}{-2K\mathrm{sign}(\langle x,p\rangle)\langle x, p\rangle} \sqrt{(1-K\|x\|_2^2)^2 + 4K\langle x,p\rangle^2} \\
                &= \frac{1-K\|x\|_2^2 - \sqrt{(1-K\|x\|_2^2)^2 + 4K \langle x, p\rangle^2}}{-2K \langle x, p\rangle}.
            \end{aligned}
        \end{equation}
    \end{enumerate}
\end{proof}

% We observe that the projection on the geodesic in the Lorentz model can be done by first projecting on the subspace $\mathrm{span}(x^0,v)$ and then by projecting on the hyperboloid by normalizing. This is %strictly 
% analogous to the spherical case studied in \citep{bonet2022spherical}, %where the geodesic projection on a geodesic, which is a great circle, can be obtained by first projecting on the plane intersecting the sphere and then by normalizing. 
% % The differences in this case are that, 
% the differences being that, in the hyperbolic case, we are on the Minkowski space and that the geodesics are not periodic, contrary to the sphere. Moreover, we only integrate \emph{w.r.t.} geodesics passing through the origin when \citet{bonet2022spherical} integrate over all possible geodesics, as the sphere does not have a natural origin.

\begin{proof}\textbf{of \Cref{prop:hsw_coord_geod_proj}} \label{proof:prop_hsw_coord_geod_proj} \leavevmode
    \begin{enumerate}
        % \item \textbf{Lorentz model.} The coordinate on the geodesic can be obtained as 
        % \begin{equation}
        %     P^v(x) = \argmin_{t\in\mathbb{R}}\ d_\mathbb{L}\big(\exp_{x^0}(tv), x\big).
        % \end{equation}
        % Hence, by using \eqref{eq:hsw_coord_opt}, we obtain that the optimal $t$ satisfies
        % \begin{equation}
        %     \tanh(t) = - \frac{\langle x, v\rangle_\mathbb{L}}{\langle x,x^0\rangle_\mathbb{L}} \iff t = \arctanh\left(-\frac{\langle x, v\rangle_\mathbb{L}}{\langle x, x^0\rangle_\mathbb{L}}\right).
        % \end{equation}

        \item \textit{Lorentz model.} The coordinate on the geodesic can be obtained as 
        \begin{equation}
            P^v(x) = \argmin_{t\in\mathbb{R}}\ d_\mathbb{L}\big(\exp_{x^0}(tv), x\big).
        \end{equation}
        Hence, by using \eqref{eq:hsw_coord_opt}, we obtain that the optimal $t$ satisfies
        \begin{equation}
            \tanh(\sqrt{-K}t) = - \frac{1}{\sqrt{-K}} \frac{\langle x, v\rangle_\mathbb{L}}{\langle x,x^0\rangle_\mathbb{L}} \iff t = \frac{1}{\sqrt{-K}}\arctanh\left(-\frac{1}{\sqrt{-K}}\frac{\langle x, v\rangle_\mathbb{L}}{\langle x, x^0\rangle_\mathbb{L}}\right).
        \end{equation}        
        % \item \textbf{Poincaré ball.} As a geodesic is of the form $\gamma(t)=\tanh\left(\frac{t}{2}\right) p$ for all $t\in\mathbb{R}$, we deduce from \Cref{prop:hsw_geodesic_proj} that
        % \begin{equation}
        %     s(x) = \tanh\left(\frac{t}{2}\right) \iff t = 2 \arctanh\big(s(x)\big).
        % \end{equation}
        \item \textit{Poincaré ball.} As a geodesic is of the form $\gamma(t)=\tanh\left(\frac{\sqrt{-K}t}{2}\right) \frac{p}{\sqrt{-K}}$ for all $t\in\mathbb{R}$, we deduce from \Cref{prop:hsw_geodesic_proj} that
        \begin{equation}
            s(x) = \frac{1}{\sqrt{-K}}\tanh\left(\frac{\sqrt{-K}t}{2}\right) \iff t = \frac{2}{\sqrt{-K}} \arctanh\big(\sqrt{-K} s(x)\big).
        \end{equation}
    \end{enumerate}
\end{proof}

% \subsection{Proof of \Cref{prop:busemann_closed_forms}} \label{proof:prop_busemann_closed_forms}

We now derive the closed forms of the horospherical projections which we recall in \Cref{prop:busemann_closed_forms}.

\begin{proposition}[Busemann function on Hyperbolic space] \label[proposition]{prop:busemann_closed_forms} \leavevmode
    \begin{enumerate}
        \item On $\mathbb{L}^d_K$, for any direction $v\in T_{x^0}\mathbb{L}^d_K\cap S^d$, 
        \begin{equation}
            \forall x\in\mathbb{L}^d_K,\ B^v(x) %= \frac{1}{\sqrt{-K}}\log\left(K\left(\langle x,x^0\rangle_\mathbb{L} + \frac{\langle x,v\rangle_\mathbb{L}}{\sqrt{-K}}\right)\right) 
            = \frac{1}{\sqrt{-K}} \log\left(-\sqrt{-K}\left\langle x, \sqrt{-K} x^0 + v\right\rangle_\mathbb{L}\right).
        \end{equation}
        \item On $\mathbb{B}^d_K$, for any ideal point $\Tilde{v}\in S^{d-1}$,
        \begin{equation}
            \forall x\in \mathbb{B}^d_K,\ B^{\Tilde{v}}(x) = \frac{1}{\sqrt{-K}}\log\left(\frac{\|\Tilde{v}-\sqrt{-K}x\|_2^2}{1+K\|x\|_2^2}\right).
        \end{equation}
    \end{enumerate}
\end{proposition}

\begin{proof}\textbf{of \Cref{prop:busemann_closed_forms}} \leavevmode
    \begin{enumerate}

        \item \textit{Lorentz model.}
            The geodesic in direction $v$ can be characterized by
            \begin{equation}
                \forall t\in \mathbb{R},\ \gamma_v(t) = \cosh(\sqrt{-K}t)x^0 + \sinh(\sqrt{-K}t) \frac{v}{\sqrt{-K}}.
            \end{equation}
            Hence, we have for all $x\in \mathbb{L}^d_K$,
            \begin{equation}
                \begin{aligned}
                    &d_{\mathbb{L}}(\gamma_v(t), x)  \\
                    &= \frac{1}{\sqrt{-K}} \arccosh(K\langle \gamma_v(t), x\rangle_\mathbb{L}) \\
                    &= \frac{1}{\sqrt{-K}} \arccosh(K\cosh(\sqrt{-K}t)\langle x,x^0\rangle_\mathbb{L} + \frac{K}{\sqrt{-K}} \sinh(\sqrt{-K}t) \langle x, v\rangle_\mathbb{L} ) \\
                    &= \frac{1}{\sqrt{-K}} \arccosh\left(K \frac{e^{\sqrt{-K}t} + e^{-\sqrt{-K}t}}{2} \langle x,x^0\rangle_\mathbb{L} + \frac{K}{\sqrt{-K}} \frac{e^{\sqrt{-K}t} - e^{-\sqrt{-K}t}}{2}\langle x,v\rangle_\mathbb{L}\right) \\
                    &= \frac{1}{\sqrt{-K}} \arccosh\left(K\frac{e^{\sqrt{-K}t}}{2}\big((1+e^{-2\sqrt{-K}t})\langle x, x^0\rangle_\mathbb{L} + \frac{1}{\sqrt{-K}} (1-e^{-2\sqrt{-K}t})\langle x,v\rangle_\mathbb{L}\big)\right) \\
                    &= \frac{1}{\sqrt{-K}} \arccosh\big(x(t)\big).
                \end{aligned}
            \end{equation}
            Then, on one hand, we have $x(t) \underset{t\to\infty}{\to} \pm \infty$, and using that $\arccosh(x)=\log\big(x+\sqrt{x^2-1}\big)$, we have
            \begin{equation}
                \begin{aligned}
                    d_{\mathbb{L}}(\gamma_v(t),x) - t &= \frac{1}{\sqrt{-K}} \big( \log\big(x(t) + \sqrt{x(t)^2 - 1}\big) - \sqrt{-K} t\big) \\
                    &= \frac{1}{\sqrt{-K}}\log\left(\big(x(t) + \sqrt{x(t)^2-1}\big) e^{-\sqrt{-K}t}\right) \\
                    &= \frac{1}{\sqrt{-K}}\log\left(e^{-\sqrt{-K}t} x(t) + e^{-\sqrt{-K}t}x(t) \sqrt{1-\frac{1}{x(t)^2}}\right) \\
                    &\underset{\infty}{=} \frac{1}{\sqrt{-K}}\log\left(e^{-\sqrt{-K}t}x(t) + e^{-\sqrt{-K}t}x(t) \left(1-\frac{1}{2x(t)^2} + o\left(\frac{1}{x(t)^2}\right)\right)\right).
                \end{aligned}
            \end{equation}
            Moreover,
            \begin{equation}
                \begin{aligned}
                    e^{-\sqrt{-K}t}x(t) &= \frac{K}{2} (1+e^{-2\sqrt{-K}t})\langle x,x^0\rangle_\mathbb{L} + \frac{K}{2\sqrt{-K}}(1-e^{-2\sqrt{-K}t})\langle x,v\rangle_\mathbb{L} \\ &\underset{t\to\infty}{\to} \frac{K}{2} \left(\langle x,  x^0 \rangle_\mathbb{L} + \frac{\langle x, v\rangle_\mathbb{L}}{\sqrt{-K}}\right).
                \end{aligned}
            \end{equation}
            Hence,
            \begin{equation}
                B^v(x) = \frac{1}{\sqrt{-K}} \log\left(K\left(\langle x, x^0\rangle_\mathbb{L} + \frac{\langle x, v\rangle_\mathbb{L}}{\sqrt{-K}}\right) \right).
            \end{equation}

        \item \textit{Poincaré ball.}
        
        % Note that this proof can be found \emph{e.g.} in the Appendix of \citep{ghadimi2021hyperbolic}. We report it for the sake of completeness.
        
            Let $p\in S^{d-1}$, then the geodesic from $0$ to $p$ is of the form $\gamma_p(t) = \exp_0(tp) = \tanh(\frac{\sqrt{-K}t}{2})\frac{p}{\sqrt{-K}}$.
            Moreover, recall that $\arccosh(x) = \log(x+\sqrt{x^2 -1})$ and 
            \begin{equation}
                \begin{aligned}
                    d_{\mathbb{B}}(\gamma_p(t),x) &= \frac{1}{\sqrt{-K}} \arccosh\left(1-2K\frac{\|\tanh(\frac{\sqrt{-K}t}{2})\frac{p}{\sqrt{-K}}-x\|_2^2}{(1-\tanh^2(\frac{\sqrt{-K}t}{2}))(1+K\|x\|_2^2)}\right) \\
                    &= \frac{1}{\sqrt{-K}}\arccosh(1+x(t)),
                \end{aligned}
            \end{equation}
            where
            \begin{equation}
                x(t) = -2K\frac{\|\tanh(\frac{\sqrt{-K}t}{2})\frac{p}{\sqrt{-K}}-x\|_2^2}{(1-\tanh^2(\frac{\sqrt{-K}t}{2}))(1+K\|x\|_2^2)}.
            \end{equation}
            Now, on one hand, we have
            \begin{equation}
                \begin{aligned}
                    B^p(x) &= \lim_{t\to\infty}\ (d_\mathbb{B}(\gamma_p(t),x)-t) \\
                    &= \lim_{t\to \infty}\ \frac{1}{\sqrt{-K}}\left(\log\big(1+x(t)+\sqrt{x(t)^2+2x(t)}\big)-\sqrt{-K}t\right) \\
                    &= \lim_{t\to \infty}\ \frac{1}{\sqrt{-K}}\log\big(e^{-\sqrt{-K}t}(1+x(t)+\sqrt{x(t)^2 + 2x(t)})\big).
                \end{aligned}
            \end{equation}
            On the other hand, using that $\tanh(\frac{t}{2}) = \frac{e^t-1}{e^t+1}$,
            \begin{equation}
                \begin{aligned}
                    e^{-\sqrt{-K}t}x(t) &= -2K e^{-\sqrt{-K}t} \frac{\|\frac{e^{\sqrt{-K}t}-1}{e^{\sqrt{-K}t}+1}\frac{p}{\sqrt{-K}}-x\|_2^2}{(1-(\frac{e^{\sqrt{-K}t}-1}{e^{\sqrt{-K}t}+1})^2)(1+K\|x\|_2^2)} \\
                    &= 2e^{-\sqrt{-K}t} \frac{\|e^{\sqrt{-K}t} p - p - \sqrt{-K} e^{\sqrt{-K}t} x - \sqrt{-K} x\|_2^2}{4e^{\sqrt{-K}t} (1+K\|x\|_2^2)} \\
                    &= \frac12 \frac{\|p-e^{-\sqrt{-K}t}p-\sqrt{-K}x-\sqrt{-K}e^{-\sqrt{-K}t}x\|_2^2}{1+K\|x\|_2^2} \\
                    &\underset{t\to\infty}{\to}\frac12 \frac{\|p-\sqrt{-K}x\|_2^2}{1+K\|x\|_2^2}.
                \end{aligned}
            \end{equation}
            Hence,
            \begin{equation}
                \begin{aligned}
                    B^p(x) &= \lim_{t\to\infty}\ \frac{1}{\sqrt{-K}}\log\left( e^{-\sqrt{-K}t} + e^{-\sqrt{-K}t}x(t) + e^{-\sqrt{-K}t}x(t)\sqrt{1+\frac{2}{x(t)}}\right) \\ &= \frac{1}{\sqrt{-K}}\log\left(\frac{\|p-\sqrt{-K}x\|_2^2}{1+K\|x\|_2^2}\right),
                \end{aligned}
            \end{equation}
            using that $\sqrt{1+\frac{2}{x(t)}} = 1 + \frac{1}{x(t)} + o(\frac{1}{x(t)})$ and $\frac{1}{x(t)}\to_{t\to\infty} 0$.
    \end{enumerate}
\end{proof}

\subsection{Proof of \Cref{prop:isometry_chsw}} \label{proof:prof_isometry_chsw}

First, we recall and show two lemmas.

\begin{lemma}[Proposition 5.6.c in \citep{lee2006riemannian}] \label[lemma]{lemma:lee_geodesics}
    Suppose $\phi:(\mathcal{M}, g)\to(\Tilde{\mathcal{M}}, \Tilde{g})$ is an isometry. Then, $\phi$ takes geodesics to geodesics, \emph{i.e.} if $\gamma$ is the geodesic in $\mathcal{M}$ with $\gamma(0)=p$ and $\gamma'(0)=v$, then $\phi\circ\gamma$ is the geodesic in $\Tilde{\mathcal{M}}$ with $\phi(\gamma(0)) = \phi(p)$ and $(\phi\circ\gamma)'(0) = \phi_{*,p}(v)$.
\end{lemma}

\begin{lemma} \label[lemma]{lemma:commute_projs}
    Let $\phi:(\mathcal{M},g)\to(\Tilde{\mathcal{M}}, \Tilde{g})$ an isometry and $v\in T_o\mathcal{M}$ such that $\|v\|_o = 1$. Then for all $x\in \mathcal{M}$,
    \begin{align}
        B^v(x) = B^{\phi_{*,o}(v)}\big(\phi(x)\big) \label{eq:isometry1}, \\
        P^v(x) = P^{\phi_{*,o}(v)}\big(\phi(x)\big) \label{eq:isometry2}.
    \end{align}
\end{lemma}

\begin{proof}\textbf{of \Cref{lemma:commute_projs}} \leavevmode
    Let $v\in T_o\mathcal{M}$ such that $\|v\|_o = 1$, $x\in \mathcal{M}$. By \Cref{lemma:lee_geodesics}, we have $\phi\big(\exp_o(tv)\big) = \exp_{\phi(o)}\big(t\phi_{*,o}(v)\big)$.

    \noindent \textit{Proof of \Cref{eq:isometry1}.} Let us show that $B^v(x) = B^{\phi_{*,o}(v)}\big(\phi(x)\big)$. By definition of the Busemann function, we have
    \begin{equation} \label{eq:busemann_equality_isometry}
        \begin{aligned}
            B^v(x) &= \lim_{t\to\infty}\ d_\mathcal{M}\big(x, \exp_o(tv)\big) - t \\
            &= \lim_{t\to \infty}\ d_{\Tilde{\mathcal{M}}}\big(\phi(x), \phi(\exp_o(tv))\big) - t \quad \text{since $\phi$ is an isometry} \\
            &= \lim_{t\to\infty}\ d_{\Tilde{\mathcal{M}}}\big(\phi(x), \exp_{\phi(o)}(t\phi_{*,o}(v))\big) - t \\
            &= B^{\phi_{*,o}(v)}\big(\phi(x)\big).
        \end{aligned}
    \end{equation}

    % Now, by recalling that by \eqref{eq:busemann_proj}, $\Tilde{B}^{v}(x) = \exp_o\big(-B^v(x) v\big)$ and applying again \Cref{lemma:lee_geodesics}, we obtain
    % \begin{equation}
    %     \begin{aligned}
    %         \phi\big(\Tilde{B}^v(x)\big) &= \phi\big(\exp_o(-B^v(x)v)\big) \\
    %         &= \phi\big(\exp_o(-B^{\phi_{*,o}(v)}(\phi(x)) v)\big) \quad \text{by \eqref{eq:busemann_equality_isometry}} \\
    %         &= \exp_{\phi(o)}\big( - B^{\phi_{*,o}(v)}(\phi(x)) \phi_{*,o}(v)\big) \quad \text{by \Cref{lemma:lee_geodesics}} \\
    %         &= \Tilde{B}^{\phi_{*,o}(v)}\big(\phi(x)\big).
    %     \end{aligned}
    % \end{equation}

    \noindent \textit{Proof of \Cref{eq:isometry2}.} Let us now show that $P^v(x) = P^{\phi_{*,o}(v)}\big(\phi(x)\big)$. Then,
    \begin{equation}
        \begin{aligned}
            P^v(x) &= \argmin_{t\in\mathbb{R}}\ d_\mathcal{M}\big(x,\exp_o(tv)\big) \\
            &= \argmin_{t\in\mathbb{R}}\ d_{\Tilde{\mathcal{M}}}\big(\phi(x), \phi(\exp_o(tv))\big) \quad \text{since $\phi$ is an isometry} \\
            &= \argmin_{t\in\mathbb{R}}\ d_{\Tilde{\mathcal{M}}}\big(\phi(x), \exp_{\phi(o)}(t\phi_{*,o}(v))\big) \quad \text{by \Cref{lemma:lee_geodesics}} \\
            &= P^{\phi_{*,o}(v)}\big(\phi(x)\big).
        \end{aligned}
    \end{equation}
\end{proof}

\begin{proof}\textbf{of \Cref{prop:isometry_chsw}}
    First, let us show that for $\lambda_o$-almost all $v\in S_o$, $W_p^p(B^{v}_\#\mu, B^v_\#\nu) = W_p^p(B^{\phi_{*,o}(v)}_\#\Tilde{\mu}, B^{\phi_{*,o}(v)}_\#\Tilde{\nu})$ and $W_p^p(P^v_\#\mu, P^v_\#\nu) = W_p^p(P^{\phi_{*,o}(v)}_\#\mu, P^{\phi_{*,o}(v)}_\#\nu)$. Using \Cref{lemma:commute_projs}, we have
    \begin{align}
        W_p^p(B^v_\#\mu, B^v_\#\nu) = W_p^p(B^{\phi_{*,o}(v)}_\#\phi_\#\mu, B^{\phi_{*,o}(v)}_\#\phi_\#\nu) = W_p^p(B^{\phi_{*,o}(v)}_\#\Tilde{\mu}, B^{\phi_{*,o}(v)}_\#\Tilde{\nu}), \\
        W_p^p(P^v_\#\mu, P^v_\#\nu) = W_p^p(P^{\phi_{*,o}(v)}_\#\phi_\#\mu, P^{\phi_{*,o}(v)}_\#\phi_\#\nu) = W_p^p(P^{\phi_{*,o}(v)}_\#\Tilde{\mu}, P^{\phi_{*,o}(v)}_\#\Tilde{\nu}).
    \end{align}
    These results are true for all $v\in S_o$, and therefore for $\lambda_o$-almost all $v\in S_o$. Thus, by integrating with respect to $\lambda_o$, and performing the change of $v\mapsto \phi_{*,o}(v)$ on the right side, we find
    \begin{align}
        &\hchsw_p^p(\mu,\nu;\lambda_o) = \hchsw_p^p(\Tilde{\mu},\Tilde{\nu};(\phi_{*,o})_\#\lambda_o), \\
        &\gchsw_p^p(\mu,\nu;\lambda_o) = \gchsw_p^p(\Tilde{\mu}, \Tilde{\nu}; (\phi_{*,o})_\#\lambda_o).
    \end{align}
    Finally, we can conclude by using that $\phi_{*,o}$ is an isometry between the tangent spaces and hence $(\phi_{*,o})_\#\lambda_o = \lambda_{\phi(o)}$.
\end{proof}

\subsection{Proof of \Cref{prop:proj_log_cholesky}} \label{proof:prop_proj_log_cholesky}

First, let us compute the differential of $\phi=\varphi\circ \mathcal{L}$. In that purpose, we first recall the differential of $\mathcal{L}:X=LL^T\mapsto L$ derived in \citep[Proposition 4]{lin2019riemannian}.

\begin{lemma}[Proposition 4 in \citep{lin2019riemannian}]
    Let $X\in S_d^{++}(\mathbb{R})$ and $V\in S_d(\mathbb{R})$. The differential operator $\mathcal{L}_{*,X}:T_{X}S_d^{++}(\mathbb{R})\to T_{\mathcal{L}(X)}L_d^{++}(\mathbb{R})$ of $\mathcal{L}$ at $X$ is given by
    \begin{equation}
        \mathcal{L}_{*,X}(V) = \mathcal{L}(X) \Big(\lfloor \mathcal{L}(X)^{-1} V \mathcal{L}(X)^{-T} \rfloor + \frac12 \mathrm{diag}\big(\mathcal{L}(X)^{-1} V \mathcal{L}(X)^{-T}\big)\Big).
    \end{equation}
\end{lemma}

\begin{lemma} \label[lemma]{lemma:diff_log_cholesky}
    Let $\phi:X\mapsto \varphi\big(\mathcal{L}(X)\big)$ and $X=LL^T\in S_d^{++}(\mathbb{R})$ with $L\in L_d^{++}(\mathbb{R})$ obtained by the Cholesky decomposition. The differential operator of $\phi$ at $X$ is given by
    \begin{equation}
        \forall V\in T_X S_d^{++}(\mathbb{R}),\ \phi_{*,X}(V) = \lfloor \mathcal{L}_{*,X}(V)\rfloor + \mathrm{diag}\big(\mathcal{L}(X)\big)^{-1} \mathrm{diag}\big(\mathcal{L}_{*,X}(V)\big),
    \end{equation}
    where
    \begin{equation}
        \mathcal{L}_{*,X}(V) = \mathcal{L}(X) \Big(\lfloor \mathcal{L}(X)^{-1} V \mathcal{L}(X)^{-T} \rfloor + \frac12 \mathrm{diag}\big(\mathcal{L}(X)^{-1} V \mathcal{L}(X)^{-T}\big)\Big).
    \end{equation}
\end{lemma}

\begin{proof}\textbf{of \Cref{lemma:diff_log_cholesky}}
    Using the chain rule, we have, for $X\in S_d^{++}(\mathbb{R})$ and $V\in S_d(\mathbb{R})$,
    \begin{equation}
        \begin{aligned}
            \phi_{*,X}(V) &= \varphi_{*,X}\big(\mathcal{L}_{*,X}(V)\big) = \lfloor \mathcal{L}_{*,X}(V)\rfloor + \log_{*,\mathrm{diag}(L)}\big(\mathrm{diag}(\mathcal{L}_{*,X}(V))\big) \\
            &= \lfloor \mathcal{L}_{*,X}(V)\rfloor + \Sigma\big(\mathrm{diag}(\mathcal{L}_{*,X}(V))\big),
        \end{aligned}
    \end{equation}
    using \Cref{lemma:pennec_diff_log} for the differenutal of the log with
    \begin{equation}
        \begin{aligned}
            \Sigma\big(\mathrm{diag}(\mathcal{L}_{*,X}(V))\big) &= \mathrm{diag}(\mathcal{L}_{*,X}(V)) \odot \Gamma \\
            &= \mathrm{diag}(\mathcal{L}_{*,X}(V)) \oslash \mathrm{diag}(\mathcal{L}(X))  \\
            &= \mathrm{diag}(\mathcal{L}(X))^{-1} \mathrm{diag}(\mathcal{L}_{*,X}(V)).
        \end{aligned}
    \end{equation}
    Thus, we conclude that $\phi_{*,X}(V) = \lfloor \mathcal{L}_{*,X}(V)\rfloor + \mathrm{diag}\big(\mathcal{L}(X)\big)^{-1} \mathrm{diag}\big(\mathcal{L}_{*,X}(V)\big)$.
\end{proof}

\begin{proof}\textbf{of \Cref{prop:proj_log_cholesky}}
    On one hand we have $\phi(I_d) = 0$, $\mathcal{L}_{*,I_d}(V) = \lfloor V\rfloor + \frac12 \mathrm{diag}(V)$ and thus $\phi_{*,I_d}(V) = \lfloor V \rfloor + \frac12 \mathrm{diag}(V)$ since $\mathcal{L}(I_d) = I_d$. Thus, using \Cref{prop:proj_coord_pullback}, the projection is given, for $A\in S_d(\mathbb{R})$ such that $\|A\|_{I_d}^2 = \langle \phi_{*,I_d}(A), \phi_{*,I_d}(A)\rangle_F = 1$, by
    \begin{equation}
        \begin{aligned}
            \forall X=LL^T\in S_d^{++}(\mathbb{R}),\ P^A(X) &= \langle \phi(X), \phi_{*,I_d}(A)\rangle_F \\
            &= \langle \lfloor L\rfloor + \log(\mathrm{diag}(L)), \lfloor A\rfloor + \frac12 \mathrm{diag}(A)\rangle_F \\
            &= \langle \lfloor L\rfloor, \lfloor A\rfloor\rangle_F + \langle \log(\mathrm{diag}(L)), \frac12 \mathrm{diag}(A)\rangle_F.
        \end{aligned}
    \end{equation}
\end{proof}

\subsection{Proof of \Cref{prop:busemann_product}} \label{proof:prop_busemann_product}

\begin{proof}\textbf{of \Cref{prop:busemann_product}}
    We use that $B^\gamma(x) = \lim_{t\to\infty}\ \frac{d(x,\gamma(t))^2 - t^2}{2t}$ (see \emph{e.g.} \citep[II. 8.24]{bridson2013metric}). Thus,
    \begin{equation}
        \begin{aligned}
            B^\gamma(x) &= \lim_{t\to\infty}\ d(x,\gamma(t)) - t \\
            &= \lim_{t\to\infty}\ \frac{d(x,\gamma(t))^2 - t^2}{2t} \\
            &= \lim_{t\to\infty}\ \sum_{i=1}^n \lambda_i \frac{d_i(x_i, \gamma_i(\lambda_i t))^2 - \lambda_i^2 t^2}{2\lambda_i t} \\
            &= \sum_{i=1}^n \lambda_i B^{\gamma_i}(x_i).
        \end{aligned}
    \end{equation}
\end{proof}

\section{Proofs of \Cref{section:chsw_properties}} \label{proofs:section_chsw_properties}

\subsection{Proof of \Cref{prop:chsw_pseudo_distance}} \label{proof:prop_chsw_pseudo_distance}

\begin{proof}\textbf{of \Cref{prop:chsw_pseudo_distance}}
    First, we will show that for any $\mu,\nu\in\mathcal{P}_p(\mathcal{M})$, $\chsw_p(\mu,\nu) < \infty$. Let $\mu,\nu \in \mathcal{P}_p(\mathcal{M})$, and let $\gamma\in \Pi(\mu,\nu)$ be an arbitrary coupling between them.
    Then by using first \Cref{lemma:paty} followed by the 1-Lipschitzness of the projections \Cref{lemma:lipschitz} and \Cref{lemma:inequality_distance}, we obtain
    \begin{equation}
        \begin{aligned}
            W_p^p(P^v_\#\mu, P^v_\#\nu) &= \inf_{\gamma\in\Pi(\mu,\nu)}\ \int |P^v(x)-P^v(y)|^p\ \mathrm{d}\gamma(x,y) \\
            &\le \int |P^v(x)-P^v(y)|^p\ \mathrm{d}\gamma(x,y) \\
            &\le \int d(x,y)^p\ \mathrm{d}\gamma(x,y) \\
            &\le 2^{p-1} \left(\int d(x,o)^p\ \mathrm{d}\mu(x) + \int d(o, y)^p\ \mathrm{d}\nu(y)\right) \\
            &< \infty.
        \end{aligned}
    \end{equation}
    Hence, we can conclude that $\chsw_p^p(\mu,\nu) < \infty$.

    Now, let us show that it is a pseudo-distance. First, it is straightforward to see that $\chsw_p(\mu,\nu)\ge 0$, that it is symmetric, \emph{i.e.} $\chsw_p(\mu,\nu)=\chsw_p(\nu,\mu)$, and that $\mu=\nu$ implies that $\chsw_p(\mu,\nu)=0$ using that $W_p$ is well a distance.

    For the triangular inequality, we can derive it using the triangular inequality for $W_p$ and the Minkowski inequality. Let $\mu,\nu,\alpha\in\mathcal{P}_p(\mathcal{M})$,
    \begin{equation}
        \begin{aligned}
            \chsw_p(\mu,\nu) &= \left(\int_{S_o} W_p^p(P^v_\#\mu, P^v_\#\nu)\ \mathrm{d}\lambda_o(v)\right)^{\frac{1}{p}} \\
            &\le \left(\int_{S_o} \big(W_p(P^v_\#\mu, P^v_\#\alpha) + W_p(P^v_\#\alpha, P^v_\#\nu)\big)^p\ \mathrm{d}\lambda_o(v)\right)^{\frac{1}{p}} \\ 
            &\le \left(\int_{S_o} W_p^p(P^v_\#\mu, P^v_\#\alpha)\ \mathrm{d}\lambda_o(v)\right)^{\frac{1}{p}} + \left(\int_{S_o} W_p^p(P^v_\#\alpha, P^v_\#\nu)\ \mathrm{d}\lambda_o(v)\right)^{\frac{1}{p}} \\
            &= \chsw_p(\mu,\alpha) + \chsw_p(\alpha,\nu).
        \end{aligned}
    \end{equation}
\end{proof}

\subsection{Proof of \Cref{prop:chsw_dual_rt}} \label{proof:prop_chsw_dual_rt}

\begin{proof}\textbf{of \Cref{prop:chsw_dual_rt}}
    Let $f\in L^1(\mathcal{M})$, $g\in C_0(\mathbb{R}\times S_o)$, then by Fubini's theorem,
    \begin{equation}
        \begin{aligned}
            \langle \chr f, g\rangle_{\mathbb{R} \times S_o} &= \int_{S_o}\int_{\mathbb{R}} \chr f(t, v) g(t, v) \ \mathrm{d}t\mathrm{d}\lambda_o(v) \\
            &= \int_{S_o}\int_{\mathbb{R}} \int_{\mathcal{M}}f(x) \mathbb{1}_{\{t= P^v(x)\}} g(t,v) \ \mathrm{d}\vol(x)\mathrm{d}t\mathrm{d}\lambda_o(v) \\
            &= \int_{\mathcal{M}} f(x) \int_{S_o}\int_{\mathbb{R}} g(t,v)\mathbb{1}_{\{t= P^v(x)\}}\ \mathrm{d}t\mathrm{d}\lambda_o(v)\mathrm{d}\vol(x) \\
            &= \int_{\mathcal{M}} f(x) \int_{S_o} g\big(P^v(x), v\big)\ \mathrm{d}\lambda_o(v)\mathrm{d}\vol(x) \\
            &= \int_{\mathcal{M}} f(x) \chr^*g(x)\ \mathrm{d}\vol(x) \\
            &= \langle f, \chr^*g\rangle_{\mathcal{M}}.
        \end{aligned}
    \end{equation}
\end{proof}

\subsection{Proof of \Cref{prop:chr_vanish}} \label{proof:prop_chr_vanish}

\begin{proof}\textbf{of \Cref{prop:chr_vanish}}
    We follow the proof of \citep[Lemma 1]{boman2009support}. On one hand, $g\in C_0(\mathbb{R}\times S_o)$, thus for all $\epsilon>0$, there exists $M>0$ such that $|t|\ge M$ implies $|g(t,v)|\le \epsilon$ for all $v\in S_o$.

    Let $\epsilon>0$ and $M>0$ which satisfies the previous property. Denote $E(x,M) = \{v\in S_o,\ |P^v(x)|<M\}$. Then, as $d(x,o)>0$, we have
    \begin{equation}
        E(x,M) = \{v\in S_o,\ |P^v(x)|<M\} = \left\{v\in S_o,\ \frac{P^v(x)}{d(x,o)} < \frac{M}{d(x,o)}\right\} \xrightarrow[d(x,o)\to\infty]{} \emptyset.
    \end{equation}
    Thus, $\lambda_o\big(E(x,M)\big) \xrightarrow[d(x,o)\to\infty]{} 0$. Choose $M'$ such that $d(x,o)>M'$ implies that $\lambda_o\big(E(x,M)\big) < \epsilon$.

    Then, for $x\in \mathcal{M}$ such that $|P^v(x)| \ge \max(M,M')$ (and thus $d(x,o)\ge M'$ since $|P^v(x)|\le d(x,o)$ as $P^v$ is Lipschitz,
    \begin{equation}
        \begin{aligned}
            |\chr^*g(x)| &\le \left| \int_{E(x,M)} g(P^v(x),v)\ \mathrm{d}\lambda_o(v) \right| + \left| \int_{E(x,M)^c} g(P^v(x),v)\ \mathrm{d}\lambda_o(v) \right| \\
            &\le \|g\|_\infty\ \lambda_o\big(E(x,M)\big)  + \epsilon \lambda_o\big(E(x,M)^c\big) \\
            &\le \|g\|_\infty \epsilon + \epsilon.
        \end{aligned}
    \end{equation}
    Thus, we showed that $\chr^*g(x) \xrightarrow[d(x,o)\to\infty]{} 0$, and thus $\chr^*g\in C_0(\mathcal{M})$.
\end{proof}

\subsection{Proof of \Cref{prop:chsw_disintegration}} \label{proof:prop_chsw_disintegration}

\begin{proof}\textbf{of \Cref{prop:chsw_disintegration}}
    Let $g\in C_0(\mathbb{R} \times S_o)$, as $\chr\mu=\lambda_o\otimes K_\mu$, we have by definition
    \begin{equation}
        \int_{S_o} \int_{\mathbb{R}} g(t, v) \ K_\mu(v,\mathrm{d}t)\ \mathrm{d}\lambda_o(v) = \int_{\mathbb{R} \times S_o} g(t,v)\ \mathrm{d}(\chr\mu)(t,v).
    \end{equation}
    Hence, using the property of the dual, we have for all $g\in C_o(\mathbb{R}\times S_o)$,
    \begin{equation}
        \begin{aligned}
            \int_{S_o} \int_{\mathbb{R}} g(t, v) \ K_\mu(v,\mathrm{d}t)\ \mathrm{d}\lambda_o(v) &= \int_{\mathbb{R} \times S_o} g(t,v)\ \mathrm{d}(\chr\mu)(t,v) \\
            &= \int_{\mathcal{M}} \chr^*g(x) \ \mathrm{d}\mu(x) \\
            &= \int_{\mathcal{M}} \int_{S_o} g(P^v(x), v)\ \mathrm{d}\lambda_o(v) \mathrm{d}\mu(x) \\
            &= \int_{S_o} \int_{\mathcal{M}} g(P^v(x), v) \ \mathrm{d}\mu(x) \mathrm{d}\lambda_o(v) \\
            &= \int_{S_o} \int_{\mathbb{R}} g(t, v) \ \mathrm{d}(P^v_\#\mu)(t)\mathrm{d}\lambda_o(v).
        \end{aligned}
    \end{equation}
    Hence, for $\lambda_o$-almost every $v\in S_o$, $K_\mu(v,\cdot) = P^v_\#\mu$.
\end{proof}

\subsection{Proof of \Cref{prop:chsw_upperbound}} \label{proof:prop_chsw_upperbound}

\begin{proof}\textbf{of \Cref{prop:chsw_upperbound}}
    Using \Cref{lemma:paty} and that the projections are 1-Lipschitz (\Cref{lemma:lipschitz}), we can show that, for any $\mu,\nu\in\mathcal{P}_p(\mathcal{M})$,
    \begin{equation}
        \begin{aligned}
            \chsw_p^p(\mu,\nu) &= \inf_{\gamma\in \Pi(\mu,\nu)}\ \int |P^v(x)-P^v(y)|^p\ \mathrm{d}\gamma(x,y).
        \end{aligned}
    \end{equation}
    Let $\gamma^*\in \Pi(\mu,\nu)$ being an optimal coupling for the Wasserstein distance with ground cost $d$, then,
    \begin{equation}
        \begin{aligned}
            \chsw_p^p(\mu,\nu) &\le \int |P^v(x)-P^v(y)|^p\ \mathrm{d}\gamma^*(x,y) \\
            &\le \int d(x,y)^p\ \mathrm{d}\gamma^*(x,y) \\
            &= W_p^p(\mu,\nu).
        \end{aligned}
    \end{equation}
\end{proof}

\subsection{Proof of \Cref{prop:chsw_hilbertian}} \label{proof:prop_chsw_hilbertian}

\begin{proof}\textbf{of \Cref{prop:chsw_hilbertian}}
    Let $\mu,\nu\in\mathcal{P}_p(\mathcal{M})$, then
    \begin{equation}
        \begin{aligned}
            \chsw_p^p(\mu,\nu) &= \int_{S_o} W_p^p(P^v_\#\mu, P^v_\#\nu)\ \mathrm{d}\lambda_o(v) \\
            &= \int_{S_o} \|F_{P^v_\#\mu}^{-1} - F_{P^v_\#\nu}^{-1}\|^p_{L^p([0,1])}\ \mathrm{d}\lambda_o(v) \\
            &= \int_{S_o} \int_0^1 \big( F_{P^v_\#\mu}^{-1}(q) - F_{P^v_\#\nu}^{-1}(q)\big)^p\ \mathrm{d}q\ \mathrm{d}\lambda_o(v) \\
            &= \|\Phi(\mu)-\Phi(\nu)\|_{\mathcal{H}}^p.
        \end{aligned}
    \end{equation}
    Thus, $\chsw_p$ is Hilbertian.
\end{proof}

\subsection{Proof of \Cref{lemma:chsw_pullback}} \label{proof:lemma_chsw_pullback}

\begin{proof}\textbf{of \Cref{lemma:chsw_pullback}}
    % \begin{equation}
    %     \begin{aligned}
    %         \chsw_p^p(\mu,\nu) &= \int_{S_o} W_p^p(P^v_\#\mu, P^v_\#\nu)\ \mathrm{d}\lambda(v) \\
    %         &= \int_{S_o} W_p^p(t^v_\#\phi_\#\mu, t^v_\#\phi_\#\nu)\ \mathrm{d}\lambda(v) \\
    %         &= \sw_p^p(\phi_\#\mu, \phi_\#\nu).
    %     \end{aligned}
    % \end{equation}

    Since for any $v\in S_o$ and $x\in \mathcal{M}$, $P^v(x) = \langle \phi(x)-\phi(o), \phi_{*,o}(x)\rangle$, by using \Cref{lemma:paty} we have 
    \begin{equation}
        \begin{aligned}
            W_p^p(P^v_\#\mu, P^v_\#\nu) &= \inf_{\gamma\in \Pi(P^v_\#\mu, P^v_\#\nu)} \int |x-y|^p\ \mathrm{d}\gamma(x,y) \\
            &= \inf_{\gamma\in\Pi(\mu,\nu)}\ \int |P^v(x)-P^v(y)|^p\ \mathrm{d}\gamma(x,y) \\
            &= \inf_{\gamma\in\Pi(\mu,\nu)}\ \int |\langle \phi(x)-\phi(y), \phi_{*,o}(v)\rangle|^p\ \mathrm{d}\gamma(x,y).
        \end{aligned}
    \end{equation}
    Let's note $Q^v(x) = \langle x,v\rangle$. Then, we obtain
    \begin{equation}
        \begin{aligned}
            W_p^p(P^v_\#\mu, P^v_\#\nu) &= \inf_{\gamma\in\Pi(\mu,\nu)} \int |Q^{\phi_{*,o}(v)}(\phi(x)) - Q^{\phi_{*,o}(v)}(\phi(y))|^p\ \mathrm{d}\gamma(x,y) \\ &= W_p^p(Q^{\phi_{*,o}(v)}_\#\phi_\#\mu, Q^{\phi_{*,o}(v)}_\#\phi_\#\nu).
        \end{aligned}
    \end{equation}
    Therefore, we obtain
    \begin{equation}
        \begin{aligned}
            \chsw_p^p(\mu,\nu) &= \int_{S_o} W_p^p(Q^{\phi_{*,o}(v)}_\#\phi_\#\mu, Q^{\phi_{*,o}(v)}_\#\phi_\#\nu)\ \mathrm{d}\lambda_o(v) \\ &= \int_{S_{\phi(o)}} W_p^p(Q^v_\#\phi_\#\mu, Q^v_\#\phi_\#\nu)\ \mathrm{d}\big((\phi_{*,o})_\#\lambda_o\big)(v).
        \end{aligned}
    \end{equation}
    % Finally, by \citep[Proposition 9.8]{gallier2011notes}, 
    Finally, since $\phi_{*,o}$ is an isometry between the tangent spaces by definition of the metric, we have $(\phi_{*,o})_\#\lambda_o = \lambda_{\phi(o)}$.
\end{proof}

\subsection{Proof of \Cref{prop:distance_chsw_pullback}} \label{proof:prop_distance_chsw_pullback}

\begin{proof}\textbf{of \Cref{prop:distance_chsw_pullback}}
    % We know by \Cref{prop:chsw_pseudo_distance} that $\chsw_p$ is a finite pseudo-distance. For the indiscernible property, using \Cref{lemma:chsw_pullback} and the distance property of $\sw_p$, we have that $\chsw_p(\mu,\nu)=\sw_p^p(\phi_\#\mu,\phi_\#\nu; (\phi_{*,o})_\#\lambda_o)=0$ implies that $\phi_\#\mu=\phi_\#\nu$ by applying the same proof of \citep[Proposition 5.1.2]{bonnotte2013unidimensional}. Indeed, we have that $\sw_p^p(\phi_\#\mu, \phi_\#\nu; (\phi_{*,o})_\#\lambda_o) = 0$ implies $W_p^p(Q^v_\#\phi_\#\mu, Q^v_\#\phi_\#\nu) =0$ for $(\phi_{*,o})_\#\lambda_{o}$-almost every $v\in S_{\phi(o)}$, and thus that $Q^v_\#\phi_\#\mu = Q^v_\#\phi_\#\nu$ since $W_p$ is a distance. Hence, using the Fourier transform and that $(\phi_{*,o})_\#\lambda_o$ is absolutely continuous with respect to the Lebesgue measure, we obtain that $\phi_\#\mu = \phi_\#\nu$.

    We know by \Cref{prop:chsw_pseudo_distance} that $\chsw_p$ is a finite pseudo-distance. For the indiscernible property, using \Cref{lemma:chsw_pullback} and the distance property of $\sw_p$, we have that $\chsw_p(\mu,\nu)=\sw_p^p(\phi_\#\mu,\phi_\#\nu)=0$ implies that $\phi_\#\mu=\phi_\#\nu$ by applying the same proof of \citep[Proposition 5.1.2]{bonnotte2013unidimensional}. Indeed, we have that $\sw_p^p(\phi_\#\mu, \phi_\#\nu) = 0$ implies $W_p^p(Q^v_\#\phi_\#\mu, Q^v_\#\phi_\#\nu) =0$ for $\lambda_{\phi(o)}$-almost every $v\in S_{\phi(o)}$, and thus that $Q^v_\#\phi_\#\mu = Q^v_\#\phi_\#\nu$ since $W_p$ is a distance. Hence, using the Fourier transform and that $\lambda_{\phi(o)}$ is absolutely continuous with respect to the Lebesgue measure, we obtain that $\phi_\#\mu = \phi_\#\nu$.

    Then, as $\phi$ is a bijection from $\mathcal{M}$ to $\mathcal{N}$, we have for all Borelian $C\subset \mathcal{M}$,
    \begin{equation}
        \begin{aligned}
            \mu(C) &= \int_{\mathcal{M}} \mathbb{1}_C(x)\ \mathrm{d}\mu(x) \\
            &= \int_{\mathcal{N}} \mathbb{1}_C\big(\phi^{-1}(y)\big)\ \mathrm{d}(\phi_\#\mu)(y) \\
            &= \int_{\mathcal{N}} \mathbb{1}_C\big(\phi^{-1}(y)\big)\ \mathrm{d}(\phi_\#\nu)(y) \\
            &= \int_{\mathcal{M}} \mathbb{1}_C(x)\ \mathrm{d}\nu(x) \\
            &= \nu(C).
        \end{aligned}
    \end{equation}
\end{proof}

\subsection{Proof of \Cref{prop:weak_cv_chsw_pullback}} \label{proof:prop_weak_cv_chsw_pullback}

To prove \Cref{prop:weak_cv_chsw_pullback}, we will adapt the proof of \citet{nadjahi2020statistical} to our projection. First, we start to adapt \citet[Lemma S1]{nadjahi2020statistical}:
\begin{lemma}[Lemma S1 in \citet{nadjahi2020statistical}] \label{lemma:nadjahi}
    Let $(\mu_k)_k \in \mathcal{P}_p(\mathcal{M})$ and $\mu\in \mathcal{P}_p(\mathcal{M})$ such that $\lim_{k\to\infty}\ \chsw_1(\mu_k,\mu)=0$. Then, there exists $\varphi:\mathbb{N}\to\mathbb{N}$ non decreasing such that $\mu_{\varphi(k)} \xrightarrow[k\to\infty]{\mathcal{L}} \mu$.
\end{lemma}

\begin{proof}\textbf{of \Cref{lemma:nadjahi}}
    % Using \Cref{lemma:chsw_pullback}, we know that $\chsw_1(\mu,\nu) = \sw_1(\phi_\#\mu,\phi_\#\nu;(\phi_{*,o})_\#\lambda_o)$. Let's note $\alpha_k = \phi_\#\mu_k\in\mathcal{P}_p(\mathcal{N})$ and $\alpha = \phi_\#\mu\in\mathcal{P}_p(\mathcal{N})$ and $Q^v(x) = \langle v, x\rangle$.

    Using \Cref{lemma:chsw_pullback}, we know that $\chsw_1(\mu,\nu) = \sw_1(\phi_\#\mu,\phi_\#\nu)$. Let's note $\alpha_k = \phi_\#\mu_k\in\mathcal{P}_p(\mathcal{N})$ and $\alpha = \phi_\#\mu\in\mathcal{P}_p(\mathcal{N})$ and $Q^v(x) = \langle v, x\rangle$.
    
    Then, by \citet[Theorem 2.2.5]{bogachev2007measure}, 
    % \begin{equation}
    %     \lim_{k\to\infty}\ \int_{S_{\phi(o)}} W_1(Q^v_\#\alpha_k, Q^v_\#\alpha)\ \mathrm{d}(\phi_{*,o})_\#\lambda_o(v) = 0
    % \end{equation}
    % implies that there exists a subsequence $(\mu_{\varphi(k)})_k$ such that for $(\phi_{*,o})_\#\lambda_o$-almost every $v$,
    \begin{equation}
        \lim_{k\to\infty}\ \int_{S_{\phi(o)}} W_1(Q^v_\#\alpha_k, Q^v_\#\alpha)\ \mathrm{d}\lambda_{\phi(o)}(v) = 0
    \end{equation}
    implies that there exists a subsequence $(\mu_{\varphi(k)})_k$ such that for $\lambda_{\phi(o)}$-almost every $v$,
    \begin{equation}
        W_1(Q^v_\#\alpha_{\varphi(k)}, Q^v_\#\alpha) \xrightarrow[k\to \infty]{}0.
    \end{equation}
    As the Wasserstein distance metrizes the weak convergence, this is equivalent to $Q^v_\#\mu_{\varphi(k)} \xrightarrow[k\to\infty]{\mathcal{L}} Q^v_\#\mu$.
    
    Then, by Levy's characterization theorem, this is equivalent with the pointwise convergence of the characterization function, \emph{i.e.} for all $t\in \mathbb{R}$,\ $\Phi_{Q^v_\#\alpha_{\varphi(k)}}(t) \xrightarrow[k\to \infty]{} \Phi_{Q^v_\#\mu}(t)$.
    Then, working in $T_{\phi(o)}\mathcal{N}$ with the Euclidean norm, we can use the same proof of \citet{nadjahi2020statistical} by using a convolution with a gaussian kernel and show that it implies that $\alpha_{\varphi(k)}\xrightarrow[k\to\infty]{\mathcal{L}}\alpha$, \emph{i.e.} $\phi_\#\mu_{\varphi(k)} \xrightarrow[k\to\infty]{\mathcal{L}}\phi_\#\mu$.
    
    Finally, let's show that it implies the weak convergence of $(\mu_{\varphi(k)})_k$ towards $\mu$. Let $f\in C_b(\mathcal{M})$, then
    \begin{equation}
        \begin{aligned}
            \int_{\mathcal{M}} f \ \mathrm{d}\mu_{\varphi(k)} = \int_{\mathcal{N}} f\circ \phi^{-1}\ \mathrm{d}(\phi_\#\mu_{\varphi(k)}) \xrightarrow[k\to \infty]{} \int_{\mathcal{N}} f\circ \phi^{-1} \ \mathrm{d}(\phi_\#\mu) 
            = \int_{\mathcal{M}} f\ \mathrm{d}\mu.
        \end{aligned}
    \end{equation}
    Hence, we an conclude that $\mu_{\varphi(k)} \xrightarrow[k\to\infty]{\mathcal{L}} \mu$.
\end{proof}

\begin{proof}\textbf{of \Cref{prop:weak_cv_chsw_pullback}}
    First, we suppose that $\mu_k \xrightarrow[k\to\infty]{\mathcal{L}} \mu$ in $\mathcal{P}_p(\mathcal{M})$. Then, by continuity, we have that for $\lambda_o$ almost every $v\in T_o\mathcal{M}$, $P^v_\#\mu_k \xrightarrow[k\to \infty]{} P^v_\#\mu$. Moreover, as the Wasserstein distance on $\mathbb{R}$ metrizes the weak convergence, $W_p(P^v_\#\mu_k, P^v_\#\mu) \xrightarrow[k\to\infty]{} 0$. Finally, as $W_p$ is bounded and it converges for $\lambda_o$-almost every $v$, we have by the Lebesgue convergence dominated theorem that $\chsw_p^p(\mu_k,\mu) \xrightarrow[k\to\infty]{} 0$.

    For the opposite side, suppose that $\chsw_p(\mu_k, \mu) \xrightarrow[k\to\infty]{}0$. Then, since we generalized \citep[Lemma S1]{nadjahi2020statistical} to our setting in \Cref{lemma:nadjahi}, we can use the same contradiction argument as \citet{nadjahi2020statistical} and we conclude that $(\mu_k)_k$ converges weakly to $\mu$.

    % For the opposite side, suppose that $\chsw_p(\mu_k,\mu)\xrightarrow[k\to\infty]{}0$. Thus, using \Cref{lemma:chsw_pullback}, $\sw_p(\phi_\#\mu_k,\phi_\#\mu)\xrightarrow[k\to\infty]{}0$ which implies that $\phi_\#\mu_k \xrightarrow[k\to\infty]{\mathcal{L}} \phi_\#\mu$ since $\sw_p$ metrizes the weak convergence by \citep[Theorem 1]{nadjahi2019asymptotic}. Finally, let's show that it implies that $\mu_k\xrightarrow[k\to\infty]{\mathcal{L}} \mu$ using that $\phi$ is a diffeomorphism. Let $f\in C_b(\mathcal{M})$, then
    % \begin{equation}
    %     \begin{aligned}
    %         \int_{\mathcal{M}} f \mathrm{d}\mu_k = \int_{\mathcal{N}} f\circ\phi^{-1}\ \mathrm{d}(\phi_\#\mu_k) &\xrightarrow[k\to\infty]{} \int_{\mathcal{N}} f\circ\phi^{-1}\ \mathrm{d}(\phi_\#\mu) = \int_{\mathcal{M}} f\ \mathrm{d}\phi.
    %     \end{aligned}
    % \end{equation}
    % Thus, we deduce that $\mu_k\xrightarrow[k\to\infty]{\mathcal{L}}\mu$.
\end{proof}

\subsection{Proof of \Cref{prop:lowerbound_chsw_pullback}} \label{proof:prop_lowerbound_chsw_pullback}

\begin{proof}\textbf{of \Cref{prop:lowerbound_chsw_pullback}}
    By using \Cref{lemma:paty}, let us first observe that
    \begin{equation}
        \begin{aligned}
            W_1(\mu,\nu) &= \inf_{\gamma\in\Pi(\mu,\nu)}\ \int_{\mathcal{M}\times \mathcal{M}} d_\mathcal{M}(x,y)\ \mathrm{d}\gamma(x,y) \\
            &= \inf_{\gamma\in\Pi(\mu,\nu)}\ \int_{\mathcal{M}\times\mathcal{M}} \|\phi(x)-\phi(y)\|\ \mathrm{d}\gamma(x,y) \\
            &= \inf_{\gamma\in\Pi(\mu,\nu)}\ \int_{\mathcal{N}\times\mathcal{N}} \|x-y\|\ \mathrm{d}(\phi\otimes\phi)_\#\gamma(x,y) \\
            &= \inf_{\gamma\in\Pi(\phi_\#\mu,\phi_\#\nu)}\ \int_{\mathcal{N}\times \mathcal{N}} \|x-y\|\ \mathrm{d}\gamma(x,y) \\
            &= W_1(\phi_\#\mu,\phi_\#\nu).
        \end{aligned}
    \end{equation}
    Here, we note that $W_1$ must be understood with respect to the ground cost metric which makes sense given the space, \emph{i.e.} $d_\mathcal{M}$ on $\mathcal{M}$ and $\|\cdot-\cdot\|$ on $\mathcal{N}$.

    Then, using \Cref{lemma:chsw_pullback}% and that $\phi_{*,o}=\id$
    , we have
    \begin{equation}
        \chsw_1(\mu,\nu) = \sw_1(\phi_\#\mu, \phi_\#\nu).
    \end{equation}
    Since $\mathcal{N}$ is a Euclidean inner product space of dimension $d$, we can apply \citep[Lemma 5.14]{bonnotte2013unidimensional}, and we obtain
    \begin{equation}
        W_1(\mu,\nu) = W_1(\phi_\#\mu,\phi_\#\nu) \le C_{d,p,r} \sw_1(\phi_\#\mu,\phi_\#\nu)^{\frac{1}{d+1}} = C_{d,p,r} \chsw_1(\mu,\nu)^{\frac{1}{d+1}}.
    \end{equation}
    Then, using that $W_p^p(\mu,\nu) \le (2r)^{p-1} W_1(\mu,\nu)$ and that by the Hölder inequality, $\chsw_1(\mu,\nu) \le \chsw_p(\mu,\nu)$, we obtain (with a different constant $C_{d,r,p}$)
    \begin{equation}
        W_p^p(\mu,\nu) \le C_{d,r,p} \chsw_p(\mu,\nu)^{\frac{1}{d+1}}.
    \end{equation}
\end{proof}

\subsection{Proof of \Cref{prop:chsw_sample_complexity}} \label{proof:prop_chsw_sample_complexity}

\begin{proof}\textbf{of \Cref{prop:chsw_sample_complexity}}
    First, using the triangular inequality, the reverse triangular inequality and the Jensen inequality for $x\mapsto x^{1/p}$ (which is concave since $p\ge 1$),  we have the following inequality

    \begin{equation}
        \begin{aligned}
            &\mathbb{E}[|\chsw_p(\hat{\mu}_n,\hat{\nu}_n) - \chsw_p(\mu,\nu)|] \\ &= \mathbb{E}[|\chsw_p(\hat{\mu}_n,\hat{\nu}_n) - \chsw_p(\hat{\mu}_n,\nu) + \chsw_p(\hat{\mu}_n,\nu) - \chsw_p(\mu,\nu)|] \\
            &\le \mathbb{E}[|\chsw_p(\hat{\mu}_n,\hat{\nu}_n) - \chsw_p(\hat{\mu}_n,\nu)|] + \mathbb{E}[|\chsw_p(\hat{\mu}_n,\nu) - \chsw_p(\mu,\nu)|] \\
            &\le \mathbb{E}[\chsw_p(\nu,\hat{\nu}_n)] + \mathbb{E}[\chsw_p(\mu,\hat{\mu}_n)] \\
            &\le \mathbb{E}[\chsw_p^p(\nu,\hat{\nu}_n)]^{1/p} + \mathbb{E}[\chsw_p^p(\mu, \hat{\mu}_n)]^{1/p}.
        \end{aligned}
    \end{equation}

    Moreover, by Fubini-Tonelli,
    \begin{equation}
        \begin{aligned}
            \mathbb{E}[\chsw_p^p(\hat{\mu}_n,\mu)] &= \mathbb{E}\left[\int_{S_o} W_p^p(P^v_\#\hat{\mu}_n, \mu)\ \mathrm{d}\lambda_o(v)\right] \\
            &= \int_{S_o} \mathbb{E}[W_p^p(P^v_\#\hat{\mu}_n,P^v_\#\mu)]\ \mathrm{d}\lambda_o(v).
        \end{aligned}
    \end{equation}
    Then, by applying \Cref{lemma:fournier}, we get that for $q>p$, there exists a constant $C_{p,q}$ such that,
    \begin{equation}
        \mathbb{E}[W_p^p(P^v_\#\hat{\mu}_n, P^v_\#\nu)] \le C_{p,q} \Tilde{M}_q(P^v_\#\mu)^{p/q} \left(n^{-1/2}\mathbb{1}_{\{q>2p\}} + n^{-1/2}\log(n) \mathbb{1}_{\{q=2p\}} + n^{-(q-p)/q} \mathbb{1}_{\{q\in(p,2p)\}}\right).
    \end{equation}

    Then, noting that necessarily, $P^v(o)=0$ (for both the horospherical and geodesic projection, since the geodesic is of the form $\exp_o(tv)$), and using that $P^v$ is 1-Lipschitz \Cref{lemma:lipschitz}, we can bound the moments as
    \begin{equation}
        \begin{aligned}
            \Tilde{M}_q(P^v_\#\mu) &= \int_\mathbb{R} |x|^q\ \mathrm{d}(P^v_\#\mu)(x) \\
            &= \int_\mathcal{M} |P^v(x)|^q\ \mathrm{d}\mu(x) \\
            &= \int_\mathcal{M} |P^v(x)-P^v(o)|^q\ \mathrm{d}\mu(x) \\
            &\le \int_\mathcal{M} d(x,o)^q\ \mathrm{d}\mu(x) \\
            &= M_q(\mu).
        \end{aligned}
    \end{equation}

    Therefore, we have
    \begin{equation}
        \mathbb{E}[\chsw_p^p(\hat{\mu}_n, \mu)] \le C_{p,q} M_q(\mu)^{p/q} \left(n^{-1/2}\mathbb{1}_{\{q>2p\}} + n^{-1/2}\log(n) \mathbb{1}_{\{q=2p\}} + n^{-(q-p)/q} \mathbb{1}_{\{q\in(p,2p)\}}\right),
    \end{equation}
    and similarly,
    \begin{equation}
        \mathbb{E}[\chsw_p^p(\hat{\nu}_n, \nu)] \le C_{p,q} M_q(\nu)^{p/q} \left(n^{-1/2}\mathbb{1}_{\{q>2p\}} + n^{-1/2}\log(n) \mathbb{1}_{\{q=2p\}} + n^{-(q-p)/q} \mathbb{1}_{\{q\in(p,2p)\}}\right).
    \end{equation}

    Hence, we conclude that 
    \begin{equation}
        \mathbb{E}[|\chsw_p(\hat{\mu}_n,\hat{\nu}_n) - \chsw_p(\mu,\nu)|] \le 2 C_{p,q}^{1/p} M_q(\nu)^{1/q} \begin{cases}
            n^{-1/(2p)}\ \text{if } q>2p \\
            n^{-1/(2p)} \log(n)^{1/p}\ \text{if } q=2p \\
            n^{-(q-p)/(pq)}\ \text{if } q \in (p,2p).
        \end{cases}
    \end{equation}
\end{proof}

\subsection{Proof of \Cref{prop:chsw_proj_complexity}} \label{proof:prop_chsw_proj_complexity}

\begin{proof}\textbf{of \Cref{prop:chsw_proj_complexity}}
    Let $(v_\ell)_{\ell=1}^L$ be iid samples of $\lambda_o$. Then, by first using Jensen inequality and then remembering that $\mathbb{E}_v[W_p^p(P^v_\#\mu,P^v_\#\nu)] = \chsw_p^p(\mu,\nu)$, we have
    \begin{equation}
        \begin{aligned}
            \mathbb{E}_v\left[|\widehat{\chsw}_{p,L}^p(\mu,\nu)-\chsw_p^p(\mu,\nu)|\right]^2 &\le \mathbb{E}_v\left[\left|\widehat{\chsw}_{p,L}^p(\mu,\nu)-\chsw_p^p(\mu,\nu)\right|^2\right]\\
            &= \mathbb{E}_v\left[\left|\frac{1}{L} \sum_{\ell=1}^L \big(W_p^p(P^{v_\ell}_\#\mu,P^{v_\ell}_\#\nu) - \chsw_p^p(\mu,\nu)\big)\right|^2\right] \\
            &= \frac{1}{L^2} \mathrm{Var}_v\left(\sum_{\ell=1}^L W_p^p(P^{v_\ell}_\#\mu,P^{v_\ell}_\#\nu)\right) \\
            &= \frac{1}{L} \mathrm{Var}_v\left(W_p^p(P^v_\#\mu,P^v_\#\nu)\right) \\
            &= \frac{1}{L} \int_{S_o} \left(W_p^p(P^v_\#\mu,P^v_\#\nu)-\chsw_p^p(\mu,\nu)\right)^2\ \mathrm{d}\lambda_o(v).
        \end{aligned}
    \end{equation}
\end{proof}

\section{Proofs of \Cref{section:chswf}} \label{proofs:section_chswf}

\subsection{Proof of \Cref{prop:chsw_1st_variation}} \label{proof:prop_chsw_1st_variation}

\begin{proof}\textbf{of \Cref{prop:chsw_1st_variation}}
    This proof follows the proof in the Euclidean case derived in \citep[Proposition 5.1.7]{bonnotte2013unidimensional} or in \citep[Proposition 1.33]{candau_tilh}.

    As $\mu$ is absolutely continuous, $P^v_\#\mu$ is also absolutely continuous and there is a Kantorovitch potential $\psi_v$ between $P^v_\#\mu$ and $P^v_\#\nu$. Moreover, as the support is restricted to a compact, it is Lipschitz and thus differentiable almost everywhere.
    
    First, using the duality formula, we obtain the following lower bound for all $\epsilon>0$,
    \begin{equation}
        \frac{\chsw_2^2\big((T_\epsilon)_\#\mu,\nu\big) - \chsw_2^2(\mu,\nu)}{2\epsilon} \ge \int_{S_o} \int_{\mathcal{M}} \frac{\psi_v(P^v(T_\epsilon(x))) - \psi_v(P^v(x))}{\epsilon}\ \mathrm{d}\mu(x)\mathrm{d}\lambda_o(v).
    \end{equation}
    Then, we know that the exponential map satisfies $\exp_x(0)=x$ and $\frac{\mathrm{d}}{\mathrm{d}t}\exp(tv)|_{t=0} = v$. Taking the limit $\epsilon\to 0$, the right term is equal to $\frac{\mathrm{d}}{\mathrm{d}t}g(t)|_{t=0}$ with $g(t) = \psi_v(P^v(T_t(x)))$ and is equal to
    \begin{equation}
        \frac{\mathrm{d}}{\mathrm{d}t}g(t)|_{t=0} = \psi_v'(P^v(T_0(x))) \langle \nabla P^v(T_0(x)), \frac{\mathrm{d}}{\mathrm{d}t}T_t(x)|_{t=0}\rangle_x = \psi_v'(P^v(x))\langle \mathrm{grad}_{\mathcal{M}} P^v(x),\xi(x)\rangle_x.
    \end{equation}
    Therefore, by the Lebesgue dominated convergence theorem (we have the convergence $\lambda_o$-almost surely and $|\psi_v(P^v(T_\epsilon(x)))-\psi_v(P^v(x))| \le \epsilon$ using that $\psi_v$ and $P^v$ are Lipschitz and that $d\big(\exp_x(\epsilon\xi(x)), \exp_x(0)\big) \le C\epsilon$),
    \begin{equation}
        \begin{aligned}
            &\liminf_{\epsilon\to 0^+}\ \frac{\chsw_2^2\big((T_\epsilon)_\#\mu,\nu\big)-\chsw_2^2(\mu,\nu)}{2\epsilon} \\ &\ge \int_{S_o} \int_{\mathcal{M}} \psi_v'(P^v(x))\langle \mathrm{grad}_\mathcal{M} P^v(x), \xi(x)\rangle\ \mathrm{d}\mu(x)\mathrm{d}\lambda_o(v).
        \end{aligned}
    \end{equation}

    For the upper bound, first, let $\pi^v\in\Pi(\mu,\nu)$ a coupling such that $\Tilde{\pi}^v=(P^v\otimes P^v)_\#\pi^v\in \Pi(P^v_\#\mu, P^v_\#\nu)$ is an optimal coupling for the regular quadratic cost. For $\Tilde{\pi}^v$-almost every $(x,y)$, $y=x-\psi_v'(x)$ and thus for $\pi^v$-almost every $(x,y)$, $P^v(y) = P^v(x) - \psi_v'\big(P^v(x)\big)$. Therefore,
    \begin{equation}
        \begin{aligned}
            \chsw_2^2(\mu,\nu) &= \int_{S_o} W_2^2(P^v_\#\mu, P^v_\#\nu)\ \mathrm{d}\lambda_o(v) \\
            &= \int_{S_o} \int_{\mathbb{R}\times\mathbb{R}} |x-y|^2 \ \mathrm{d}\Tilde{\pi}^v(x,y)\ \mathrm{d}\lambda_o(v) \\
            &= \int_{S_o} \int_{\mathcal{M}\times\mathcal{M}} |P^v(x) - P^v(y) |^2\ \mathrm{d}\pi^v(x,y)\ \mathrm{d}\lambda_o(v).% \\
            % &= \int_{S_o} \int_{\mathbb{R}} |\psi_v'(x)|^2\ \mathrm{d}(P^v_\#\mu)(x)\ \mathrm{d}\lambda(v) \\
            % &= \int_{S_o} \int_{\mathcal{M}} |\psi_v'\big(P^v(x)\big)|^2\ \mathrm{d}\mu(x)\ \mathrm{d}\lambda(v).
        \end{aligned}
    \end{equation}
    On the other hand, $((P^v\circ T_\epsilon)\otimes P^v)_\#\pi^v\in\Pi(P^v_\#(T_\epsilon)_\#\mu, P^v_\#\nu)$ and hence
    \begin{equation}
        \begin{aligned}
            \chsw_2^2\big((T_\epsilon)_\#\mu,\nu\big) &= \int_{S_o} W_2^2(P^v_\#(T_\epsilon)_\#\mu, P^v_\#\nu)\ \mathrm{d}\lambda_o(v) \\
            &\le \int_{S_o}\int_{\mathcal{M}\times\mathcal{M}} |P^v(T_\epsilon(x))-P^v(y)|^2\ \mathrm{d}\pi^v(x,y)\ \mathrm{d}\lambda_o(v).
        \end{aligned}
    \end{equation}
    Therefore,
    \begin{equation}
        \begin{aligned}
            &\frac{\chsw_2^2\big((T_\epsilon)_\#\mu,\nu\big) - \chsw_2^2(\mu,\nu)}{2\epsilon} \\ &\le \int_{S_o} \int_{\mathcal{M}\times \mathcal{M}} \frac{|P^v(T_\epsilon(x))- P^v(y)|^2-|P^v(x)-P^v(y)|^2}{2\epsilon} \ \mathrm{d}\pi^v(x,y)\ \mathrm{d}\lambda_o(v).
        \end{aligned}
    \end{equation}
    Note $g(\epsilon) = \big(P^v(T_\epsilon(x)) - P^v(y)\big)^2$. Then, $\frac{\mathrm{d}}{\mathrm{d}\epsilon}g(\epsilon)|_{\epsilon=0} = 2\big(P^v(x)-P^v(y)\big) \langle \mathrm{grad}_\mathcal{M} P^v(x), \xi(x)\rangle_x$. But, as for $\pi^v$-almost every $(x,y)$, $P^v(y) = P^v(x) - \psi_v'(P^v(x))$, we have
    \begin{equation}
        \frac{\mathrm{d}}{\mathrm{d}\epsilon}g(\epsilon)|_{\epsilon=0} = 2 \psi_v'\big(P^v(x)\big) \langle\mathrm{grad}_{\mathcal{M}}P^v(x), \xi(x)\rangle_x.
    \end{equation}
    Finally, by the Lebesgue dominated convergence theorem, we obtain
    \begin{equation}
        \begin{aligned}
            &\limsup_{\epsilon\to 0^+}\frac{\chsw_2^2\big((T_\epsilon)_\#\mu,\nu\big) - \chsw_2^2(\mu,\nu)}{2\epsilon} \\ &\le \int_{S_o} \int_{\mathcal{M}} \psi_v'(P^v(x))\langle \mathrm{grad}_\mathcal{M} P^v(x), \xi(x)\rangle_x\ \mathrm{d}\mu(x)\mathrm{d}\lambda_o(v).
        \end{aligned}
    \end{equation}
\end{proof}

\subsection{Proof of \Cref{prop:grad_lorentz_projs}} \label{proof:prop_grad_lorentz_projs}

\begin{proof}\textbf{of \Cref{prop:grad_lorentz_projs}}
    We apply \Cref{prop:grad_lorentz}. First, using that for $f:x\mapsto \langle x,y\rangle_\mathbb{L}$, $\nabla f(x) = -KJ y$, for all $x\in\mathbb{L}^d_K$,
    \begin{equation}
        \nabla B^v(x) = \sqrt{-K} J \frac{\sqrt{-K}x^0 + v}{\langle x, \sqrt{-K}x^0 + v\rangle_\mathbb{L}}.
    \end{equation}
    Thus, noticing that $J^2 = I_{d+1}$,
    \begin{equation}
        \begin{aligned}
            \mathrm{grad}_{\mathbb{L}_K^d} B^v(x) &= \mathrm{Proj}_x^K\big(-KJ\nabla B^v(x)\big) \\
            &= \mathrm{Proj}_x^K\left(-K\sqrt{-K} \frac{\sqrt{-K}x^0 + v}{\langle x,\sqrt{-K}x^0 +v\rangle_\mathbb{L}}\right) \\
            &= -K\sqrt{-K} \frac{\sqrt{-K} x^0+v}{\langle x,\sqrt{-K}x^0 + v\rangle_\mathbb{L}} - K \left\langle x, -K\sqrt{-K} \frac{\sqrt{-K} x^0 + v}{\langle x, \sqrt{-K} x^0 +v\rangle_\mathbb{L}}\right\rangle_\mathbb{L} x \\
            &= K\sqrt{-K} \left(\frac{\sqrt{-K}x^0 + v}{\langle x, \sqrt{-K}x^0 +v\rangle_\mathbb{L}} + K \frac{\langle x, \sqrt{-K} x^0 +v\rangle_\mathbb{L}}{\langle x, \sqrt{-K} x^0 + v\rangle_\mathbb{L}} x\right) \\
            &= K\sqrt{-K} \left( -\frac{\sqrt{-K}x^0 + v}{\langle x,\sqrt{-K}x^0 + v\rangle_\mathbb{L}} + Kx \right).
        \end{aligned}
    \end{equation}

    Similarly, we have
    \begin{equation}
        \nabla P^v(x) = \frac{-KJ\big(\langle x,x^0\rangle_\mathbb{L} v - \langle x,v\rangle_\mathbb{L}x^0\big)}{\langle x,v\rangle_\mathbb{L}^2 + K\langle x,x^0\rangle_\mathbb{L}^2}.
    \end{equation}
    Thus, observing that $\langle x,\nabla P^v(x)\rangle_\mathbb{L} = 0$, we have
    \begin{equation}
        \begin{aligned}
            \mathrm{grad}_{\mathbb{L}^d_K}P^v(x) &= \mathrm{Proj}_x^K\big(-KJ\nabla P^v(x)\big) \\
            &= -KJ\nabla P^v(x) - K \langle x, -KJ\nabla P^v(x)\rangle_\mathbb{L} x \\
            &= -KJ\nabla P^v(x) \\
            &= \frac{K^2\big(\langle x,x^0\rangle_\mathbb{L} v - \langle x,v\rangle_\mathbb{L} x^0\big)}{\langle x,v\rangle_\mathbb{L}^2 + K\langle x,x^0\rangle_\mathbb{L}^2}.
        \end{aligned}
    \end{equation}
\end{proof}

\subsection{Proof of \Cref{lemma:inverse_differential_log}} \label{proof:lemma_inverse_differential_log}

\begin{proof}\textbf{of \Cref{lemma:inverse_differential_log}}
    By \Cref{lemma:pennec_diff_log}, we have $\phi_{*,X}(V) = U\Sigma(V)U^T$ with $\Sigma(V) = U^T V U \odot \Gamma$. Thus,
    \begin{equation}
        \begin{aligned}
            U \Sigma(V) U^T = W &\iff \Sigma(V) = U^T W U \\
            &\iff U^T V U \odot \Gamma = U^T W U \\
            &\iff U^T V U = U^T W U \oslash \Gamma \\
            &\iff V = U\big(U^T W U \oslash\Gamma\big) U^T.
        \end{aligned}
    \end{equation}
\end{proof}

\subsection{Proof of \Cref{lemma:grad_proj_le}} \label{proof:lemma_grad_proj_le}

\begin{proof}\textbf{of \Cref{lemma:grad_proj_le}}
    By \citep[Equation 3.8]{pennec2020manifold}, we know that $\langle \log_{*,X}(V), Y\rangle = \langle \log_{*,X}(Y), V\rangle$. Thus, by linearity, we have that
    \begin{equation}
        \forall V\in T_XS_d^{++}(\mathbb{R}),\ P_{*,X}^A(V) = \langle A, \log_{*,X}(V)\rangle_F = \langle \log_{*,X}(A), V\rangle_F.
    \end{equation}
    Then, applying \Cref{lemma:pennec_diff_log}, we have the result.
\end{proof}

\section{Busemann Function on SPDs endowed with Affine-Invariant Metric} \label{appendix:busemann}

Let $A\in S_d(\mathbb{R})$, $M\in S_d^{++}(\mathbb{R})$, we recall from \Cref{section:spdai} that the Busemann function can be computed as
\begin{equation}
    \begin{aligned}
        B^A(M) &= \lim_{t\to\infty}\ d_{AI}\big(\exp(tA), M\big)-t \\
        &= - \langle A, \log(\pi_A(M)\rangle_F,
    \end{aligned}
\end{equation}
where $\pi_A$ is a projection on the spaces of matrices commuting with $\exp(A)$ which belongs to a group $G\subset GL_d(\mathbb{R})$ leaving the Busemann function invariant. In the next paragraph, we detail how we can proceed to obtain $\pi^A$.

 When $A$ is diagonal with sorted values such that $A_{11} > \dots > A_{dd}$, then the group leaving the Busemann function invariant is the set of upper triangular matrices with ones on the diagonal \citep[II. Proposition 10.66]{bridson2013metric}, \emph{i.e.} for any such matrix $g$ in that group, $B^A(M) = B^A(gMg^T)$. If the points are sorted in increasing order, then the group is the set of lower triangular matrices. Let's note $G_U$ the set of upper triangular matrices with ones on the diagonal. For a general $A\in S_d(\mathbb{R})$, we can first find an appropriate diagonalization $A=P\Tilde{A}P^T$, where $\Tilde{A}$ is diagonal sorted, and apply the change of basis $\Tilde{M}=P^TMP$ \citep{fletcher2009computing}. We suppose that all the eigenvalues of $A$ have an order of multiplicity of one. By the affine-invariance property, the distances do not change, \emph{i.e.} $d_{AI}(\exp(tA),M) = d_{AI}(\exp(t\Tilde{A}),\Tilde{M})$ and hence, using the definition of the Busemann function, we have that $B^A(M) = B^{\Tilde{A}}(\Tilde{M})$. Then, we need to project $\Tilde{M}$ on the space of matrices commuting with $\exp(\Tilde{A})$ which we denote $F(A)$. By \citet[II. Proposition 10.67]{bridson2013metric}, this space corresponds to the diagonal matrices. Moreover, by \citet[II. Proposition 10.69]{bridson2013metric}, there is a unique pair $(g,D)\in G_U\times F(A)$ such that $\Tilde{M} = gDg^T$, and therefore, we can note $\pi_A(\Tilde{M})=D$. This decomposition actually corresponds to a UDU decomposition. If the eigenvalues of $A$ are sorted in increasing order, this would correspond to a LDL decomposition.

\section{Additional Details on Experiments} \label{appendix:xp_classif}

\begin{table}[H]
    \centering
    \caption{Dataset characteristics.}
    \resizebox{0.8\linewidth}{!}{
        \begin{tabular}{ccccc}
             & \textbf{BBCSport} & \textbf{Movies} & \textbf{Goodreads genre} & \textbf{Goodreads like} \\ \toprule
            Doc & 737 & 2000 & 1003 & 1003 \\
            Train & 517 & 1500 & 752 & 752 \\
            Test & 220 & 500 & 251 & 251 \\
            Classes & 5 & 2 & 8 & 2 \\
            Mean words by doc & $116\pm 54$ & $182 \pm 65$ & $1491 \pm 538$ & $1491 \pm 538$\\
            Median words by doc & 104 & 175 & 1518 & 1518 \\
            Max words by doc & 469 & 577 & 3499 & 3499 \\
            \bottomrule
        \end{tabular}
    }
    \label{tab:summary_docs}
\end{table}

We sum up the statistics of the different datasets in \Cref{tab:summary_docs}.

% \appendix
% \section{}
% \label{app:theorem}

% % Note: in this sample, the section number is hard-coded in. Following
% % proper LaTeX conventions, it should properly be coded as a reference:

% %In this appendix we prove the following theorem from
% %Section~\ref{sec:textree-generalization}:

% In this appendix we prove the following theorem from
% Section~6.2:

% \noindent
% {\bf Theorem} {\it Let $u,v,w$ be discrete variables such that $v, w$ do
% not co-occur with $u$ (i.e., $u\neq0\;\Rightarrow \;v=w=0$ in a given
% dataset $\dataset$). Let $N_{v0},N_{w0}$ be the number of data points for
% which $v=0, w=0$ respectively, and let $I_{uv},I_{uw}$ be the
% respective empirical mutual information values based on the sample
% $\dataset$. Then
% \[
% 	N_{v0} \;>\; N_{w0}\;\;\Rightarrow\;\;I_{uv} \;\leq\;I_{uw}
% \]
% with equality only if $u$ is identically 0.} \hfill\BlackBox

% \section{}

% \noindent
% {\bf Proof}. We use the notation:
% \[
% P_v(i) \;=\;\frac{N_v^i}{N},\;\;\;i \neq 0;\;\;\;
% P_{v0}\;\equiv\;P_v(0)\; = \;1 - \sum_{i\neq 0}P_v(i).
% \]
% These values represent the (empirical) probabilities of $v$
% taking value $i\neq 0$ and 0 respectively.  Entropies will be denoted
% by $H$. We aim to show that $\fracpartial{I_{uv}}{P_{v0}} < 0$....\\

% {\noindent \em Remainder omitted in this sample. See http://www.jmlr.org/papers/ for full paper.}

\vskip 0.2in
\bibliography{references}

\end{document}